%% file: main.tex
\documentclass[]{article}
\usepackage[margin=1in]{geometry}
\usepackage{natbib}
\usepackage{times}

\usepackage{amsmath,amssymb,amsfonts}
\usepackage{graphicx}
\usepackage{textcomp}
\usepackage{xcolor}

\usepackage[utf8]{inputenc} %
\usepackage[T1]{fontenc}    %
\usepackage{hyperref}       %
\usepackage{url}            %
\usepackage{booktabs}       %
\usepackage{amsthm}         %
\usepackage{nicefrac}       %
\usepackage{microtype}      %
\usepackage{xspace}         %
\usepackage{tikz}           %
\usepackage{siunitx}        %
\usepackage{multirow}       %
\usepackage{bbm}            %
\usepackage{subcaption}     %
\usepackage{tabularx}       %
\usepackage{pifont}         %
\usepackage{xurl}           %
\usepackage{mathtools}      %
\usepackage{pifont}         %
\usepackage{threeparttable} %
\usepackage{algorithm}      %
\usepackage{algpseudocode}  %
\usepackage{cleveref}       %
\usepackage{wrapfig}        %
\usepackage{enumitem}       %

\input{macros}

\title{A Certified Robustness Framework for Backdoor \\\ Poisoning Attacks using $f$-Divergences}

\author{Evan Rose \\
Northeastern University \\
\texttt{rose.ev@northeastern.edu}
\and Alina Oprea \\
Northeastern University \\
\texttt{a.oprea@northeastern.edu}}

\date{}

\begin{document}

\maketitle

\begin{abstract}
    Machine learning systems trained on untrusted data are vulnerable to data and backdoor poisoning attacks, in which an adversary corrupts training and test samples to manipulate predictions at inference time.
    Randomized smoothing yields provable defenses against such attacks, but existing applications to train-time robustness are limited to specific threat models and smoothing mechanisms.
    We propose a general framework that allows expressing complex adversary constraints on the number of corrupted training examples, feature and label perturbation magnitudes, and test-time corruptions.
    Our framework randomizes inputs by combining data subsampling with feature, label and test-sample noising.
    We develop a composable certification procedure based on $f$-divergences that handles arbitrary combinations of threat models and noising strategies, and recovers several prior certification regimes exactly as special cases.
    We demonstrate our framework experimentally by certifying, for the first time, robustness under general $\ell_p$-constrained poisoning adversaries under bagging with additive Laplace, Gaussian, and Uniform noises.
\end{abstract}

\input{sections/01_introduction}
\input{sections/02_background}
\input{sections/03_problem}
\input{sections/04_methodology}
\input{sections/05_certification}
\input{sections/06_evaluation}
\input{sections/07_conclusion}

\subsubsection*{Acknowledgments}
This work was supported by NSF awards CNS-2312875 and CNS-2331081, the U.S. Army Combat Capabilities Development Command Army Research Laboratory (DEVCOM ARL) under Cooperative Agreement Number W911NF-24-2-0115, and by a grant from Good Ventures Foundation.

\bibliography{references}
\bibliographystyle{plainnat}

\appendix
\crefalias{section}{appendix}
\crefalias{subsection}{appendix}
\crefalias{subsubsection}{appendix}

\input{sections/A_smoothing_background}
\input{sections/A_divergence_background}

\input{sections/A_detailed_prior_work}
\input{sections/A_proofs}
\input{sections/A_tightness}

\input{sections/A_methodology_details}

\input{sections/A_divergence_bounds}

\end{document}

%% file: macros.tex
\newtheorem{dfn}{Definition}[section]
\newtheorem{thm}{Theorem}[section]
\newtheorem{lem}{Lemma}[section]
\newtheorem{cor}{Corollary}[section]

\crefname{dfn}{Definition}{Definitions}
\crefname{thm}{Theorem}{Theorems}
\crefname{lem}{Lemma}{Lemmas}
\crefname{cor}{Corollary}{Corollaries}
\crefname{fct}{Fact}{Facts}
\crefname{asm}{Assumption}{Assumptions}
\crefname{exm}{Example}{Examples}
\crefname{cnj}{Conjecture}{Conjectures}
\crefname{rmk}{Remark}{Remarks}

\DeclarePairedDelimiter\abs{\lvert}{\rvert}
\DeclarePairedDelimiter\norm{\lVert}{\rVert}
\DeclarePairedDelimiter\floor{\lfloor}{\rfloor}

\newcommand{\ind}{\mathbbm{1}}
\DeclareMathOperator{\supp}{supp}
\DeclareMathOperator*{\argmax}{arg\,max}

\DeclareMathOperator{\dom}{dom}

\DeclareMathOperator{\esssup}{ess\,sup}

\DeclareMathOperator*{\bbE}{\mathbb{E}} %
\newcommand{\bbP}{\mathbb{P}}
\newcommand{\bbR}{\mathbb{R}}

\newcommand{\bbN}{\mathbb{N}}

\newcommand{\cA}{\mathcal{A}}
\newcommand{\cB}{\mathcal{B}}
\newcommand{\cD}{\mathcal{D}}
\newcommand{\cF}{\mathcal{F}}
\newcommand{\cG}{\mathcal{G}}

\newcommand{\cM}{\mathcal{M}}
\newcommand{\cN}{\mathcal{N}}
\newcommand{\cO}{\mathcal{O}}
\newcommand{\cP}{\mathcal{P}}
\newcommand{\cR}{\mathcal{R}}
\newcommand{\cS}{\mathcal{S}}

\newcommand{\cV}{\mathcal{V}}
\newcommand{\cW}{\mathcal{W}}
\newcommand{\cX}{\mathcal{X}}
\newcommand{\cY}{\mathcal{Y}}
\newcommand{\cZ}{\mathcal{Z}}

\newcommand{\smooth}[1]{\overline{#1}}

\DeclareMathOperator{\Binom}{Binom}
\DeclareMathOperator{\Lap}{Lap}
\DeclareMathOperator{\Unif}{Unif}

\DeclareMathOperator{\unif}{Unif}

\newcommand{\dle}{\preceq}

\newcommand{\KL}{D_{\mathrm{KL}}}
\newcommand{\TV}{D_{\mathrm{TV}}}
\newcommand{\RY}[1]{D_{\mathrm{R},#1}}

\newcommand{\HS}[1]{D_{\mathrm{HS},#1}}

\newcommand{\Df}[1][f]{D_{#1}}
\newcommand{\Hsq}{D_{\mathrm{H}^2}}
\newcommand{\Hel}[1][\alpha]{D_{\mathrm{H},#1}}
\newcommand{\Ren}[1][\alpha]{D_{\mathrm{R},#1}}
\newcommand{\Hock}[1][\beta]{D_{\mathrm{HS},#1}}
\newcommand{\HockU}[1][\beta]{\widetilde{D}_{\mathrm{HS},#1}}
\newcommand{\chiP}{D_{\chi^2_{\mathrm{P}}}}
\newcommand{\chiN}{D_{\chi^2_{\mathrm{N}}}}
\newcommand{\Hint}[1][\alpha]{Z_{#1}}

\newcommand{\defeq}{\coloneqq}

\newcommand{\emp}{m_{\mathrm{p}}}

\newcommand{\epsf}{\varepsilon_{\mathrm{f}}}
\newcommand{\epsl}{\varepsilon_{\mathrm{l}}}
\newcommand{\epste}{\varepsilon_{\mathrm{te}}}

\algrenewcommand\algorithmicrequire{\textbf{Input:}}
\algrenewcommand\algorithmicensure{\textbf{Output:}}

\usetikzlibrary{calc,arrows.meta}

%% file: sections/01_introduction.tex
\section{Introduction}
\label{sec:introduction}

Despite their continued adoption in wide-reaching domains, machine learning systems remain susceptible to adversarial attacks that subvert their reliability and can cause harmful outcomes.
These attacks span a broad class of threat models, including both training- and inference-phase attacks.
Of particular concern are \textit{data} and \textit{backdoor poisoning} attacks, in which an adversary corrupts a portion of the training data to degrade or control model behavior~\citep{gu2019badnets,shafahi2018poison,saha2020hidden,aghakhani2021bullseye,geiping2021witches,souri2022sleeper}.
To defend against such attacks, certified defenses have been proposed that provide provable guarantees on the consistency of a classifier's output around an input.
The most popular of these techniques, randomized smoothing, works by lifting the classifier into a function over a noised version of the input and defining a new classifier by taking the expectation~\citep{lecuyer2019certified,cohen2019certified}.

The intuitive reason randomized smoothing works is that averaging a function over a noised input induces a uniform continuity property whose stability depends primarily on the consistency of the function around the input, rather than on the continuity properties of the base function.
This principle applies whether the input and noise are simple (e.g., classification under Gaussian noise) or complex (e.g., an entire randomized dataset).
In the case of data poisoning or backdoor attacks, several natural noising strategies exist, including data subsampling, feature noising, and label flipping.
However, as the noising procedures become more complex, analyzing the certified guarantees becomes increasingly difficult.
As a result, existing certified defenses for data poisoning attacks remain limited to smoothing strategies such as simple bagging or discrete label flipping~\citep{jia2021intrinsic,chen2022collective,zhang2022bagflip}.

In this work, we ask, \emph{can we develop a highly general randomized smoothing framework for data poisoning and backdoor adversaries that easily handles complex threat models and smoothing strategies?}
Our insight is to leverage the convex-optimization view of randomized smoothing first introduced by~\citet{dvijotham2020framework}, in which the admissible perturbation set is replaced by a convex outer approximation defined by $f$-divergence constraints.
The certification condition relies on the smoothing mechanism and threat model only through a set of divergence constraints corresponding to a certain dual problem with nice composition properties.
We argue this is the ``right'' perspective for composite mechanisms such as those arising in the backdoor poisoning setting.

\paragraph{Contributions}
We present a randomized smoothing framework for backdoor and data poisoning attacks that incorporates rich attacker threat models, a multi-step randomization scheme, and efficient certification via convex relaxation using $f$-divergences.
Our framework makes it straightforward to describe complex threat models spanning feature, label, and test instance corruptions, to describe noising strategies that defend against them, and to derive sound robustness certificates under the resulting schemes.
We explain how our framework recovers the tight certificates of several prior works as special cases.
We use our framework to study certifiably robust classification under backdoor and data poisoning attacks in a wide range of settings not considered in prior works.

%% file: sections/02_background.tex
\section{Background and related work}
\label{sec:background}

A growing body of work studies how to obtain certifiable robustness guarantees against data poisoning and backdoor attacks.
\Cref{apx:prior-work} gives a more detailed treatment of selected works.

\paragraph{Randomized smoothing}
In response to the weaknesses of empirical defenses~\citep{athalye2018obfuscated},
randomized smoothing emerged as a technique to obtain provable robustness to adversarial examples~\citep{lecuyer2019certified,cohen2019certified,li2019certified,salman2019provably}.
Early work established the standard certificate analysis methodology via the Neyman--Pearson Lemma~\citep{neyman1933problem} and the use of confidence-corrected estimates of the smoothed prediction.
Other characterizations have also been explored, including functional optimization over $f$-divergence constraints~\citep{dvijotham2020framework} and a nonlinear Lipschitz condition~\citep{salman2019provably}.

\paragraph{Learner-specific defenses}
Some works derive robustness certificates for specific learning algorithms.
Early work studied duality-based upper bounds on the total test error~\citep{steinhardt2017certified}.
\citet{rosenfeld2020certified} derive an efficient certified inference algorithm under label flipping attacks for classifiers trained using linear regression and label flipping noise.
\citet{drews2020proving} train certifiably robust decision trees with abstract interpretation.
\citet{jia2022certified} derive efficient certified guarantees for nearest neighbor classifiers under label flipping attacks.
\citet{sabanayagam2025exact} and~\citet{gosch2026certifying} certify graph neural networks and SVMs against data poisoning, backdoor, and label flipping attacks using the neural tangent kernel and linear programming.
\citet{sosnin2025certified} extend abstract interpretation to gradient-based training to certify gradient-based learners.

\paragraph{Differential privacy for robustness}
Another line of work obtains robustness from differentially private (DP) learners.
\citet{ma2019data} certify global accuracy using the stability of expectation under a differentially private algorithm, and \citet{xie2023unraveling} study a similar property in federated learning.
\citet{liu2023enhancing} extend the results of \citet{ma2019data} to instance-wise poisoning by improving the group privacy analysis of the subsampled Gaussian mechanism.
\citet{saxena2026provable} connect the standard hypothesis testing framework for certification to the theory of DP privacy profiles, enabling efficient analysis of composed mechanisms such as DP-SGD.

\paragraph{Voting-based poisoning defenses}
The most popular strategy for provable robustness to poisoning is to train an ensemble of learners over different partitions or randomized samples and to use robustness properties of the majority vote.
One line of work considers deterministic partitioning strategies~\citep{levine2021deep,wang2022improved,rezaei2023runoff,hammoudeh2024provable}.
Another line of work considers Randomized Smoothing by training ensembles on randomized samples of training data~\citep{weber2023rab,jia2021intrinsic,chen2022collective,zhang2022bagflip}.

%% file: sections/03_problem.tex
\section{Problem statement and preliminaries}
\label{sec:problem-statement}

\subsection{Learning setup}
\label{sec:learning-setup}
We consider classification tasks with input space $\cX \subseteq \bbR^d$ and label space $\cY = [C]$, and write $\cZ \defeq \cX \times \cY$ for the labeled input space and $\cS \defeq \cZ^\ast$ for the space of training sets.
We absorb training and inference into a single procedure: define the joint input space $\cW \defeq \cS \times \cX$ and suppose $\cW \subseteq \widetilde{\cW}$ for some ambient space $\widetilde{\cW}$.
\footnote{We introduce $\widetilde{\cW}$ because the randomization we introduce later means the base learner must process datasets from a slightly different space than $\cW$.}
A \textit{base predictor} $G \colon \widetilde{\cW} \to \cP(\cY)$ maps a joint input $w = (S, x_{\mathrm{te}})$ to a distribution over labels.
Concretely, $G$ composes a learning algorithm $\cA$ with evaluation at a test point, $G(w) = G(S, x_{\mathrm{te}}) = \cA(S)(x_{\mathrm{te}})$.
The induced hard predictor is $g(w) \defeq \argmax_{y \in \cY} G(w)_y$, with a fixed tie-breaking rule.
$\cA$ can be randomized, in which case $G$ denotes the expectation over $\cA$'s internal randomness (and thus $G$ is a deterministic function of $w$).

\subsection{Threat model}
\label{sec:threat-model}
We consider a replacement adversary who jointly corrupts the training set and test sample.
For joint input $w = (S, x_{\mathrm{te}})$ with $S = ((x_i, y_i))_{i=1}^n$, denote the corrupted input by $w' = (S', x'_{\mathrm{te}})$, where $S' = ((x'_i, y'_i))_{i=1}^n$.
In particular, the adversary may corrupt training features $x_i \mapsto x'_i$, training labels $y_i \mapsto y'_i$, and the test sample $x_{\mathrm{te}} \mapsto x'_{\mathrm{te}}$, but may not insert or delete training examples.

We model corruption constraints along four axes.
The first axis constrains the number of corrupted examples in the training set.
The remaining axes constrain per-example feature or label corruptions in terms of a feature-space distance function $d_{\cX}$ or a label space distance function $d_{\cY}$.
The adversary's capability is thus parameterized by the pair $(d_{\cX}, d_{\cY})$ together with a \textit{threat vector} $\rho = (\emp, \epsf, \epsl, \epste)$.
Define $I(w, w') \defeq \{i \in [n] \,:\, (x_i, y_i) \ne (x'_i, y'_i)\}$.
The threat vector fixes the constraints (1) $\abs{I} \le \emp$, (2) $\max_{i \in I} d_{\cX}(x_i, x'_i) \le \epsf$, (3) $\max_{i \in I} d_{\cY}(y_i, y'_i) \le \epsl$, (4) $d_{\cX}(x_{\mathrm{te}}, x'_{\mathrm{te}}) \le \epste$.
These constraints determine the \textit{perturbation set} $\cB_{\rho}(w) \defeq \{w' \in \cW \,:\, \text{constraints (1)--(4) hold}\}$.

Decoupling the budget into four independent axes allows us to recover the threat models of several prior works as special cases~\citep{rosenfeld2020certified,jia2021intrinsic,zhang2022bagflip}, as well as to generalize to several previously unconsidered threat models (e.g., $\emp < \infty$ with $d_{\cX} = \norm{\cdot}_p$ for $p \ne 0$).
Usually we will take $d_{\cY}$ discrete and $d_{\cX} \equiv \norm{\cdot}_p$ for some $0 \le p \le \infty$.

\subsection{Certified robustness}
\label{sec:certification-definition}

We seek to provably defend against attacks within our threat model under the following definition.

\begin{dfn}[$\rho$-Robustness]
    \label{dfn:certified-robustness}
    Let $g \colon \cW \to \cY$ be a deterministic hard predictor.
    We say that $g$ is \emph{$\rho$-robust} at $w \in \cW$ if $g(w) = g(w')$ for all $w' \in \cB_\rho(w)$.
\end{dfn}

Existing certified poisoning defenses typically report robustness as a scalar \textit{radius} specifying the tolerable corruption magnitude along a single axis (e.g., $\emp$ or $\epsf$)~\citep{rosenfeld2020certified,levine2021deep,zhang2022bagflip}.
Instead, we are interested in characterizing the full certified region directly.
For fixed $g$, define the \textit{robustness region} $\cR(w) \defeq \{\rho \,:\, g \text{ is } \rho\text{-robust at } w\}$.
Defining the robustness region $\cR$ encodes the idea that varying one threat parameter (e.g., $\epste$) can affect robustness against other parameters (e.g., $\epsf$), and we study the entire object holistically.

%% file: sections/04_methodology.tex
\section{Methodology}
\label{sec:methodology}
Our goal is to develop a general defense for training classifiers that handles the threat models from \Cref{sec:problem-statement} and admits an efficient certification procedure proving robustness.

\subsection{Robust learning procedure}
\label{sec:robust-learning-procedure}
\paragraph{Robust training via randomized smoothing}
We design a certifiably robust training procedure following the methods of \textit{Randomized Smoothing}~\citep{lecuyer2019certified,li2019certified,cohen2019certified},
initially introduced for certified robustness against evasion attacks.
The main idea behind randomized smoothing is to transform a pre-existing \textit{base learner} (in our case, the joint training-inference procedure $G$) into a provably robust \textit{smoothed} classifier $\smooth{G_{\cM}}$.
The smoothed classifier $\smooth{G_{\cM}}$ is obtained by applying a randomized \textit{smoothing mechanism} $\cM \colon \cW \to \cP(\widetilde{\cW})$ to the inputs and lifting the base classifier $G$ to a classifier on the space of \textit{distributions} over the joint input space according to the aggregation rule: $\smooth{G_{\cM}}(w) \defeq \bbE_{\widetilde{w} \sim \cM(w)} \left[G(\widetilde{w})\right]$.
$\smooth{G_{\cM}}$ is a deterministic soft classifier whose associated hard classifier is $\overline{g_{\cM}}(w) = \argmax_{y} \smooth{G_{\cM}}(w)_y$.

\paragraph{Randomized mechanism}
The smoothing mechanism $\cM$ should be chosen in a way to ``wipe away'' any influence the corruptions might have on the prediction.
We design a multi-step randomization procedure that consists of a data subsampling phase followed by a feature and label noising phase.
The resulting probability measure takes the form of a mixture distribution whose components correspond to the samples retained during the subsampling phase.
Each component of the mixture follows structurally identical randomization behavior, but is instantiated with different parameters.

For a training set $S = ((x_i, y_i))_{i = 1}^n$ and sample-index sequence $U = (u_1, \ldots, u_k) \in [n]^k$, define $S_U \defeq ((x_{u_i}, y_{u_i}))_{i = 1}^k$.
Let $\Pi \in \cP([n]^\ast)$ be a subsampling distribution over retained sample indices $U$, $\phi_{\cX}(\cdot)$ a \textit{feature noise} mechanism (e.g., $\phi(x) = \cN(x, \sigma^2 I)$), and $\phi_{\cY}(\cdot)$ a \textit{label noise} mechanism (e.g., $\phi(y) = (1-\gamma) \delta_y + \gamma \unif([C])$).
Our smoothing mechanism $\cM \defeq \cM_{\Pi, \phi_{\cX}, \phi_{\cY}}$ is generated by the following procedure.
\textbf{(1)} Draw retained indices $U \sim \Pi$.
\textbf{(2)} Apply per-example noise to features and labels:
For each $i \in [\abs{U}]$, independently sample $\widetilde{x}_i \sim \phi_{\cX}(\cdot \mid x_{u_i})$ and $\widetilde{y}_i \sim \phi_\cY(\cdot \mid y_{u_i})$.
Construct $\widetilde{S} \defeq ((\widetilde{x}_{i}, \widetilde{y}_{i}))_{i=1}^{\abs{U}}$.
\textbf{(3)} Sample $\widetilde{x}_{\mathrm{te}} \sim \phi_\cX(\cdot \mid x_{\mathrm{te}})$.
\footnote{We couple train and test noise mechanisms for simplicity; decoupling them is conceptually straightforward.}
\textbf{(4)} Output $(\widetilde{S}, \widetilde{x}_{\mathrm{te}})$.

We make a few observations on the distributions resulting from the smoothing mechanism $\cM$.
By construction, sampling from $\cM$ is efficient.
$\cM$ can be written as a mixture distribution whose components are product measures parameterized by a subsample of $S$.
$\cM$ has $\abs{\supp \Pi}$ components, which is generally quite large (e.g., $n^k$ for bagging with replacement).
This makes evaluating the density (or mass) at a point a priori intractable.
For a large number of norm-based smoothing mechanisms (e.g., Gaussian), the product measure induced by $\phi_{\cX}$ reduces to a well-understood measure over the full dataset space identified with $\bbR^{nd}$ (e.g., $\prod_{i=1}^n \cN(0_d, \sigma^2I_d) \cong \cN(0_{nd}, \sigma^2 I_{nd})$).

\paragraph{Concrete instantiations}
We motivate our smoothing scheme with a few concrete instantiations.
Natural choices of $\Pi$ are bagging (with or without replacement) or Poisson subsampling.
Natural choices for $\phi_{\cX}$ include Gaussian, Laplace, Uniform, or discrete noise, but we can also accommodate more exotic noises such as the norm-based noises of \citet{yang2020randomized}.
Most often $\phi_{\cY}$ will be a flipping distribution $\phi(y) = \phi^{\textrm{flip}}_{\gamma, K}(y) = (1-\gamma) \delta_{y} + \gamma \mathrm{Unif}([K])$ for some alphabet size $K$, but we can also accommodate structured noises $\phi_{\cY}(y) = M_{y \widetilde{y}}$ for a Markovian transition matrix $M$.

The smoothing mechanisms of several prior works can be recovered for suitable choices of $\Pi, \phi_{\cX}, \phi_{\cY}$.
RAB~\citep{weber2023rab} takes $\Pi$ as a point mass on the original dataset (no subsampling) and sets $\phi(x) = \cN(x, \sigma^2 I_d)$.
\citet{jia2021intrinsic} take $\Pi$ as bagging with replacement with no feature or label noise.
BagFlip~\citep{zhang2022bagflip} considers bagging with replacement and coupled feature-label noise $\phi_{\cX}(x) = \bigotimes_{j=1}^{d} \phi^{\textrm{flip}}_{\gamma, K}(x_j)$ and $\phi_{\cY}(y) = \phi^{\textrm{flip}}_{\gamma, K}(y)$, where notably the feature and label alphabets are the same size and use the same flipping probability.

%% file: sections/05_certification.tex
\subsection{Certifying predictions: overview}
\label{sec:certifying}
\paragraph{Core challenges}
The main technical challenge we face is how to certify robustness for such a large class of general smoothing methods against the adversaries defined in \Cref{sec:threat-model}.
Prior works applying randomized smoothing to data poisoning or backdoor attacks~\citep{weber2023rab,jia2021intrinsic,zhang2022bagflip} apply techniques introduced by \citet{cohen2019certified}, who frame certification as a hypothesis testing problem and construct Neyman--Pearson tests to derive tight certificates.
Unfortunately, this path is only accessible for simple smoothing schemes and threat models, and does not scale to our setting: the likelihood ratio structure of $\cM$ is complicated due to its combinatorial mixture structure and is typically only directly tractable for simple distributions such as the discrete noises of \citet{zhang2022bagflip}.
Furthermore, changing the mechanisms $\Pi, \phi_{\cX}, \phi_{\cY}$ or the threat model $d_{\cX}, d_{\cY}, \rho$ a priori requires a laborious new derivation of the certification condition by hand.

What we desire instead is a generic way to \textit{automatically construct} the certification conditions given the threat model $d_{\cX}, d_{\cY}, \rho$ and the smoothing mechanism $\Pi, \phi_{\cX}, \phi_{\cY}$.
All relevant information can be derived for each component independently, and different combinations can be composed automatically with basic computational rules.
It turns out that such a procedure is possible if we turn to a certain dual formulation of Randomized Smoothing, first considered by \citet{dvijotham2020framework} for the evasion setting in the language of \emph{$f$-divergences}.
We start by introducing some preliminaries.

\subsection{Certification via convex optimization}
Define a \textit{specification} as a bounded measurable $\psi \colon \Omega \to \bbR$, and call it \textit{robustly certified} at $\nu \in \cP(\Omega)$ w.r.t. $\cD \subseteq\cP(\Omega)$ if $\bbE_{\nu'}[\psi] \ge 0$ for every $\nu' \in \cD$ (\Cref{dfn:specification}).
For convex $f$ with $f(1) = 0$ (a \textit{generator}), the induced $f$-divergence is $\Df(\nu' \| \nu) \defeq \bbE_{\nu}[f(d\nu' / d\nu)]$ (\Cref{dfn:f-divergence}).
Given generators $f_1,\ldots,f_m$ with bounds $\tau_i \ge 0$, define $\cD_\cF \defeq \bigcap_i\{\nu' : \Df[f_i](\nu' \| \nu) \le \tau_i\}$ (\Cref{dfn:f-divergence-set}).

We use the following result of~\citet{dvijotham2020framework}, stated in full in \Cref{thm:lim-info-cvx}.
Let $\theta_a, \theta_b$ denote lower and upper bounds on the majority and runner-up predicted probabilities from $\smooth{G_{\cM}}(w)$, respectively, and write $\theta = (\theta_a, \theta_b, 1 - \theta_a - \theta_b)$.
Then, every specification consistent with $\theta_a, \theta_b$ is robustly certified w.r.t. $\cD_\cF$ iff the solution to the following problem is non-negative:
\begin{equation}
    \label{eq:lim-info-program-brief}
    \inf_{\zeta \in \bbR_+^3} \zeta_a - \zeta_b
    \quad\text{s.t.}\quad
    \zeta_a + \zeta_b + \zeta_c = 1,
    \quad
    \Df[f_i](\zeta \| \theta) \le \tau_i,\;\; i=1,\ldots,m.
\end{equation}
Robustness certification in the sense of \Cref{dfn:specification}
is related to our own robustness notion (\Cref{dfn:certified-robustness}).
Writing $c \defeq \smooth{g_\cM}(w)$ and taking specifications $\psi_{c,c'}(\widetilde{w}) \defeq G(\widetilde{w})_c - G(\widetilde{w})_{c'}$, $c \in \argmax_y \smooth{G_\cM}(w')_y$ for all $w' \in \cB_\rho(w)$ iff every $\psi_{c,c'}$ with $c \ne c'$ is robustly certified at $\cM(w)$ w.r.t. $\cD_{w,\rho} \defeq \{\cM(w') \;:\; w' \in \cB_{\rho}(w)\}$.
We state and prove the formal statement in \Cref{thm:robustness-equivalence}.

In working with \Cref{eq:lim-info-program-brief}, it suffices to choose any $\tau_i \ge \Df[f_i](\nu' \| \cM(w))$ for every $\nu' \in \cD_{w, \rho}$, with the tightest choice being
\begin{equation}
    \label{eq:tau-star}
    \tau_f^{\star}(w, \rho) \defeq \sup_{w' \in \cB_\rho(w)} \Df(\cM(w') \| \cM(w)).
\end{equation}
One can use any $f$-divergence for which an upper bound on \Cref{eq:tau-star} is available.
The most interesting cases are the hockey-stick divergences, due to their characterization of tight certificates (\Cref{thm:tight-certificate-hockey}), and the R\'enyi divergences, due to their nice tensorization properties (\Cref{lem:tensorization}).

\subsection{Computing divergence bounds}
\label{sec:bounding-mixtures}

Under \Cref{thm:robustness-equivalence,thm:lim-info-cvx}, the problem of certifying robustness for $\overline{G_{\cM}}$ at $w \in \cW$ reduces to establishing a bound $\tau_f(w, \rho) \ge \tau_f^\star(w, \rho)$ for suitable $f$.
The technical challenge is grappling with the complex structure of $\cM$: subsampling through $\Pi$ induces a mixture with combinatorially many components, and noising through $\phi$ makes each component a product measure.
Our analysis proceeds in two steps: first, we reduce the combinatorial structure of the mixture distribution to a lower-dimensional problem, and second, we analyze divergences between paired product components.

\paragraph{Bounding divergences for mixtures}
The first key observation for bounding divergences is that the components of any $\cM(w')$ and $\cM(w)$ share mixing weights when paired by subsample, which interact nicely with $\Df(\cM(w') \| \cM(w))$.
Moreover, despite a combinatorial number of components, the components can often be grouped together, which gives us the following bound:

\begin{thm}[Reorganized Joint Convexity Bound]
    \label{thm:joint-convexity-bound}
    Let $\mu = \sum_{i=1}^N \theta_i \mu_i$ and $\nu = \sum_{i=1}^N \theta_i \nu_i$ be mixtures with shared mixing weights $\theta_i$.
    Suppose there is a partition of indices $[N] = \bigsqcup_{i=1}^k S_i$ and a sequence of nonnegative real numbers $(\delta_i)_{i=1}^k$ such that for all $i \in [k]$ and all $j \in S_i$, $D_f(\mu_j \| \nu_j) \le \delta_i$.
    Then one has $D_f(\mu \| \nu) \le \tau_f$ for $\tau_f \defeq \sum_{i=1}^k \beta_i \delta_i$, where $\beta_i \defeq \sum_{j \in S_i} \theta_j$.
\end{thm}
We prove \Cref{thm:joint-convexity-bound} in \Cref{apx:deferred-proofs}, along with an analogue for R\'enyi divergences (\Cref{cor:renyi-mixture-full}).
\Cref{thm:joint-convexity-bound} does not depend on the specific geometry of the feature or label noise and only uses pairing of components and joint convexity of $\Df$.
The grouping of components is important computationally, as it reduces the sum to a tractable number of terms.
For our noises, the component bound depends on $U$ only through the number $s(U)$ of corrupted indices.
Grouping by $s$ leaves only a few terms, e.g., $k + 1$ for bagging and $\min(k, \emp) + 1$ for bagging without replacement.
This trick can be viewed as an extension to $f$-divergences of earlier techniques that partition the output space into regions of constant likelihood ratio~\citep{lee2019tight,bojchevski2020efficient,zhang2022bagflip}.

\subsection{Bounding component divergences}
It remains to find the bounds $\delta_j$.
Consider a fixed corruption $w'$ of $w$, and apply \Cref{thm:joint-convexity-bound}.
Writing $z_U \defeq (S_U, x_{\mathrm{te}})$ for an un-noised subsampled input, we obtain the component pairs $(\phi(z_U), \phi(z'_U))$ with $z'_U$ determined by $w'$.
A small complication is that the $z_U$ are coupled through $w'$, so we relax the problem slightly and take $\delta_j \ge \sup_{z'} \Df(\phi(z') \| \phi(z))$ over all corruptions $z'$ of $z$ that are feasible given group $j$ (for example, a component might correspond to $U$ with $s$ corrupted examples).
This is tight when some $w'$ attains every component supremum simultaneously.

\paragraph{Adjacency parameters}
The supremum $\sup_{z'} \Df(\phi(z') \| \phi(z))$ is natively over a high-dimensional space.
However, for sufficiently symmetrical noises, the divergence between $(x,x')$ depends only on a low-dimensional \textit{adjacency parameter} $\theta(x,x')$, in the sense that every $f$-divergence can be written in the form $\Df(\phi(x') \| \phi(x)) = \delta_f(\theta(x, x'))$ for some $\delta_f$ depending on $f$.
One way to see this is by the invariance of $f$-divergences under bimeasurable bijections (\Cref{lem:data-processing}).
As a concrete example, for Gaussian noise any $f$-divergence depends only on the $\ell_2$ separation $r = \norm{x-x'}_2 / \sigma$.

\begin{table}[ht]
    \caption{Divergence bounds derived in \Cref{apx:calculus}.
    Here $\theta$ is the adjacency parameter of \Cref{dfn:adjacency-parameter} and $\Delta \defeq x - x'$. Notation $p_+$, $p_0$, $p_-$, $\kappa_\alpha(t)$, $d_\beta$, and $P_d$ are defined in \Cref{apx:calculus}.}
    \label{tab:divergences}
    \centering
    \scriptsize
    \setlength{\tabcolsep}{3pt}
    \renewcommand{\arraystretch}{1.35}
    \begin{tabular}{@{}llll@{}}
        \toprule
        \textbf{Noise} & \textbf{Adjacency $\theta$} & $\Ren[\alpha](\mu \| \nu)$ & $\Hock[\beta](\mu \| \nu)$ \\
        \midrule
        $\Unif(\gamma)$
          & $1 - \prod_i \bigl(1 - \tfrac{\abs{\Delta_i}}{2\gamma}\bigr)_{\!+}$
          & $\tfrac{1}{\alpha - 1}\log(1 - \theta)$
          & $\min(1, \beta)\,\theta$ \\
        $\cN(\cdot, \sigma^2 I)$
          & $r = \norm{\Delta}_2 / \sigma$
          & $\alpha r^2 / 2$
          & $\Phi\bigl(\tfrac{r}{2} - \tfrac{\log\beta}{r}\bigr) - \beta \Phi\bigl(-\tfrac{r}{2} - \tfrac{\log\beta}{r}\bigr) - (1 - \beta)_+$ \\
        $\Lap(\cdot, b)$
          & $\mathbf{t} = \abs{\Delta} / b$
          & $\sum_j \kappa_\alpha(\mathbf{t}_j)$
          & $\bigl[(1 - \sqrt{\beta}e^{-t/2})_+ - (1-\beta)_+\bigr]_+$ \\
        $\phi^{\textrm{flip}}_{\gamma, C}$
          & $s = \norm{\Delta}_0$
          & $\tfrac{s}{\alpha - 1}\log\bigl(p_+^\alpha p_-^{1-\alpha} + p_+^{1-\alpha}p_-^\alpha + p_0\bigr)$
          & $\sum_{d \ge d_\beta} \bigl(1 - \beta (p_+/p_-)^{-d}\bigr) P_d$ \\
        \bottomrule
    \end{tabular}
    \label{tab:divergence-bounds}
\end{table}

\paragraph{Worst-case pairs}
A priori, \Cref{eq:tau-star} admits a different maximizer for each $f$.
In most cases, this is fortunately not the case.
Write $(\mu', \nu') \dle (\mu, \nu)$ when $\Df(\mu' \| \nu') \le \Df(\mu \| \nu)$ for every generator $f$.
We show that $\dle$ is entirely determined by the Hockey-stick envelope (\Cref{lem:dle-via-hs}): $(\mu', \nu') \dle (\mu, \nu)$ iff $\Hock(\mu' \| \nu') \le \Hock(\mu \| \nu)$ for all $\beta > 0$.
In such cases, a single $(\mu^\star, \nu^\star)$ generates every $\tau_f^\star$ simultaneously.
This property also composes nicely under product measures (\Cref{lem:worst-case-product}), which is useful for combining feature and label flipping.
We show that such a pair exists for Gaussian and Uniform noise under all $\ell_p$ norms, for flipping noise under $\ell_0$, and for Laplace noise under $\ell_1$ and $\ell_\infty$ norms.
For $1 < p < \infty$
For Laplace noise under $\ell_2$, the behavior is more complicated, requiring special treatment (\Cref{apx:worst-case-laplace}).

Worst-case pairs, when they exist, reduce divergence bound computation to optimizing the adjacency parameter.
For Gaussian noise this corresponds to maximizing the $\ell_2$ perturbation, which is easy.
For Uniform noise the problem is nontrivial, and we must resort to a loose relaxation (\Cref{apx:worst-cast-uniform}).

\paragraph{Composition under products}
Our framework involves applying independent noise (between features and labels within an example, and independence between examples), which requires us to understand $f$-divergences between product measures.
R\'enyi divergences are attractive here due to their nice tensorization properties (\Cref{lem:tensorization}) and can be computed immediately from components.
We show in \Cref{lem:hockey-stick-product} how to compute Hockey-stick divergences $\Hock(\mu_1 \otimes \mu_2 \| \nu_1 \otimes \nu_2)$ in closed form via the formula $\HockU(\mu \| \nu) = \bbE_{\mu_1}[\HockU[\beta / \Lambda_1](\mu_2 \| \nu_2)]$.
This is especially useful for modeling dirty-label adversaries, where the corresponding noise $\phi$ can be written as a product of a label flipping measure $\mu_1 = \phi^{\textrm{flip}}_{\gamma, C}$ and a continuous noise $\mu_2$ (e.g., $\cN(x, \sigma^2 I)$).

\paragraph{Numerical computation}
Gaussian, flipping, and Uniform noise have hockey-stick profiles computable in closed-form once their worst-case adjacency parameter is known.
For Laplace noise the profile involves repeated self-convolution of the log likelihood ratio distribution (which has no nice closed form), so we compute it with numerical integration techniques.

\paragraph{Worked example}
Consider $\cX = \bbR^{d}$, $\cY = [C]$, and $S = ((x_i, y_i))_{i=1}^n$.
Let $\Pi$ be bagging with replacement of $k$ samples, and let $\phi_{\cX}(x) = \cN(x, \sigma^2 I_d)$, so that $\cM(S) = \sum_{U \in [n]^k} n^{-k} \cN(S_U, \sigma^2 I_{dk})$ has $n^k$ components.
Under an $\ell_2$ adversary, a component with $s$ corrupted samples has adjacency parameter $r = \norm{S_U - S'_U}_2 / \sigma \le \epsf \sqrt{s} / \sigma$, which is attained at maximal corruption magnitude.
Since $s \sim \Binom(k, \emp / n)$ under $\Pi$, \Cref{thm:joint-convexity-bound} gives $\tau_f = \sum_{s=0}^k \Binom(s;k,\emp / n) \delta_f(\epsf \sqrt{s} / \sigma)$, where $\delta_f$ depends on $f$ (e.g., $\delta_{\mathrm{TV}}(r) = 2\Phi(r/2)-1$; see \Cref{tab:divergence-bounds}).
If we also model a backdoor trigger, we can replace $\epsf^2 \,s$ by $\epsf^2 \, s + \epste^2$ to give $\delta_f(\sqrt{\epsf^2\,s +\epste^2 } / \sigma)$ in the sum.
For label-flipping adversaries and randomized labels, we can use \Cref{lem:hockey-stick-product} to compute $\delta_\beta$ as a weighted sum of Gaussian hockey-stick values at thresholds $\beta (p_+ / p_-)^{-d}$ (\Cref{thm:gauss-flip-div}).

%% file: sections/06_evaluation.tex
\section{Evaluation}
\label{sec:evaluation}

We present several applications of our framework to data poisoning and backdoor poisoning attacks.
Our experiments focus on image classification on the MNIST~\citep{lecun1998gradient} and CIFAR10~\citep{krizhevsky2009learning} datasets, in addition to the class-subsampled variants MNIST-17 and CIFAR10-02~\citep{zhang2022bagflip,weber2023rab}.
We train ensembles of size $N=10,000$ and $N=1,000$ for MNIST and CIFAR10, respectively.
We use confidence level $\alpha=0.001$.

\subsection{$\ell_p$ feature adversaries}
\label{sec:lp-adversary}
\begin{figure}[ht]
    \centering
    \begin{subfigure}[t]{0.31\linewidth}
        \centering
        \captionsetup{font=scriptsize}
        \includegraphics[width=\linewidth]{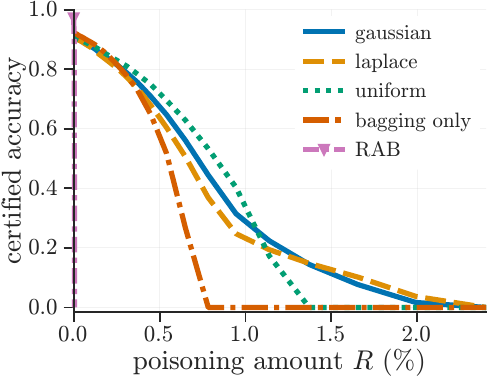}
        \caption{$\ell_1$ adversary, $\epsf=1$.}
        \label{fig:mnist-cleanlabel-l1}
    \end{subfigure}
    \hfill
    \begin{subfigure}[t]{0.31\linewidth}
        \centering
        \captionsetup{font=scriptsize}
        \includegraphics[width=\linewidth]{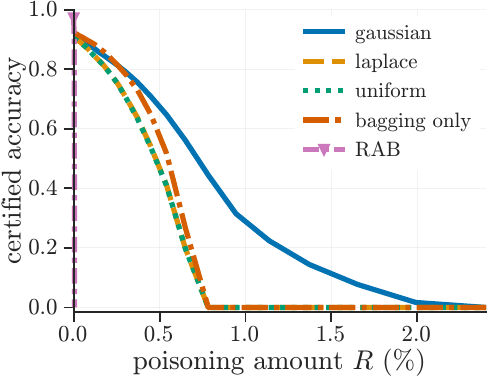}
        \caption{$\ell_2$ adversary, $\epsf=1$.}
        \label{fig:mnist-cleanlabel-l2}
    \end{subfigure}
    \hfill
    \begin{subfigure}[t]{0.31\linewidth}
        \centering
        \captionsetup{font=scriptsize}
        \includegraphics[width=\linewidth]{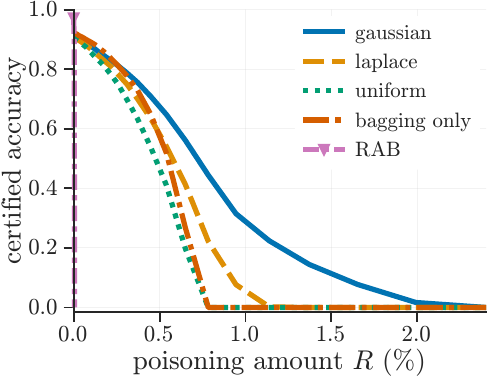}
        \caption{$\ell_\infty$ adversary, $\epsf=1/28$.}
        \label{fig:mnist-cleanlabel-linf}
    \end{subfigure}
    \hfill
    \caption{Certified accuracies on MNIST for clean-label triggerless attacks under different $\ell_p$ adversaries. The budget in each case is calibrated to correspond to the same $\ell_2$ budget, which has the effect that Gaussian noise certifies the same radius across all curves.}
    \label{fig:mnist-cleanlabel-lp}
\end{figure}

\paragraph{Setup}
We consider a setting in which the attacker corrupts at most $\emp$ training examples, each by at most $\norm{x-x'}_{p} \le \epsf$.
The test sample and labels are not corrupted.
We instantiate $\Pi$ as bagging and compare Gaussian, Laplace, and Uniform feature noise.
We also evaluate pure bagging~\citep{jia2021intrinsic} and Gaussian noise without subsampling~\citep{weber2023rab} as baselines.

No prior method makes effective use of both the $\emp$ and $\epsf$ constraints.
For example, \citet{jia2021intrinsic} only subsample and pessimistically assume that a single retained poison compromises the learner, while \citet{weber2023rab} do not subsample data.
\citet{zhang2022bagflip} use discrete noise that only works under an $\ell_0$ constraint, and cannot certify more than bagging in $\ell_p$: while $\ell_p$ constraints can be converted into (looser) $\ell_q$ constraints for $p,q > 0$, they cannot convert into nontrivial $\ell_0$ constraints.

For $p > 0$, we compare with the certificate obtainable from the RAB defense~\citep{weber2023rab}, which adds Gaussian noise to the entire training set without bagging.
RAB's certificate can be used in our threat model by computing the maximum admissible $\ell_2$ perturbation under our threat model.
In particular, an $\ell_p$ bound can be converted into an $\ell_2$ bound by taking $\epsf^{(2)} \le \epsf^{(1)}$ or $\epsf^{(2)} \le \sqrt{d}\epsf^{(\infty)}$, respectively, and with $\emp$ poisoned examples we obtain the full dataset $\ell_2$ bound $\sqrt{\emp} \epsf^{(2)}$.

\paragraph{Findings}
We make several observations on the clean-label poisoning robustness reported in \Cref{fig:mnist-cleanlabel-lp}.
First, pure bagging always attains a higher non-robust certified accuracy, which makes sense since adding noise to data should only reduce classifier accuracy.
Second, uniform noise is competitive only for $\ell_1$ and for moderate poisoning amounts ($< 1.2\%$).
This finding corroborates \citet{yang2020randomized}, who find in the evasion setting that Gaussian noise yields the largest certified accuracies for $\ell_2$ and $\ell_\infty$, whereas Uniform noise works best for $\ell_1$.
Laplace noise is only competitive for $\ell_1$ adversaries at high poisoning rates, where it marginally beats Gaussian noise for $R \ge 1.5$.

We show the dependence on $\epsf$ for MNIST, MNIST-17, CIFAR10, and CIFAR10-02 in \Cref{fig:budgets} when training under Gaussian noise.
In this case, the advantage of constraining $\epsf$ deteriorates quickly: on MNIST, the certified radius is substantially larger for $\epsf=1$, but by $\epsf = 2$ the certificate nearly matches that at $\epsf=\infty$ (corresponding to the bagging certificate of \citet{jia2021intrinsic}).

\subsection{Label flipping and backdoor adversaries}
\label{sec:adversariess}
\paragraph{Setup}
We study the impact of different threat models on certifiable robustness.
We consider two main attacker settings: joint feature-label attacks and joint train-test (backdoor) attacks.
In both settings, the adversary can corrupt at most $\emp$ training examples and, for each corrupted sample, modify features up to $\epsf$ distance in $\ell_2$.
In feature-label attacks, the adversary can additionally corrupt the label.
In backdoor attacks, the adversary also corrupts the test sample with test budget $\epste = \epsf$.

\paragraph{Findings}
We report the robustness tradeoffs for different datasets.
The label-flipping adversary is reported in \Cref{tab:label-flip}, and the backdoor adversary is reported in \Cref{tab:backdoor}.

Our method significantly outperforms RAB under backdoor adversaries.
On MNIST-17, while RAB reaches trivial certified accuracy after $R=0.5\%$, our method retains over 45\% certified accuracy at $R=2\%$.
The improvement is due to bagging: with high probability, we retain only a small number $\cO(k \cdot \emp / n)$ of poisoned samples, whereas RAB retains all poisoned samples.
For RAB to certify comparable poisoned radii, it would need to train with $\sigma = \cO(\epsf \sqrt{\emp})$, which would destroy predictive performance, whereas for our method $\sigma = \cO(\epsf \sqrt{\emp \, k / n})$ still yields large certificates.

\begin{table}[ht]
    \centering
    \footnotesize
    \caption{
    Robustness tradeoff for joint feature \& label adversaries at poisoned fractions $R$, with $\epsl = 1$ and $\ell_2$ budget $\epsf = 1$ (MNIST, MNIST-17) or $0.5$ (CIFAR10, CIFAR10-02).
    Bag sizes $k$ and noise scales $\sigma_a < \sigma_b$ are fixed per dataset (\Cref{tab:configs}).
    Label noise retains the true label with $p=0.9$.
    }
    \setlength{\tabcolsep}{3.5pt}
    \begin{tabular}{l cccc cccc cccc cccc}
        \toprule
        & \multicolumn{4}{c}{\textbf{MNIST}} & \multicolumn{4}{c}{\textbf{MNIST-17}} & \multicolumn{4}{c}{\textbf{CIFAR10}} & \multicolumn{4}{c}{\textbf{CIFAR10-02}} \\
        \cmidrule(lr){2-5} \cmidrule(lr){6-9} \cmidrule(lr){10-13} \cmidrule(lr){14-17}
        $R$ (\%) & $0$ & $0.5$ & $1$ & $1.5$ & $0$ & $1$ & $2.5$ & $3.5$ & $0$ & $0.05$ & $0.1$ & $0.15$ & $0$ & $0.8$ & $1.6$ & $2.4$ \\
        Bagging & \textbf{82.3} & \textbf{56.8} & \textbf{31.9} & \textbf{12.0} & \textbf{99.1} & \textbf{98.0} & 0.0 & 0.0 & \textbf{32.5} & \textbf{14.7} & \textbf{5.3} & \textbf{1.1} & \textbf{78.3} & \textbf{64.4} & \textbf{40.2} & 2.4 \\
        \midrule
        Ours, $\sigma_a$ & 81.2 & 50.0 & 23.5 & 5.1 & 98.7 & 96.6 & \textbf{48.6} & 0.0 & 22.3 & 9.8 & 3.3 & 0.6 & 66.9 & 48.5 & 26.6 & 0.6 \\
        Ours, $\sigma_b$ & 78.7 & 44.8 & 17.9 & 3.7 & 98.8 & 93.0 & 35.3 & \textbf{5.3} & 16.7 & 6.4 & 1.7 & 0.1 & 64.2 & 45.4 & 27.5 & \textbf{3.4} \\
        \bottomrule
    \end{tabular}
    \label{tab:label-flip}
\end{table}

\begin{table}[ht]
    \centering
    \footnotesize
    \caption{
    Robustness tradeoff for clean label backdoor adversaries on MNIST, MNIST-17, CIFAR10, and CIFAR10-02 at different poisoned fractions $R$.
    The $\ell_2$ constraint is $\epsf = \epste = 1$ on MNIST and MNIST-17 and $\epsf = \epste = 0.5$ on CIFAR10 and CIFAR10-02.
    $\sigma_a, \sigma_b$ are as in \Cref{tab:configs}.
    RAB trains on the full training set.
    }
    \setlength{\tabcolsep}{3.5pt}
    \begin{tabular}{lcccccccccccccccc}
        \toprule
        & \multicolumn{4}{c}{\textbf{MNIST}} & \multicolumn{4}{c}{\textbf{MNIST-17}} & \multicolumn{4}{c}{\textbf{CIFAR10}} & \multicolumn{4}{c}{\textbf{CIFAR10-02}} \\
        \cmidrule(lr){2-5} \cmidrule(lr){6-9} \cmidrule(lr){10-13} \cmidrule(lr){14-17}
        $R$ (\%) & $0$ & $0.1$ & $0.5$ & $1$ & $0$ & $0.5$ & $1$ & $2$ & $0$ & $0.02$ & $0.06$ & $0.1$ & $0$ & $0.2$ & $0.4$ & $0.6$ \\
        RAB, $\sigma_a$ & \textbf{74.1} & 0.0 & 0.0 & 0.0 & \textbf{95.9} & 0.0 & 0.0 & 0.0 & 18.4 & 0.0 & 0.0 & 0.0 & 63.5 & 0.0 & 0.0 & 0.0 \\
        RAB, $\sigma_b$ & 66.5 & 0.0 & 0.0 & 0.0 & 94.7 & 0.0 & 0.0 & 0.0 & \textbf{19.3} & 0.0 & 0.0 & 0.0 & \textbf{67.5} & 0.0 & 0.0 & 0.0 \\
        \midrule
        Ours, $\sigma_a$ & 48.6 & \textbf{42.2} & \textbf{20.5} & 5.0 & 92.9 & \textbf{88.0} & \textbf{76.6} & 41.8 & 6.8 & 5.6 & 3.8 & 2.5 & 40.9 & 34.4 & 27.5 & 20.2 \\
        Ours, $\sigma_b$ & 41.9 & 36.6 & 17.8 & \textbf{6.2} & 91.3 & 86.1 & 76.1 & \textbf{45.4} & 8.7 & \textbf{7.9} & \textbf{6.5} & \textbf{5.3} & 54.0 & \textbf{51.4} & \textbf{48.7} & \textbf{45.8} \\
        \bottomrule
    \end{tabular}
    \label{tab:backdoor}
\end{table}

\begin{figure}[ht]
    \centering
    \begin{subfigure}[t]{0.24\linewidth}
        \centering
        \captionsetup{font=scriptsize}
        \includegraphics[width=\linewidth]{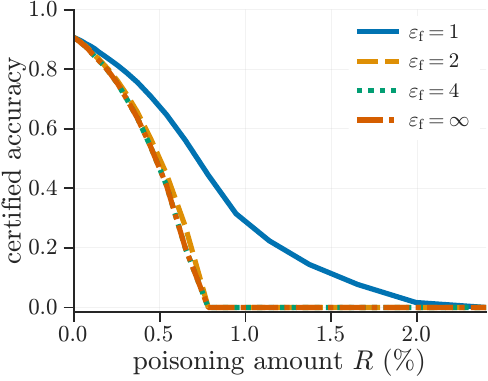}
        \caption{MNIST}
        \label{fig:budget-mnist}
    \end{subfigure}
    \hfill
    \begin{subfigure}[t]{0.24\linewidth}
        \centering
        \captionsetup{font=scriptsize}
        \includegraphics[width=\linewidth]{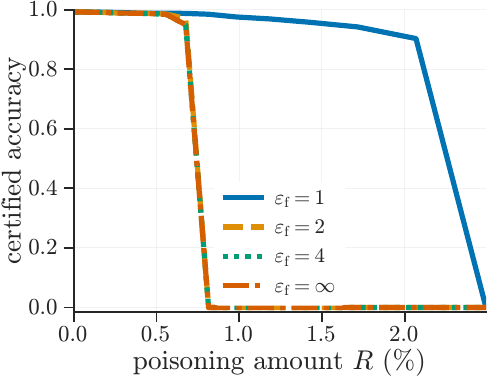}
        \caption{MNIST-17}
        \label{fig:budget-mnist17}
    \end{subfigure}
    \hfill
    \begin{subfigure}[t]{0.24\linewidth}
        \centering
        \captionsetup{font=scriptsize}
        \includegraphics[width=\linewidth]{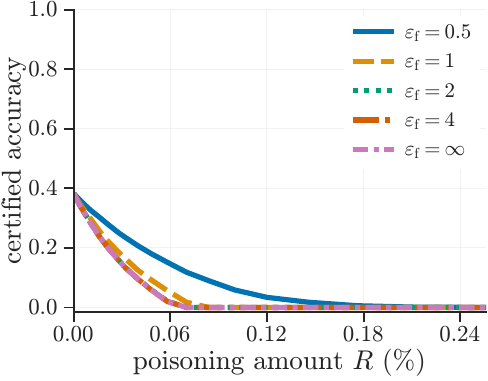}
        \caption{CIFAR10}
        \label{fig:budget-cifar10}
    \end{subfigure}
    \hfill
    \begin{subfigure}[t]{0.24\linewidth}
        \centering
        \captionsetup{font=scriptsize}
        \includegraphics[width=\linewidth]{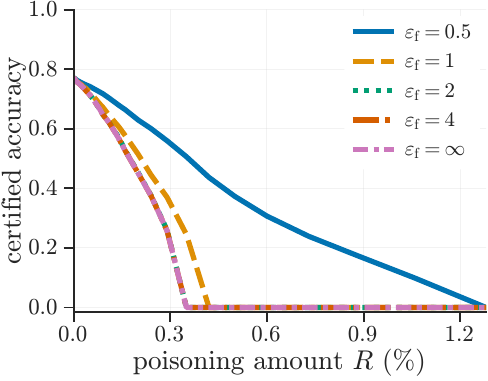}
        \caption{CIFAR10-02}
        \label{fig:budget-cifar10-02}
    \end{subfigure}
    \caption{Certified accuracy versus poisoning fraction $R$ at each per-example budget $\epsf$, under an $\ell_2$ adversary with Gaussian feature noise. $\sigma = 0.5$ for MNIST and $\sigma=0.35$ for CIFAR10.}
    \label{fig:budgets}
\end{figure}

\paragraph{When is more noise useful?}
Aggressive noise erases corruptions, which corresponds to larger certified radii via smaller divergence bounds, but also degrades the base learner, lowering the vote confidence.
In general, two training-data noising strategies exist: either we use aggressive plain bagging (hoping to throw away bad examples), or we try to recover performance from retained bad examples by adding additional noise to them.
It is therefore important to understand how different noises affect the resulting certificates.

As a simple example, consider $\ell_2$ corruptions with Gaussian noise.
To keep the component divergences small, the noise level $\sigma$ must grow with the adversary budget $\epsf$.
If $\sigma$ becomes large, the input becomes impossible to classify, and the base learner accuracy will drop dramatically.
In this regime, it is preferable to use no feature noising at all and rely on pure bagging to provide certificates.

Label noising is especially delicate (\Cref{tab:label-flip}): we only observe improvements over the plain bagging method for binary classification datasets MNIST-17 and CIFAR10-02.
This is intuitive: whereas features can be noised individually while still preserving high information content taken over the full input, the label serves as the only supervision signal for the learning task, so noising it is especially destructive.
Still, label flipping can be effective for sufficiently easy tasks: on MNIST-17, we observe certified radius out to $3.5\%$ of the training dataset achieved by aggressive feature and label noising compared with the plain bagging baseline.

\subsection{Comparison with prior works}
Our method recovers the exact certificates of~\citet{zhang2022bagflip} under all threat models and smoothing strategies they consider (\Cref{apx:tightness-bagflip}, \Cref{fig:bagflip-tightness}), and those of \citet{jia2021intrinsic} for replacement adversaries (\Cref{apx:tightness-bagging}).

\begin{wrapfigure}{r}{0.26\textwidth}
    \vspace{-1.0em}
    \centering
    \includegraphics[width=\linewidth]{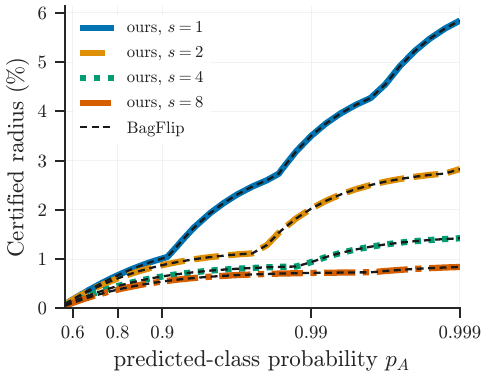}
    \caption{Tightness of our certificate with BagFlip's.}
    \label{fig:bagflip-tightness}
    \vspace{-1em}
\end{wrapfigure}
Another useful finding is that, by writing down the hockey stick divergences in closed-form, we are able to implement certification for a fixed $\epsf$ bound in a much more efficient $\cO(k^2 \epsf^2)$ compared with BagFlip's $\cO(d^3 + k^2 d^2)$ preprocessing cost when $\epsf \ll d$.
Notably, we depend on $d$ only through $\epsf$.
Our formulation is also amenable to parallelized precomputation across all desired $\epsf$, whereas the linear recurrence of \citet{zhang2022bagflip} is not.

Because BagFlip couples the feature space to the label space, it can only certify dirty-label attacks when $\abs{\cX} = \abs{\cY}$ and when labels are flipped with the same distribution as features.
This in particular means it cannot certify dirty-label attacks in the multi-class setting if the features are binary.
Our framework, in contrast, handles both of these cases directly.
We are able to compute certificates in a setting where the label noise differs from the feature flipping noise, or where the number of distinct feature values differs from the number of classes.

%% file: sections/07_conclusion.tex
\section{Conclusion}
\label{sec:conclusion}
We provide a randomized smoothing framework for certified robustness against backdoor and data poisoning attacks.
Our framework decomposes the threat model into independent axes bounding corruption magnitudes, and applies composite noise comprising data subsampling, feature noising, and label flipping.
We reduce certification in this regime to a convex program over $f$-divergence constraints and show how to compute the corresponding constraint sets.
Our constraint computation leverages reorganized joint convexity to efficiently handle the mixture structure induced by subsampling and uses general, composable principles to accommodate a wide range of smoothing measures.
We demonstrate this generality by certifying, for the first time, $\ell_1$, $\ell_2$, and $\ell_\infty$ adversaries on models trained with Gaussian, Laplace, Uniform, and discrete noise.

\paragraph{Limitations and open questions}
Our method trains thousands of base learners for sound prediction estimation.
We mitigate this cost with an efficient parallelized training implementation, but future methods might seek efficient closed-form solutions such as those considered in prior work~\citep{rosenfeld2020certified,weber2023rab}.
Our method does not exploit properties specific to the dataset or base learner.
Future work should consider whether using such information can yield stronger certificates.
An interesting remaining challenge is to precisely characterize the structural properties of optimal noising schemes for complex backdoor poisoning threat models.
For example, we do not understand the tradeoff between subsampling and data noising very well.
Finally, learner-based robustness properties (such as data-smoothness) might be leveraged to obtain even tighter certificates.

%% file: sections/A_smoothing_background.tex
\section{Extended Background on Randomized Smoothing}

\subsection{Tightness of Randomized Smoothing}

Historically, several analytical approaches have been applied to derive sound robustness guarantees under randomized smoothing procedures.
Originally, the differential privacy guarantees of mechanisms like adding Gaussian noise were used to derive sound but slack bounds on the certified radius~\citep{lecuyer2019certified}.
Other, tighter, but still slack bounds were derived by analyzing connection between Renyi divergence and certified robustness \citet{li2019certified}.
\citet{cohen2019certified} were the first to derive provably optimal robustness certificates by analyzing the connection between robustness certification and hypothesis testing.
\citet{salman2019provably} derived the same tight certificate as a consequence of a nonlinear Lipschitz continuity property of Weierstrass transforms.

\citet{dvijotham2020framework} extend the results of ~\citet{li2019certified} to formulate general robustness certification as a convex optimization problem.
This optimization problem involves bounding various statistical divergences between noised versions of the clean and corrupted algorithm inputs.
The general formulation not tight under most choices of divergence.
However, \citet{dvijotham2020framework} demonstrate a connection with hockey-stick divergences which can recover optimal certificates in the general case.
Moreover, hockey-stick divergences are connected to differential privacy (specifically, $(\varepsilon, \delta)$-approximate differential privacy), thereby connecting privacy, convex optimization and hypothesis testing together as tools for randomized smoothing certification.

\subsection{$f$-Divergences}
\label{apx:dvijotham}

In this section we present the certified robustness framework of \citet{dvijotham2020framework}, based on convex programming.
\citet{dvijotham2020framework} present three main results important for our work.
The first result, \Cref{thm:full-info-cvx}, characterizes robustness of a specification under $f$-divergence constraints in terms of a convex program.
The second result, \Cref{thm:lim-info-cvx}, characterizes the robustness over a broader space of specifications, aligning with the more typical certified robustness setting.
The third results, \Cref{thm:tight-certificate-hockey}, demonstrates that under the setting of \Cref{thm:lim-info-cvx}, the Hockey-stick divergences of all orders are sufficient to recover tight certificates in the original primal certification problem.

\begin{dfn}[Specification; robust certification]
    \label{dfn:specification}
    A \emph{specification} is a bounded measurable function $\psi \colon \Omega \to \bbR$.
    We say $\psi$ is \emph{robustly certified at $\nu$ with respect to $\cD$} if $\bbE_{w \sim \nu'} \left[\psi(w)\right] \ge 0$ for all $\nu' \in \cD$.
\end{dfn}

\begin{dfn}[$f$-divergence constraint set]
    \label{dfn:f-divergence-set}
    For convex $f \colon \bbR_+ \to \bbR$ with $f(1) = 0$, the $f$-divergence from $\nu$ to $\nu'$ is $D_f(\nu' \| \nu) \defeq \bbE_{\omega \sim \nu}[f(d\nu'(\omega)/d\nu(\omega))]$.
    Given a collection $\cF$ of functions $f$ each defining an $f$-divergence, and to each $f$ an associated bound $\tau_f \ge 0$, define the $\cF$-constraint set as
    \begin{equation}
        \cD_{\cF} \defeq \textstyle\bigcap_{f \in \cF} \cD_{f} = \{\nu' \in \cP(\Omega) : D_f(\nu' \| \nu) \le \tau_f \text{ for all } f \in \cF \}.
    \end{equation}
\end{dfn}

\begin{thm}[\citet{dvijotham2020framework}]
    \label{thm:full-info-cvx}
    Let $\cD_{\cF,\tau}$ be as in \Cref{dfn:f-divergence-set}.
    For $\lambda = (\lambda_1,\ldots,\lambda_M) \in \bbR_+^M$, define $f_{\lambda}(u) \defeq \sum_{i=1}^M \lambda_i f_i(u)$ and denote its convex conjugate by $f_\lambda^\ast$.
    A specification $\psi$ is robustly certified at $\nu$ with respect to $\cD_{\cF, \tau}$ (in the sense of \Cref{dfn:specification}) \textit{if and only if} the optimal value of the convex program
    \begin{equation}
    \label{eq:full-info-dual}
        \sup_{\lambda \in \bbR_+^M, \kappa \in \bbR} \kappa - \sum_{i=1}^M \lambda_i \tau_i - \bbE_{w \sim \nu}\left[f_\lambda^\ast(\kappa - \psi(w))\right]
    \end{equation}
    is non-negative.
\end{thm}

\begin{thm}[\citet{dvijotham2020framework}]
    \label{thm:lim-info-cvx}
    Let $\cD_{\cF, \tau}$ be as in \Cref{dfn:f-divergence-set}, and let $\theta_a, \theta_b \ge 0$ satisfy $\theta_a + \theta_b \le 1$.
    Define the specification class
    \begin{equation}
        \cG(\theta_a, \theta_b) \defeq \left\{ \psi \colon \cW \to \{-1,0,+1\} \;\middle|\;
        \begin{aligned}
        \Pr_{w \sim \nu}[\psi(w) = +1] &\ge \theta_a \\
        \Pr_{w \sim \nu}[\psi(w) = -1] &\le \theta_b
        \end{aligned}
        \right\}.
    \end{equation}
    Every $\psi \in \cG(\theta_a, \theta_b)$ is robustly certified at $\nu$ with respect to $\cD_{\cF, \tau}$ \emph{if and only if} the optimal value of the convex program
    \begin{equation}
    \label{eq:lim-info-program}
    \begin{aligned}
        \inf_{\zeta_a, \zeta_b, \zeta_c \ge 0} \qquad &\zeta_a - \zeta_b \\
        \text{s.t.} \qquad
        & \zeta_a + \zeta_b + \zeta_c = 1,\\
        & D_{f_i}(\zeta \| \theta) \le \tau_i, \quad{i = 1,\ldots,M}
    \end{aligned}
    \end{equation}
    is non-negative, where $\theta = (\theta_a, \theta_b, 1 - \theta_a - \theta_b)$ and  $\zeta = (\zeta_a, \zeta_b, \zeta_c)$ are interpreted as probability distributions.
\end{thm}

\begin{thm}[\citet{dvijotham2020framework}]
    \label{thm:tight-certificate-hockey}
    Let $\theta_a, \theta_b, \mu, \cD, S$ be as in \Cref{dfn:specification}.
    Define $\tau_\beta \defeq \sup_{\nu \in \cD} \HS{\beta}(\nu \| \mu)$ for all $\beta \ge 0$ and let $\beta_a^\star, \beta_b^\star$ be defined as
    \begin{equation}
        \beta_a^\star, \beta_b^\star = \argmax_{\beta_a \ge \beta_b \ge 0}\quad 1 - \beta_a(1 - \theta_a) - \beta_b \theta_b - (\tau_{\beta_a} + [1 - \beta_a]_+) - (\tau_{\beta_b} + [1 - \beta_b]_+).
    \end{equation}
    Define the constraint set
    \begin{equation}
        \cD_{\mathrm{HS}} \defeq \left\{ \nu \in \cP(\cX) \;\middle|\;
        \begin{aligned}
            \HS{\beta_a^\star}(\nu \| \mu) &\le \tau_{\beta_a^\star} \\
            \HS{\beta_b^\star}(\nu \| \mu) &\le \tau_{\beta_b^\star}
        \end{aligned}
        \right\}.
    \end{equation}
    Then, $S$ is information-limited robustly certified at $\mu$ with respect to $\cD$ \emph{if and only if} $S$ is information-limited robustly certified at $\mu$ with respect to $\cD_{\mathrm{HS}}$.
    Thus, the optimal information-limited certificate in this case can be obtained by applying the information-limited certificate of \Cref{thm:lim-info-cvx} to $\cD_{\mathrm{HS}}$.
\end{thm}

In our setting $\mu = \cM(w)$ and $\cD = \cD_{w,\rho}$, so the quantity in \Cref{thm:tight-certificate-hockey} is the hockey-stick profile
\begin{equation}
    \label{eq:hs-profile}
    \tau^\star_\beta \defeq \sup_{w' \in \cB_\rho(w)} \Hock[\beta](\cM(w') \,\|\, \cM(w)), \qquad \beta \ge 0.
\end{equation}

%% file: sections/A_divergence_background.tex
\section{Background on $f$-divergences}

In this section we formally define $f$-divergences and state some useful facts for computing with them.
We state here only the most essential results for understanding our derivations; a more complete treatment can be found in, e.g., \citet{csiszar2004information,liese2006divergences}
or \citet[ch.~7]{polyanskiy2025information}.
Additionally, we state and prove a few additional useful lemmata.

\begin{dfn}[$f$-divergence]
    \label{dfn:f-divergence}
    Let $f \colon \bbR^+ \to \bbR$ be convex with $f(1) = 0$, and take by convention
    \begin{equation*}
        f(0) \defeq \lim_{x \downarrow 0} f(x),
        \qquad
        f'(\infty) \defeq \lim_{x \downarrow 0} x f(1/x),
        \qquad
        0 \cdot f'(\infty) \defeq 0.
    \end{equation*}
    Let $\mu, \nu$ be two probability measures on a measurable space $(\Omega, \cA)$.
    If $\mu \ll \nu$ then the \emph{$f$-divergence from $\nu$ to $\mu$} is defined as
    \begin{equation}
        \label{eq:f-divergence-simple}
        D_f(\mu \| \nu) \defeq \textstyle\bbE_{\nu}\left[f\left(\frac{d\mu}{d\nu}\right)\right]
    \end{equation}
    where $\tfrac{d\mu}{d\nu}$ is a Radon--Nikodym derivative.
    More generally, let $\lambda$ be any common dominating measure.
    Then, writing $p \defeq \tfrac{d\mu}{d\lambda}$ and $q \defeq \tfrac{d\nu}{d\lambda}$, we may define
    \begin{equation}
        \label{eq:f-divergence-full}
        D_f(\mu \| \nu) \defeq \int_{q > 0} q(x) f\left(\frac{p(x)}{q(x)}\right) d\lambda + f'(\infty) \mu(q = 0).
    \end{equation}
    \Cref{eq:f-divergence-full} does not depend on the choice of common dominating measure $\lambda$, and moreover when $\mu \ll \nu$ gives the same result as \Cref{eq:f-divergence-simple}.
    We call the function $f$ the \emph{generator} of the $f$-divergence $D_f$.
\end{dfn}

\begin{table}[ht]
    \centering
    \renewcommand{\arraystretch}{1.5}
    \small
    \caption{$f$-divergences and associated generators and integral form, with $p \defeq \tfrac{d\mu}{d\lambda}$ and $q \defeq \tfrac{d\nu}{d\lambda}$.}
    \begin{tabular}{@{}llll@{}}
        \toprule
        $f$-Divergence & Regime & Generator $f(u)$ & Integral form \\
        \midrule
        Total variation $\TV(\mu \| \nu)$
            &
            & $\tfrac{1}{2}\abs{u - 1}$
            & $\frac{1}{2} \int \abs{p - q} \,d\lambda$ \\
        Hockey-stick $\HS{\beta}(\mu \| \nu)$
            & $\beta \ge 0$
            & $(u - \beta)_+ - (1 - \beta)_+$
            & $\int (p - \beta q)_+ \,d\lambda - (1 - \beta)_+$ \\
        Kullback--Leibler $\KL(\mu \| \nu)$
            &
            & $u \log u$
            & $\int p \log \frac{p}{q} \,d\lambda$ \\
        Squared Hellinger $\Hsq(\mu \| \nu)$
            &
            & $\tfrac{1}{2}(\sqrt{u} - 1)^2$
            & $1 - \int \sqrt{pq} \,d\lambda$ \\
        Hellinger $\Hel[\alpha](\mu \| \nu)$
            & $\alpha \in (0, 1) \cup (1, \infty)$
            & $\frac{u^\alpha - 1}{\alpha - 1}$
            & $\frac{1}{\alpha - 1} (\int p^\alpha q^{1-\alpha} \,d\lambda - 1)$ \\
        Pearson $\chiP(\mu \| \nu)$
            &
            & $(u - 1)^2$
            & $\int \frac{(p - q)^2}{q}\, d\lambda$ \\
        Neyman $\chiN(\mu \| \nu)$
            &
            & $\frac{(u - 1)^2}{u}$
            & $\int \frac{(p - q)^2}{p} \,d\lambda$\\
        \bottomrule
    \end{tabular}
    \label{tab:f-divergences}
\end{table}

\paragraph{Unnormalized hockey-stick divergences}
We write
\begin{equation}
    \HockU[\beta](\mu \| \nu) \defeq \int (p - \beta q)_+ \, d\lambda = \Hock[\beta](\mu \| \nu) + (1 - \beta)_+ =\sup_A \mu(A) - \beta\, \nu(A).
\end{equation}

\paragraph{$f$-divergences from likelihood ratios}
If we define the likelihood and log-likelihood ratios
\begin{equation}
    \label{eq:likelihood-ratio}
    \Lambda \defeq p / q \in [0, \infty], \qquad L \defeq \log \Lambda \in [-\infty, +\infty],
\end{equation}
we can rewrite \Cref{eq:f-divergence-full} in the form
\begin{equation}
    \label{eq:f-divergence-lambda}
    \Df(\mu \| \nu) = \textstyle\bbE_{\nu}[f(\Lambda)] + f'(\infty)\mu(\Lambda = 0).
\end{equation}

\paragraph{R\'enyi divergences and the Hellinger integral}
The generator $f_\alpha(x) \defeq \tfrac{x^\alpha - 1}{\alpha - 1}$ is convex for every $\alpha \in (0,1) \cup (1,\infty)$, and defines the \emph{Hellinger divergence of order $\alpha$}, $\Hel[\alpha](\mu \| \nu) \defeq \Df[f_\alpha](\mu \| \nu)$.
The \emph{Hellinger integral of order $\alpha$} is
\begin{equation}
    \label{eq:hellinger-integral}
    \Hint[\alpha](\mu \| \nu) \defeq \int p^\alpha q^{1 - \alpha} \, d\lambda = 1 + (\alpha - 1) \, \Hel[\alpha](\mu \| \nu),
\end{equation}
with the convention that $\Hint[\alpha] = \infty$ if $\alpha \ge 1$ and $\mu \not\ll \nu$.
The \emph{R\'enyi divergence of order $\alpha$} is related via the monotone transformation:
\begin{equation}
    \label{eq:renyi}
    \Ren[\alpha](\mu \| \nu) \defeq \frac{1}{\alpha - 1} \log \Hint(\mu \| \nu) = \frac{1}{\alpha - 1} \log\left( 1 + (\alpha - 1) \Hel(\mu \| \nu) \right).
\end{equation}
Note $Z_\alpha$ is \emph{not}, strictly speaking, an $f$-divergence, since $x \mapsto x^\alpha$ does not vanish at $1$.
Writing $L \defeq \log \tfrac{d\mu}{d\nu}$ for the log-likelihood ratio, $Z_\alpha(\mu \| \nu) = \bbE_\mu\left[ e^{(\alpha - 1) L} \right]$, which is the moment generating function of $Z$ under $\mu$ evaluated at $\alpha - 1$.
By taking limits, one may also define the R\'enyi divergences at orders $\alpha \in \{0, 1, \infty\}$:
\begin{align*}
    \Ren[0](\mu \| \nu) &\defeq -\log \nu(\{p > 0\}), \quad p = d\mu / d\lambda, \\
    \Ren[1](\mu \| \nu) &\defeq \KL(\mu \| \nu), \\
    \Ren[\infty](\mu \| \nu) &\defeq \log \esssup_\nu \tfrac{d\mu}{d\nu}.
\end{align*}

In general, it is more convenient to state the associated R\'enyi divergence $\Ren[\alpha]$ rather than the raw Hellinger divergence $\Hel[\alpha]$, due to the nice tensorization and continuity properties of $\Ren[\alpha]$ (see \Cref{lem:tensorization}).

\paragraph{Special cases}
The most useful $f$-divergences are the hockey stick divergences (due to their association with tight certificates via \Cref{thm:tight-certificate-hockey}) and the R\'enyi divergences (due to their attractive computational properties; see, e.g., \Cref{lem:tensorization}).
Moreover, several of the divergences in \Cref{tab:f-divergences} can be expressed in terms of a hockey stick or R\'enyi divergence of an appropriate order.
We have the following relationships: $\TV = \HS{1}$, $\Hel[2] = \chiP$, $\Hsq = \tfrac{1}{2} \Hel[1/2] = 1 - \Hint[1/2]$.

\paragraph{Practical considerations}
Practically speaking, the relationships between divergences in \Cref{tab:f-divergences} mean we should utilize the Hockey-stick divergences directly whenever possible, and otherwise simply compute the R\'enyi divergences of all orders (which is useful when tight Hockey-stick divergences are inaccessible).
Other specific divergences, like $\Hsq$ and $\chiP$, restate information already encoded in the R\'enyi divergences of the appropriate orders.

\begin{dfn}[Perspective function]
    \label{dfn:perspective}
    Let $f$ be a generator as in \Cref{dfn:f-divergence} with the usual conventions for $f(0)$ and $f'(\infty)$.
    The \emph{perspective} of $f$ is the function $g_f \colon \bbR^{+} \times \bbR^{+} \to \bbR \cup \left\{ +\infty\right\}$ given by
    \begin{equation}
        \label{eq:perspective}
        g_f(a, b) \defeq
        \begin{cases}
            b f(a / b) &b > 0, \\
            a f'(\infty) & b = 0, a > 0, \\
            0 & a = b = 0.
        \end{cases}
    \end{equation}
\end{dfn}

\begin{lem}[Properties of perspective functions]
    \label{lem:perspective-properties}
    $g_f$ is jointly convex and positively homogeneous, i.e., $g_f(ta, tb) = t g_f(a, b)$ for every $t > 0$.
    Consequently $g_f$ is subadditive:
    \begin{equation}
        g_f (a_1 + a_2, b_1 + b_2) \le g_f(a_1, b_1) + g_f(a_2, b_2).
    \end{equation}
    Moreover $g_f$ admits the representation
    \begin{equation}
        \label{eq:perspective-support}
        g_f(a, b) = \sup_{s \in \dom f^{\ast}} [s a - f^{*}(s) b]
    \end{equation}
    where $f^{\ast}$ is the convex conjugate of $f$, and for $b > 0$ the supremum is attained exactly at those $s \in \partial f (a / b)$.
\end{lem}

\begin{dfn}[Convex Conjugate]
    \label{dfn:convex-conjugate}
    For a function $f \colon \bbR \to \bbR$, its \emph{convex conjugate} is defined by $f^\ast(u) \defeq \sup_{v}(uv - f(v))$.
\end{dfn}

\begin{lem}[Basic properties of $f$-divergences]
    \label{lem:f-divergence-properties}
    Let $\mu, \nu, \rho$ be probability measures and $f$ a generator as in \Cref{dfn:f-divergence}.
    \crefalias{enumi}{lem}
    \begin{enumerate}[label=(\roman*), ref=\thelem(\roman*)]
        \item \label{lem:f-nonneg} \emph{(Nonnegativity.)} $D_f(\mu \| \nu) \ge 0$.
        If $f$ is strictly convex at 1, then $D_f(\mu \| \nu) = 0$ if and only if $\mu = \nu$.
        \item \label{lem:f-reversal} \emph{(Reversal.)} With the conjugate generator $\overline{f}(x) \defeq x f(1/x)$, one has $D_{\overline{f}}(\mu \| \nu) = D_f(\nu \| \mu)$.
        \item \label{lem:f-cancel} \emph{(Cancellation of common parts.)} $D_f(\mu \otimes \rho \| \nu \otimes \rho) = D_f(\mu \| \nu)$.
        \item \label{lem:f-convex} \emph{(Joint convexity.)} Let $0 < \lambda < 1$ and suppose $\mu = \lambda \mu_0 + (1 - \lambda) \mu_1$ and $\nu = \lambda \nu_0 + (1 - \lambda) \nu_1$ for probability measures $\mu_i,\nu_i$. Then $D_f(\mu \| \nu) \le \lambda D_f(\mu_0 \| \nu_0) + (1 - \lambda) D_f(\mu_1 \| \nu_1)$.
        \item \label{lem:f-perspective} \emph{(Perspective representation.)} $D_f(\mu \| \nu) = \int g_f (p, q)\,d\lambda$ for common dominating $\lambda$.
        \item \label{lem:f-variational} \emph{(Variational representation.)} With $f^\ast$ the convex conjugate of $f$,
        \begin{equation*}
            D_f(\mu \| \nu) = \sup_{g \colon \Omega \to \dom(f^\ast)} \left\{\bbE_{\mu} [g] - \bbE_{\nu} [f^{\ast} \circ g]\right\}.
        \end{equation*}
    \end{enumerate}
\end{lem}

\begin{lem}[Data-processing inequality]
    \label{lem:data-processing}
    Let $K$ be a Markov kernel from $(\Omega, \cA)$ to $(\Omega', \cA')$, and write $K\mu$ for the pushforward $A \mapsto \int K(\omega, A) \, \mu(d\omega)$.
    Then for every generator $f$,
    \begin{equation}
        \label{eq:dpi}
        D_f(K \mu \| K \nu) \leq D_f(\mu \| \nu).
    \end{equation}
    Equality holds whenever $K$ is induced by a bimeasurable bijection
    $T \colon \Omega \to \Omega'$, so that $D_f(T_\# \mu \| T_\# \nu) = D_f(\mu \| \nu)$, and more generally whenever $K$ is a sufficient statistic for $\{\mu, \nu\}$.
\end{lem}

\begin{lem}[Reversal of Hockey-Stick divergences]
    \label{lem:hockey-stick-reversal}
    Fix $\mu, \nu$ probability measures on $(\Omega, \cA)$, and $\beta > 0$.
    Define the dual parameter $\overline{\beta} \defeq 1 / \beta$, extended to $\beta \in \{0, \infty\}$ in the usual way.
    Then one has the following relationship:
    \begin{equation}
        \label{eq:hockey-stick-reversal}
        \HockU[\beta](\mu \| \nu) = \beta\, \HockU[\bar\beta](\nu \| \mu) + 1 - \beta,
        \qquad
        \Hock[\beta](\mu \| \nu) = \beta\, \Hock[\bar\beta](\nu \| \mu).
    \end{equation}
    In particular, the TV distance $\TV$ is symmetric since it corresponds to the self-dual $\beta = 1$.
    Moreover, the Hockey-Stick profile for $(\mu, \nu)$ over $0 \le \beta \le 1$ is entirely determined by the profile of the reversed pair $(\nu, \mu)$ for $\beta \ge 1$.
\end{lem}

\begin{lem}[Integral representation by Hockey-stick divergences]
    \label{lem:hs-integral-representation}
    Let $f$ be a generator as in \Cref{dfn:f-divergence}, finite on $(0, \infty)$, and let $\rho$ denote the Lebesgue-Stieltjes measure of its right derivative $f'_+$, that is, the unique Borel measure on $(0, \infty)$ with $\rho((a, b]) = f'_+(b) - f'_+(a)$ for $0 < a < b < \infty$.
    Then for all $\mu \ll \nu$,
    \begin{equation}
        \label{eq:hs-integral-representation}
        \Df(\mu \| \nu) = \int_{(0, \infty)} \Hock[\beta](\mu \| \nu)\, \rho(d\beta).
    \end{equation}
\end{lem}

\begin{lem}[Tensorization Properties for $f$-Divergences]
    \label{lem:tensorization}
    For product measures $\mu = \bigotimes_{i=1}^n \mu_i$ and $\nu = \bigotimes_{i=1}^n \nu_i$ on the same product space, $\Hint[\alpha](\mu \| \nu) = \prod_{i=1}^n \Hint[\alpha](\mu_i \| \nu_i)$ for every $\alpha \in (0, 1) \cup (1, \infty)$.
    Furthermore, $\Ren[\alpha](\mu \|\nu) = \sum_{i=1}^n \Ren[\alpha](\mu_i \| \nu_i)$.
\end{lem}

\begin{table}[ht]
    \centering
    \renewcommand{\arraystretch}{1.5}
    \small
    \caption{Tensorization of $f$-divergences under product measures.
    Throughout, $\mu = \bigotimes_1^n \mu_i$ and $\nu = \bigotimes_{i=1}^n \nu_i$ are product measures on a common product space.}
    \label{tab:tensorization}
    \begin{tabular}{@{}llll@{}}
        \toprule
        Divergence & & Tensorization & Notes \\
        \midrule
        Hellinger integral
            & $\Hint[\alpha](\mu \| \nu)$
            & $= \prod_{i=1}^n \Hint[\alpha](\mu_i \| \nu_i)$
            & $\alpha \in (0,1) \cup (1, \infty)$ \\
        R\'enyi
            & $\Ren[\alpha](\mu \| \nu)$
            & $= \sum_{i=1}^n \Ren[\alpha](\mu_i \| \nu_i)$
            & $\alpha \in [0, \infty]$ \\
        Kullback--Leibler
            & $\KL(\mu \| \nu)$
            & $= \sum_{i=1}^n \KL(\mu_i \| \nu_i)$
            & case $\alpha = 1$ above \\
        Squared Hellinger
            & $\Hsq(\mu \| \nu)$
            & $= 1 - \prod_{i=1}^n \bigl(1 - \Hsq(\mu_i \| \nu_i)\bigr)$
            & case $\alpha = 1/2$ above \\
        Pearson $\chi^2$
            & $\chiP(\mu \| \nu)$
            & $= \prod_{i=1}^n \bigl(1 + \chiP(\mu_i \| \nu_i)\bigr) - 1$
            & case $\alpha = 2$ above \\
        Neyman $\chi^2$
            & $\chiN(\mu \| \nu)$
            & $= \prod_{i=1}^n \bigl(1 + \chiN(\mu_i \| \nu_i)\bigr) - 1$
            & \Cref{lem:f-reversal} w/ $\chiP$ \\
        Hockey-stick
            & $\Hock[\beta](\mu \| \nu)$
            & $\le 1 - \prod_{i=1}^n \bigl(1 - \HS{\beta_i}(\mu_i \| \nu_i)\bigr)$
            & $\beta = \prod_i \beta_i$, $\beta_i \ge 1$ \\
        Total variation
            & $\TV(\mu \| \nu)$
            & $\le 1 - \prod_{i=1}^n \bigl(1 - \TV(\mu_i \| \nu_i)\bigr)$
            & \\
        \bottomrule
    \end{tabular}
\end{table}

As a result of \Cref{lem:tensorization}, the divergences of \Cref{tab:f-divergences} obtained by manipulating a R\'enyi-type divergence themselves possess nice tensorization properties.
We document these and more tensorization properties in \Cref{tab:tensorization}.
In addition, one has upper bounds on product measures for Hockey-stick and TV divergences, although these bounds are not tight in general.

\subsection{Additional lemmata}
In this section we state and prove some additional useful lemmata regarding various $f$-divergences.

\begin{lem}[$f$-divergence between measures with shared components]
    \label{lem:f-common-component}
    Let $\mu, \nu$ be probability measures on $(\Omega, \cA)$, let $\lambda$ be any dominating measure, and write $p \defeq d\mu / d\lambda$ and $q \defeq d\nu / d\lambda$.
    Define the overlap $S \defeq \{p > 0\} \cap \{q > 0\}$ (up to a $\lambda$-null set).
    Suppose $\mu(E) = \nu(E)$ for every measurable $E \subseteq S$.
    Then, denoting $m \defeq \mu(S^c) = \nu(S^c)$, one has $m = \TV(\mu \| \nu)$ and moreover for every generator $f$ as in \Cref{dfn:f-divergence},
    \begin{equation*}
        \Df[f](\mu \| \nu) = (f(0) + f'(\infty))m = (f(0) + f'(\infty))\TV(\mu \| \nu).
    \end{equation*}
\end{lem}
\begin{proof}
    Taking $E = S$ gives $\mu(S) = \nu(S)$ and hence $\mu(S^c) = \nu(S^c)$.
    Set
    \begin{equation*}
        S_\mu \defeq \{p > 0, q = 0\}, \qquad S_\nu \defeq \{p = 0, q > 0\}, \qquad N \defeq \{p = q = 0\}
    \end{equation*}
    so that $\Omega = S \sqcup S_\mu \sqcup S_\nu \sqcup N$ with $\mu(N) = \nu(N) = 0$.
    Evaluating the divergence via \Cref{lem:f-perspective} gives
    \begin{align*}
        \Df(\mu \| \nu) &= \int_S g_f(p, q) \,d\lambda + \int_{S_\mu} g_f(p, 0) \,d\lambda + \int_{S_\nu} g_f(0, q) \,d\lambda + \int_{N} g_f(0,0) \,d\lambda \\
        &= \int_S q f(1) \,d\lambda + \int_{S_\mu} p f'(\infty) \,d\lambda + \int_{S_\nu} q f(0) \,d\lambda + 0 \\
        &= f'(\infty)\mu(S_\mu) + f(0) \nu(S_\nu).
    \end{align*}
    To conclude, note $\mu(S_\mu) = \nu(S_\nu) = m$, and take $f(u) = \tfrac{1}{2}\abs{u - 1}$ for which $f(0) = f'(\infty) = \tfrac{1}{2}$.
\end{proof}
\begin{cor}[Hockey-stick envelope for measures with common component]
    In the setting of \Cref{lem:f-common-component}, for every $\beta \ge 0$,
    \begin{equation*}
        \Hock[\beta](\mu \| \nu) = \min(1, \beta)m,
    \end{equation*}
    where $m = \mu(S^c) = \nu(S^c)$.
    In particular, $\Hock[\beta](\mu \| \nu) = \TV(\mu \| \nu) = m$ whenever $\beta \ge 1$, independently of $\beta$.
\end{cor}
\begin{proof}
    The generator $f_\beta(u) = (u - \beta)_+ - (1 - \beta)_+$ has $f_\beta(0) = -[1 - \beta]_+$ and $f'_\beta(\infty) = \lim_{x \downarrow 0} x f_\beta(1 / x) = 1$, so $f_\beta(0) + f'_\beta(\infty) = 1 - [1 - \beta]_+ = \min(1, \beta)$.
\end{proof}

The significance of \Cref{lem:f-common-component} is that for distributions exhibiting ``simple'' likelihood ratio behavior, the TV distance encodes the entire hockey stick envelope (indeed, \emph{any} $f$-divergence), and in particular this means via \Cref{thm:tight-certificate-hockey} that the information-limited certificate is strictly a function of the TV distance between those distributions.

%% file: sections/A_detailed_prior_work.tex
\section{Detailed Comparison with Prior Work}
\label{apx:prior-work}

\subsection{Threat model and mechanism comparison}

\Cref{tab:threat-model-comparison} expresses the threat model and smoothing mechanism of several prior certified poisoning defenses in our notation and describes whether our framework recovers their certificate.

\paragraph{RAB}
RAB bounds a single aggregate $\ell_2$ perturbation applied to the entire training dataset.
Strictly speaking, our threat model of \Cref{sec:threat-model} does not capture theirs, since in RAB an adversary may spread their corruption over all training examples.
Given our threat model with $\emp$ corrupted examples and $\ell_2$ budget $\epsf$, an outer RAB approximation may be obtained by taking the aggregate bound $\epsf \sqrt{\emp}$.
In the case of backdoor attacks we can arrive at a similar bound.
This is how we use their bound in \Cref{sec:evaluation}.
Under their own threat model, our certification procedure still recovers their certificate if we compute the divergences directly: the noise induces a single Gaussian pair (\Cref{dfn:gaussian-measure}) at distance $\norm{\Delta}_2 / \sigma$ with closed form hockey-stick divergence (\Cref{thm:gaussian-measure-divergence}).
This certificate reduces to the closed form of \citet{cohen2019certified}.

\paragraph{Bagflip}
BagFlip's adversary is mostly recoverable under our threat modeling.
The only nuance is in their treatment of label-flipping adversaries.
First, they require that the finite label space and feature space coincide, $\cX = \cY = [K]$.
Then, they treat the label as an extra feature and reduce certification to the feature-only case.
In particular, a single budget is shared between the features and label.
In this case, we can technically recover their threat model by taking $\rho = (\emp, s, 0, \epste)$ on the ``merged'' feature-label space.

If we start with our label-flipping threat model, $\rho = (\emp, \epsf, 1, \epste)$, we might ask how the resulting certificate compares to BagFlip's (which would take $s = \epsf + 1$).
First, observe that $\rho$ is more restrictive in the sense that our adversary \textit{must} place one of its corruptions on the label, or not at all, whereas BagFlip's adversary could reallocate a label corruption into an extra feature corruption.
Despite this, it turns out that the resulting certificates are identical.
Since features and labels share the alphabet $[K]$ and the flipping kernel, the divergences depend only on the number of changed coordinates (\Cref{thm:flip-measure-divergence}), and both worst cases corrupt $\epsf + 1$ coordinates.

This simplification fails if $\cX \ne \cY$ or if the feature and label noises differ (for example, to enable stronger noise on features than labels), in which case our framework can still compute tight certificates and is notably more general than \citet{zhang2022bagflip}.

\begin{table}[ht]
    \centering
    \small
    \begin{threeparttable}
    \caption{Threat models and smoothing mechanisms of prior randomized smoothing works expressed in the notation of \Cref{sec:problem-statement,sec:methodology}}
    \begin{tabular}{@{}llll@{}}
        \toprule
        Work & Threat vector $\rho$ & Mechanism $(\Pi, \phi_{\cX}, \phi_{\cY})$ & Certificate \\
        \midrule
        \citet{jia2021intrinsic} & $(\emp, \infty, 1, 0)$ & bagging, none, none & identical (\Cref{thm:bagging-exact})\\
        \citet{rosenfeld2020certified} & $(\emp, 0, 1, 0)$ & full set, none, flipping & tight\tnote{a}  \\
        \citet{zhang2022bagflip}  & $(\emp, s, \epsl, \epste)$ & bagging, flipping, flipping & identical (\Cref{thm:bagflip-recovery}) \\
        \citet{weber2023rab} & aggregate $\ell_2$ on $S$ & full set, Gaussian, none & identical\tnote{b} \\
        \bottomrule
    \end{tabular}
    \label{tab:threat-model-comparison}
    \begin{tablenotes}
        \footnotesize
        \item[a] We can obtain provably tight certificates, but must estimate the vote probabilities with an ensemble. \citet{rosenfeld2020certified} give an efficient, loose closed form for learners trained with linear regression.
        \item[b] Under their threat model, the certificate reduces to the standard Gaussian certificate of \citet{cohen2019certified}, which is already understood in \citet{dvijotham2020framework}.
    \end{tablenotes}
    \end{threeparttable}
\end{table}

\subsection{Comparison with \citet{saxena2026provable}}
\label{apx:saxena}

\citet{saxena2026provable} develop a randomized smoothing-based backdoor defense from the DP perspective.
They view the joint train-inference procedure as a randomized mechanism composed of a stack of two private algorithms.

We note several differences between our works.
We do not make any assumptions about properties of the learner, like privacy, instead placing assumptions on, and then noising, the input data.
It is these data noising schemes we analyze in this work, using techniques to handle subsampling, noise products, and then the individual noise factors.
In contrast, the analysis of \citet{saxena2026provable} is placed on understanding the composition of DP mechanisms.
Our threat modeling is more expressive in its per-example budgets for replacement adversaries, handling arbitrary $\ell_p$ constraints and label corruptions, whereas \citet{saxena2026provable} focuses primarily on symmetric set distance adversaries (corresponding to the standard adjacency relation notion in DP), with one notable exception being a special case for $\ell_2$ feature adversaries and Gaussian noise.

One interesting connection between our work and the methods of \citet{saxena2026provable} is in the use of the hockey-stick profile, also called the privacy profile in the DP literature, to characterize the robustness condition.
Whereas \citet{saxena2026provable} use the language and tools of DP, we formulate our results here in the language of $f$-divergences and convex optimization.

To demonstrate the connection, we derive one of the results of \citet{saxena2026provable} in our language.
\citet{saxena2026provable} consider a ``randomized DPA'' mechanism (motivated by \citep{levine2021deep}) which works as follows.
DPA assigns each training record to one of $N$ buckets by some fixed rule (e.g., index or hashing) and trains one model per bucket.
In the randomized view, the mechanism samples one bucket uniformly at random and trains a model on it.
In our notation, this means taking $\Pi$ to be uniform over $N$ disjoint buckets with no feature or label noise, and the test input to be noised by kernel $\phi_{\cX}$.

\begin{thm}[DPA with test-time noise; Theorem D.4 of \citet{saxena2026provable}]
    \label{thm:saxena}
    Let $\Pi$ be uniform over pairwise disjoint index buckets $U_1, \ldots, U_N$.
    Let noise $\phi_{\cX}$ be applied only to the test sample.
    Suppose the corruption changes records in at most $r$ buckets, and let $\tau_\beta^{\mathrm{te}} \ge \Hock[\beta](\phi_{\cX}(x'_{\mathrm{te}}) \| \phi_{\cX}(x_{\mathrm{te}}))$ for all admissible test corruptions.
    Then, for every $\beta > 0$,
    \begin{equation}
        \Hock[\beta](\cM(w') \| \cM(w)) \le \tfrac{r}{N}\min(1, \beta) + (1 - \tfrac{r}{N})\tau_{\beta}^{\mathrm{te}}
    \end{equation}
    Equality holds if exactly $r$ buckets are changed, the test perturbation attains $\tau_{\beta}^{\mathrm{te}}$, and the sub-datasets $S_{U_1}, \ldots, S_{U_N}$ and $S'_{U_j}$ (for corrupted $j$) are pairwise distinct.
\end{thm}
\begin{proof}
    Apply \Cref{thm:joint-convexity-bound} with two groups corresponding to corrupted and uncorrupted buckets.
    For an uncorrupted bucket, the component divergence reduces to divergence between test components $\tau_{\beta}^{\mathrm{te}}$ (\Cref{lem:f-cancel}).
    For a corrupted bucket, the component pair is singular and its hockey stick is simply $\min(1, \beta)$.
    Combining these terms with the mixing weights $\tfrac{r}{N}$ and $1 - \tfrac{r}{N}$ gives the bound.
    If the equality conditions hold, the component supports are disjoint, and the bound is attained by \Cref{cor:identifiable}.
\end{proof}

The advantage of the randomized DPA mechanism is that, because the partition size $N$ is small, and because no feature or label noise is added to training data, the training can be done once for all buckets and re-used for arbitrarily many certifications (inference on a sampled model must still be performed).
Another interesting observation is that DPA can be used to obtain bounds for symmetric set difference adversaries who can add or remove training samples.
For example, if bucket assignment is performed by hashing the training example (as done in \citet{levine2021deep}), one has $r \le \min(2 \emp, N)$ for $\emp$ replaced training examples.

%% file: sections/A_proofs.tex
\section{Deferred proofs}
\label{apx:deferred-proofs}

\subsection{Equivalence of robustness conditions}
\begin{thm}[Robustness characterization]
    \label{thm:robustness-equivalence}
    Fix base learner $G$, smoothing mechanism $\cM$, input $w \in \cW$, and threat vector $\rho$.
    Write $c \defeq \smooth{g_{\cM}}(w)$.
    For each $c' \in \cY$, define the specification $\psi_{c,c'} \colon \widetilde{\cW} \to \bbR$ by $\psi_{c,c'}(\widetilde{w}) \defeq G(\widetilde{w})_c - G(\widetilde{w})_{c'}$.
    Then the following are equivalent:
    \begin{enumerate}
        \item $c \in \argmax_{y \in \cY} \smooth{G_{\cM}}(w')_y$ for every $w' \in \cB_{\rho}(w)$.
        \item For every $c' \ne c$, the specification $\psi_{c, c'}$ is robustly certified at $\cM(w)$ with respect to
        $\cD_{w,\rho} \defeq \{\cM(w') \;:\; w' \in \cB_{\rho}(w)\}$.
    \end{enumerate}
    In particular, either condition implies $\rho$-robustness of $\smooth{g_{\cM}}$ at $w$ whenever the certified margins are strict.
\end{thm}
\begin{proof}
    Fix $w' \in \cB_\rho(w)$ and $c, c' \in \cY$.
    By linearity of expectation together with the definition of the smoothed classifier,
    \begin{align}
        \label{eq:margin-identity}
        \bbE_{\widetilde{w} \sim \cM(w')}\left[\psi_{c,c'}(\widetilde{w})\right] &= \bbE_{\widetilde{w} \sim \cM(w')}\left[G(\widetilde{w})_c - G(\widetilde{w})_{c'}\right] \nonumber \\
        &= \bbE_{\widetilde{w} \sim \cM(w')}\left[G(\widetilde{w})_c\right] - \bbE_{\widetilde{w} \sim \cM(w')}\left[G(\widetilde{w})_{c'}\right] \nonumber \\
        &= \smooth{G_{\cM}}(w')_c - \smooth{G_{\cM}}(w')_{c'}.
    \end{align}
    Suppose (1) holds; that is, for every $w' \in \cB_{\rho}(w)$, we have that $c \in \argmax_{y \in \cY} \smooth{G_{\cM}}(w')_y$.
    Then for every $c' \ne c$ we have $\smooth{G_{\cM}}(w')_c \ge \smooth{G_{\cM}}(w')_{c'}$, so by \Cref{eq:margin-identity} $\bbE_{\widetilde{w} \sim \cM(w')}[\psi_{c,c'}] \ge 0$.
    Since every element of $\cD_{w, \rho}$ is of the form $\cM(w')$ for some $w' \in \cB_{\rho}(w)$, it follows that each $\psi_{c,c'}$ is robustly certified at $\cM(w)$ with respect to $\cD_{w, \rho}$.

    On the other hand, suppose (1) does not hold.
    Then there exists some $w' \in \cB_{\rho}(w)$ and $c' \ne c$ with $\smooth{G_{\cM}}(w')_{c'} > \smooth{G_{\cM}}(w')_{c}$.
    Then \Cref{eq:margin-identity} gives $\bbE_{\widetilde{w} \sim \cM(w')}[\psi_{c,c'}] < 0$, so that $\psi_{c,c'}$ is not robustly certified in the sense of \Cref{dfn:specification} and condition (2) fails.
    
    Finally, if the certified margins are strict, i.e., $\smooth{G_{\cM}}(w')_c > \smooth{G_{\cM}}(w')_{c'}$ for all $c' \ne c$ and all $w' \in \cB_\rho(w)$, then $c$ is the unique maximizer of $y \mapsto \smooth{G_{\cM}}(w')_y$, and we recover robustness in the desired sense of \Cref{dfn:certified-robustness}.    
\end{proof}

\subsection{Proof of \Cref{thm:joint-convexity-bound}}
\label{apx:proof-joint-convexity-bound}
\begin{proof}
    The result follows by applying joint convexity of $D_f$ (\Cref{lem:f-convex}) and rearranging terms according to the partition:
    \begin{align*}
        D_f(\mu \| \nu) &\le \sum_{i=1}^N \theta_i D_f(\mu_i \|\nu_i) = \sum_{i=1}^k \sum_{j \in S_i} \theta_j D_f(\mu_j \| \nu_j) \\
        &\le \sum_{i=1}^k \sum_{j \in S_i} \theta_j \delta_i = \sum_{i=1}^k \delta_i \sum_{j \in S_i} \theta_j = \sum_{i=1}^k \beta_i \delta_i. \tag*{\qedhere}
    \end{align*}
\end{proof}

\begin{cor}[R\'enyi mixture bound]
    \label{cor:renyi-mixture-full}
    With the setting in \Cref{thm:joint-convexity-bound}, the R\'enyi divergence $\RY{\alpha}(\mu \| \nu)$ satisfies the following LogSumExp bound:
    \begin{equation}
        \label{eq:renyi-mixture-1}
        \RY{\alpha}(\mu \| \nu) \le \left(\tfrac{1}{\alpha - 1}\right) \log \left(\sum_{i=1}^N \theta_i \exp((\alpha - 1) \RY{\alpha}(\mu_i \| \nu_i)) \right)
    \end{equation}
    for all $\alpha \in (0, 1) \cup (1, \infty)$.
    Moreover, given $\beta_i$ and $\delta_i$ in the setting of \Cref{thm:joint-convexity-bound}, we have the reorganized bound
    \begin{equation}
        \label{eq:renyi-mixture-2}
        \RY{\alpha}(\mu \| \nu) \le \left(\tfrac{1}{\alpha - 1}\right) \log \left(\sum_{i=1}^k \beta_i \exp((\alpha - 1) \delta_i)\right).
    \end{equation}
\end{cor}

\begin{proof}
    We first note that \Cref{eq:renyi-mixture-2} follows straightforwardly from \Cref{eq:renyi-mixture-1} due to monotonicity of $\log$ and $\exp$.

    We prove \Cref{eq:renyi-mixture-1} for the cases $0 < \alpha < 1$ and $\alpha > 1$ separately.
    \paragraph{Case 1: $0 < \alpha < 1$}
    The R\'enyi divergence is related to the $f$-divergence generated by $f(u) = 1 - u^{\alpha}$ by $\RY{\alpha}(\mu \| \nu) = (\tfrac{1}{\alpha - 1}) \log(1 - D_f(\mu \| \nu))$, where in particular joint convexity of $D_f$ and monotonicity of $\log((1 -\cdot)^{-1})$ give the bound
    \begin{align*}
        \RY{\alpha}(\mu \| \nu) &\le \left(\tfrac{1}{1 - \alpha}\right)\log\left(\left(1 - \sum_{i=1}^N \theta_i D_f(\mu_i \| \nu_i)\right)^{-1}\right) \\
        &= \left(\tfrac{1}{\alpha - 1}\right)\log\left(1 - \sum_{i=1}^N \theta_i D_f(\mu_i \| \nu_i)\right).
    \end{align*}
    Using the fact that $D_{f}(\mu_i \| \nu_i) = 1 - \exp((\alpha - 1)\RY{\alpha}(\mu_i \| \nu_i))$, we get
    \begin{align*}
        \RY{\alpha}(\mu \| \nu) &\le \left(\tfrac{1}{\alpha - 1}\right)\log\left(1 - \sum_{i=1}^N \theta_i (1 - \exp((\alpha - 1)\RY{\alpha}(\mu_i \| \nu_i)))\right) \\
        &= \left(\tfrac{1}{\alpha - 1}\right)\log\left(\sum_{i=1}^N \theta_i \exp((\alpha - 1)\RY{\alpha}(\mu_i \| \nu_i))\right)
    \end{align*}
    where the final step follows since $\sum_i \theta = 1$.
    
    \paragraph{Case 2: $\alpha > 1$}
    The R\'enyi divergence is related to the $f$-divergence generated by $f(u) = u^{\alpha} - 1$ by $\RY{\alpha}(\mu \| \nu) = (\tfrac{1}{\alpha - 1}) \log(1 + D_f(\mu \| \nu))$, where in particular joint convexity of $D_f$ and monotonicity of $\log(1 + \cdot)$ give the bound
    \begin{equation*}
        \RY{\alpha}(\mu \| \nu) \le \left(\tfrac{1}{\alpha - 1}\right) \log\left(1 + \sum_{i=1}^N \theta_i D_f(\mu_i \| \nu_i)\right).
    \end{equation*}
    Using the fact that $D_f(\mu_i \| \nu_i) = \exp((\alpha - 1)\RY{\alpha}(\mu_i \| \nu_i)) - 1$, we get
    \begin{align*}
        \RY{\alpha}(\mu \| \nu) &\le \left(\tfrac{1}{\alpha - 1}\right) \log\left(1 + \sum_{i=1}^N \theta_i (\exp((\alpha - 1)\RY{\alpha}(\mu_i \| \nu_i)) - 1)\right) \\
        &= \left(\tfrac{1}{\alpha - 1}\right) \log\left(\sum_{i=1}^N \theta_i \exp((\alpha - 1)\RY{\alpha}(\mu_i \| \nu_i))\right)
    \end{align*}
    where the final step follows since $\sum_i \theta_i = 1$.
\end{proof}

%% file: sections/A_tightness.tex
\section{Tightness of certificates}
\label{apx:tightness}
Since \citet{cohen2019certified} proved optimality of their certificate for Randomized Smoothing with Gaussian noise, it has been customary to pair derivation of certified guarantees with an analysis of the tightness of the resulting certificate.
In the ideal case, one can prove that a given Randomized Smoothing certificate is \textit{tight} in the following precise sense.
Given an input $w \in \cW$, the certification procedure determines a robustness region $\cR$ at which the input $w$ is robust ($\smooth{g}(w) = \smooth{g}(w')$ for all $w' \in \cB_\rho(w)$ where $\rho \in \cR$).
We can say that the region $\cR$ is \textit{tight} if, for any $\rho \notin \cR$, there exists a ``worst-case'' input $w' \in \cB_{\rho}(w)$ and base algorithm $G$ (subject to some constraints determined by evaluating $G$ at $w$) with the property that $\smooth{g(w')} \ne \smooth{g(w)}$.
In the information-limited setting of \Cref{thm:lim-info-cvx}, $G$ is constrained by the aggregated vote probabilities $0 \le \theta_b \le \theta_a \le 1$.
In the full-information setting of \Cref{thm:full-info-cvx}, more information about the vote distribution can be used.

Turning to our framework, our analysis introduces slack at two key steps.
First, we replace the induced constraint set of $\cD_{w, \rho}$ of \Cref{thm:robustness-equivalence} by a convex outer approximation $\cD_{\cF}$.
Second, we replace the exact divergence bounds $\tau_f^\star = \sup_{w' \in \cB_{\rho}(w)} \Df(\cM(w') \| \cM(w))$ by analytical upper bounds of \Cref{thm:joint-convexity-bound}.
\footnote{To be precise, slackness could also enter through the reorganization process of \Cref{thm:joint-convexity-bound}. However, in the cases we consider, this reorganization is exact.}

The first source of slack corresponds to the convex relaxation itself, and is already reasonably well-understood.
By \Cref{thm:tight-certificate-hockey}, if the exact Hockey-stick profile $\tau_\beta^\star$ (see \Cref{eq:hs-profile}) is available at all orders $\beta \ge 0$, then the resulting certificate is tight in the information-limited setting.
This means no procedure that observes only the vote statistics $\theta_a, \theta_b$ can certify a larger region without further assumptions on the base learner $G$.

Thus, to understand the tightness of our certificates, we should understand the error introduced by \Cref{thm:joint-convexity-bound} and how it affects certification.
For mixture distributions $\mu = \sum_i \theta_i \mu_i$ and $\nu = \sum_i \theta_i \nu_i$ on a shared space, define the joint-convexity slack
\begin{equation}
    \label{eq:jc-slack}
    E_f(\mu, \nu; \theta) \defeq \sum_{i=1}^N \theta_i \Df(\mu_i \| \nu_i) - \Df(\mu \| \nu) \;\ge\;0.
\end{equation}

In this appendix, we seek to understand the second source of slack arising from our use of the joint convexity bound \Cref{thm:joint-convexity-bound}, and how it affects the resulting certificates generated by our framework.
This analysis can be broken into several steps.
\Cref{apx:tightness-index} characterizes $E_f$ exactly as the information gained by releasing the mixture component labels along with the noised observation.
In particular, the joint convexity bound is exact in precisely the cases when the mixture component label can be recovered from the noised observation.
\Cref{apx:tightness-bagging,apx:tightness-bagflip} show that the smoothing schemes of two prior works,~\citet{jia2021intrinsic} and~\citet{zhang2022bagflip}, satisfy this hypothesis, and thus we recover their certificates exactly.

\subsection{Joint-convexity slack from label recovery}
\label{apx:tightness-index}

Throughout this subsection fix mixtures $\mu = \sum_{i=1}^N \theta_i \mu_i$ and $\nu = \sum_{i=1}^N$ on $(\Omega, \cA)$ with common weights $\theta_i > 0$, let $\lambda$ be a common dominating measure for $\{\mu_i, \nu_i\}_{i=1}^N$, and write $p_i \defeq d\mu_i / d\lambda$ and $q_i \defeq d\nu_i / d\lambda$.

Sampling from $\mu$ is a two-stage process: first, draw index $i \sim \theta$, then, draw $\omega \sim \mu_i$.
The mixture $\mu$ reflects the law of $\omega$ marginalized over $i$.
For purposes of understanding the slack $E_f$, it is useful to also define the labeled pair $(i, \omega)$:

\begin{dfn}[Labeled mixture]
    \label{dfn:labeled-mixture}
    The \emph{labeled mixtures} of $\mu$ and $\nu$ are the probability measures on $[N] \times \Omega$
    \begin{equation*}
        \widehat{\mu} \defeq \sum_{i=1}^N \theta_i (\delta_i \otimes \mu_i)
        \qquad
        \widehat{\nu} \defeq \sum_{i=1}^N \theta_i (\delta_i \otimes \nu_i)
    \end{equation*}
    We write $\pi \colon [N] \times \Omega \to \Omega$, $\pi(i, \omega) = \omega$ for the projection onto the observation $\omega$.
\end{dfn}

Note that $\pi$ is a deterministic Markov kernel with $\pi_{\#} \widehat{\mu} = \mu$ and $\pi_{\#} \widehat{\nu} = \nu$.

\begin{lem}[Additivity of labeled mixtures]
    \label{lem:label-exact}
    For every generator $f$ as in \Cref{dfn:f-divergence},
    \begin{equation}
        \Df(\widehat{\mu} \| \widehat{\nu}) = \sum_{i=1}^N \theta_i \Df(\mu_i \| \nu_i).
    \end{equation}
\end{lem}
\begin{proof}
    The measure $\lambda' \defeq \mathrm{count} \otimes \lambda$ dominates $\widehat\mu$ and $\widehat\nu$, with densities $(i, \omega) \mapsto \theta_i p_i(\omega)$ and $(i,\omega) \mapsto \theta_i q_i(\omega)$.
    By the perspective representation (\Cref{lem:f-perspective}) and positive homogeneity of $g_f$ (\Cref{lem:perspective-properties}),
    \begin{equation*}
        \Df(\widehat\mu \| \widehat\nu) = \sum_{i=1}^N \int g_f(\theta_i p_i, \theta_i q_i) \,d\lambda
        = \sum_{i=1}^N \theta_i \int g_f(p_i, q_i)\, d\lambda
        = \sum_{i=1}^N \theta_i \Df(\mu_i \| \nu_i). \qedhere
    \end{equation*}
\end{proof}

Combining \Cref{lem:label-exact} with the data-processing inequality (\Cref{lem:data-processing}) applied to $\pi$ recovers \Cref{thm:joint-convexity-bound}:
\begin{equation*}
    \Df(\mu \| \nu) = \Df(\pi_\# \widehat{\mu} \| \pi_\# \widehat{\nu}) \le \Df(\widehat{\mu} \| \widehat{\nu}) = \sum_{i=1}^N \theta_i \Df(\mu_i \| \nu_i).
\end{equation*}
In particular,
\begin{equation}
    \label{eq:slack-is-dpi-gap}
    E_f = \Df(\widehat{\mu} \| \widehat{\nu}) - \Df(\mu \| \nu).
\end{equation}

\Cref{eq:slack-is-dpi-gap} states that the slack $E_f$ is precisely the difference in $f$-divergence between labeled and unlabeled mixtures.
In our setting the index is the retained index sequence $U \sim \Pi$.

\begin{cor}[Exactness for identifiable mixtures]
    \label{cor:identifiable}
    If the sets $\Omega_i \defeq \{p_i > 0\} \cup \{q_i > 0\}$ are pairwise disjoint up to $\lambda$-null sets, then $\Df(\mu \| \nu) = \sum_{i=1}^N \theta_i \Df(\mu_i \| \nu_i)$ for every generator $f$, that is, $E_f = 0$.
\end{cor}
\begin{proof}
    Let $\iota(\omega) \defeq i$ for $\omega \in \Omega_i$ and arbitrary elsewhere, and define $T(\omega) \defeq (\iota(\omega), \omega)$.
    Each $\mu_i, \nu_i$ is supported only on $\Omega_i$, so $T_\# \mu = \widehat \mu$ and $T_\# \nu = \widehat \nu$.
    Applying \Cref{lem:data-processing} to $T$ and to $\pi$ gives $\Df(\widehat \mu \| \widehat \nu) \le \Df(\mu \| \nu) \le \Df(\widehat \mu \| \widehat \nu)$.
    The right hand side is $\sum_{i=1}^N \theta_i \Df(\mu_i \| \nu_i)$ by \Cref{lem:label-exact}.
\end{proof}

\subsection{Tightness of pure subsampling}
\label{apx:tightness-bagging}

We use these results to recover the bagging certificate of~\citet{jia2021intrinsic} as a special case of our framework.
Take $\phi_{\cX}(\cdot \mid x) = \delta_x$ and $\phi_{\cY}(\cdot \mid y) = \delta_y$ as the feature and label smoothing mechanisms (that is, no noise is added to features or labels).
Call an input $w = (S, x_{\mathrm{te}})$ and corruption $w' \in \cB_{\rho}(w)$ \textit{record-distinct} if the records of $S$ are pairwise distinct, the corrupted records $(x'_i, y'_i)$, $i \in I(w, w')$, are pairwise distinct, and the clean and corrupted records are disjoint.

\begin{thm}[Exactness for pure subsampling]
    \label{thm:bagging-exact}
    Let $\cM$ subsample without noise and let $w' \in \cB_{\rho}(w)$ be record-distinct.
    Write $m(w,w') \defeq \bbP_{U \sim \Pi}\left[\exists j.\; u_j \in I(w,w')\right]$ for the probability that the subsample retains a corrupted record.
    Then, for every generator $f$,
    \begin{equation*}
        \Df(\cM(w') \| \cM(w)) = \sum_{U} \Pi(U) \Df\left(\delta_{S'_{U}} \middle\| \delta_{S_{U}} \right) = \left(f(0) + f'(\infty)\right)\, m(w,w').
    \end{equation*}
    So \Cref{thm:joint-convexity-bound} is exact, and the Hockey-stick profile is
    \begin{equation}
        \label{eq:bagging-profile}
        \Hock[\beta](\cM(w') \| \cM(w)) = \min(1, \beta) m(w, w') \qquad \beta \ge 0.
    \end{equation}
\end{thm}
\begin{proof}
    If $U$ does not contain a corrupted index, then $S'_{U} = S_{U}$.
    Otherwise, $S'_{U}$ contains at least one corrupted record, so $S'_{U} \ne S_V$ for every $V$.
    By record-distinctness, it follows that $S_{U} \ne S_{V}$ and $S'_{U} \ne S'_{V}$ for $U \ne V$.
    Therefore, the component supports $T_U \defeq \{S_{U}, S'_{U}\}$ of the component pairs are pairwise disjoint, and exactness follows from \Cref{cor:identifiable}.
    For \Cref{eq:bagging-profile}, observe that the mixtures agree on their intersection $\{U : U \cap I = \emptyset\}$, and so we may apply \Cref{lem:f-common-component} with overlap mass $1 - m(w, w')$.
\end{proof}

\begin{cor}[Recovery of bagging certificate]
    \label{cor:jia-recovery}
    In the setting of \Cref{thm:bagging-exact}, write $m^\star_\rho \defeq \sup_{w' \in \cB_{\rho}(w)} m(w,w')$.
    Then $\cG(\theta_a, \theta_b)$ is robustly certified at $\cM(w)$ with respect to $\cD_{w, \rho}$ if and only if
    \begin{equation}
        \label{eq:bagging-certificate}
        \theta_a - \theta_b \ge 2 m^\star_\rho.
    \end{equation}
    For $\rho$ constraining only $\emp$, $m^\star_\rho = 1 - (1 - \emp / n)^k$ under size-$k$ sampling with replacement and $m^\star_\rho = 1 - \binom{n - \emp}{k} \binom{n}{k}^{-1}$ without replacement.
\end{cor}
\begin{proof}
    Substituting $\tau_\beta = \min(1, \beta)m^\star_\rho$ into the objective of \Cref{thm:tight-certificate-hockey}, we get $\beta_a^\star = \beta_b^\star = 1$ whenever $\theta_a > m^\star_\rho$ and $\theta_b + m^\star_\rho < 1$, giving optimal value of $\theta_a - \theta_b - 2m^\star_\rho$ in \Cref{eq:lim-info-program}.
\end{proof}

A noteworthy condition for our tightness proof was record-distinctness.
In the usual randomized smoothing setting~\citep{cohen2019certified,jia2021intrinsic,zhang2022bagflip} this is a mild assumption, as the certificates are derived without using additional information about the input.
If record-distinctness does not hold (either due to repetition in $S$ or impossibility of choosing an appropriate $w'$), we cannot directly apply \Cref{lem:f-common-component} to obtain the worst-case $\tau_\beta$ (although it remains a valid upper bound).

\subsection{Recovering BagFlip's certificate}
\label{apx:tightness-bagflip}

In \Cref{apx:tightness-bagging}, we exploited \Cref{cor:identifiable} to obtain a tight certificate based on the TV distance in the case of pure bagging.
Once feature or label noise is introduced, the mixture components are no longer identifiable.
Thus, in principle, the slack \Cref{eq:jc-slack} is nontrivial.

However, it turns out that~\citet{zhang2022bagflip} compute the certificate only for the \textit{label-released} mixture (see \citet[Appendix A]{zhang2022bagflip}), and thus \Cref{cor:identifiable} does apply to the setting actually considered by BagFlip.
This means two things.
First, the certificates obtained within our framework match the certificates of~\citet{zhang2022bagflip} exactly in the analytical sense.
Second, the certificate of~\citet{zhang2022bagflip} may not be tight in the sense that base learners that do not use the mixture component label (which is true of BagFlip's learners) are actually more robust than the certificate guarantees.

\begin{thm}[Recovery of BagFlip certificate]
    \label{thm:bagflip-recovery}
    Consider the BagFlip instantiation
    taking $\cX = [K]^d$, $Y = [K]$, $\Pi$ uniform on $[n]^k$, and $\phi_{\cX},\phi_{\cY}$ the flipping measures of \Cref{dfn:flipping-measure} with rate $\gamma$.
    Let $\widehat{\cM(w)}, \widehat{\cM(w')}$ denote the labeled mixture of $\cM$ as in \Cref{dfn:labeled-mixture}.
    Then the bound of \Cref{thm:joint-convexity-bound}, applied with partition of $[n]^k$ by the number $c$ of retained corrupted examples, evaluates to
    \begin{equation}
        \label{eq:bagflip-profile}
        \tau_\beta \;=\; \Hock[\beta]\left(\widehat{\cM(w')} \,\middle\|\, \widehat{\cM(w)}\right) \;=\; \sum_{c=0}^k \Binom(c;k,\tfrac{\emp}{n}) \delta_\beta\bigl(c(\epsf + \epsl) + \epste\bigr),
    \end{equation}
    where $\delta_\beta(s) \defeq \Hock[\beta](\phi_\gamma(y) \| \phi_{\gamma}(y'))$ for any $y, y'$ with $\norm{y-y'}_0 = s$, given in closed form by \Cref{eq:adj-flip-hs}.
    Consequently, the certificate obtained by applying \Cref{thm:tight-certificate-hockey} to the profile given by \Cref{eq:bagflip-profile} coincides with the radius of~\citet[Thm. 5.1]{zhang2022bagflip}.
\end{thm}
\begin{proof}
    Conditioned on the sampled indices $U$, the pair $(\mu_U, \nu_U)$ consists of two flipping measures on $[K]^{kd} \times [C]^k \times [K]^d$ whose parameters disagree in exactly $c(\epsf + \epsl) + \epste$ coordinates, where $c = \abs{\{j : u_j \in I(w, w')\}}$.
    By \Cref{thm:flip-measure-divergence} the component divergence is determined as a function of $c$, so the regrouping step preserves the exact joint convexity bound.
    Thus, we recover the exact Hockey-stick envelope $\tau_\beta, \;\beta \ge 0$, and via \Cref{thm:tight-certificate-hockey} the optimal information-limited certificate.
\end{proof}

%% file: sections/A_methodology_details.tex
\section{Methodology Details}
\label{apx:methodology-details}

\subsection{Training}
\label{apx:training}

\paragraph{Datasets}
We train models on MNIST, CIFAR10, as well as binary classification tasks derived from them.
MNIST-17 uses the MNIST digits 1 and 7, and CIFAR10-02 uses the CIFAR10 classes 0 and 2 (airplane and bird).
The training and test sets have 60,000 and 10,000 examples on MNIST, 13,007 and 2,163 on MNIST-17, 50,000 and 10,000 on CIFAR10, and 10,000 and 2,000 on CIFAR10-02.
We report certified accuracy on the full test set.
Pixel values are divided by 255, which puts every feature in $[0, 1]$, and we apply no further normalization or standardization.
$d = 784$ on MNIST and $d = 3072$ on CIFAR10.

\paragraph{Architectures}
On MNIST datasets, the base learner is a small convolutional network consisting of two $5 \times 5$ convolutions with 16 and 32 channels, respectively, each followed by ReLU activation and $2 \times 2$ average pooling, then a hidden layer of width 64 with ReLU, and finally a linear output layer.
We use dropout of 0.25 and 0.5 before each linear layer.
On the CIFAR10 datasets, we use the ResNet-20 of \citet{jia2021intrinsic} modified to use group normalization (at most eight groups) and weight-standardized convolutions.
These modifications are due to our parallel training implementation which does not support the running statistics of batch normalization.
Networks in all experiments are trained from random initializations.

\paragraph{Optimization}
All training uses the Adam optimizer with learning rate $10^{-3}$ and runs for 200 epochs with batch size 16.
We use weight decay of $2 \cdot 10^{-3}$ on MNIST and $10^{-4}$ on CIFAR10.
On CIFAR10, the learning rate is multiplied by 0.1, 0.01, 0.001, and 0.0005 from epochs 80, 120, 160, and 180, following BagFlip \citep{zhang2022bagflip}.
Models trained on the full training set (for the RAB baseline) use 20 epochs with batch size 128 on MNIST and MNIST-17, and 10 epochs with batch size 64 on CIFAR10 and CIFAR10-02.

\paragraph{Augmentation}
On MNIST, we use random translations of up to 2 pixels and $8 \times 8$ cutout.
On CIFAR10, we use random translations of up to 4 pixels and horizontal flips.
In addition, each training batch receives Gaussian noise whose standard deviation is drawn uniformly from [0, 1], motivated by the augmentation of \citet{zhang2022bagflip}, and disabled for plain bagging experiments on CIFAR10 and CIFAR10-02.

\subsection{Mechanisms and configurations}
\label{apx:mechanisms}

\paragraph{Mechanisms}
All ensembles use bagging with replacement.
The feature noise scale $\sigma$ is calibrated to a constant coordinate-wise standard deviation: Gaussian noise $\cN(0, \sigma^2)$, Laplace noise $\Lap(0, \sigma / \sqrt{2})$, and uniform noise $\Unif([-\sqrt{3}\sigma, \sqrt{3}\sigma])$.
The feature noise is Gaussian unless stated otherwise.
In the joint feature and label experiments, label noise keeps the original label with probability $1 - \gamma + \gamma / C = 0.9$, with $\gamma = \tfrac{1}{9}$ on MNIST and CIFAR10 and $\gamma = \tfrac{1}{5}$ on MNIST-17 and CIFAR10-02.
The backdoor experiments do not use label noise.

\paragraph{Ensembles}
Every ensemble uses $m = 10{,}000$ base models on MNIST and MNIST-17, and $m = 1{,}000$ base models on CIFAR10 and CIFAR10-02.

\paragraph{$\ell_p$ feature adversaries}
The experiments of \Cref{sec:lp-adversary} use $k = 100$ and $\sigma = 0.5$ on MNIST and MNIST-17, $k = 1000$ and $\sigma = 0.35$ on CIFAR10, and $k = 200$ and $\sigma = 0.35$ on CIFAR10-02.
In \Cref{fig:mnist-cleanlabel-lp}, the Gaussian, Laplace and uniform curves use this configuration, the bagging curve uses $k = 100$ with no noise, and RAB uses $\sigma = 0.5$ on the full training set.
\Cref{fig:budgets} uses the configuration of each dataset.

\paragraph{Label flipping and backdoor adversaries}
\Cref{tab:configs} lists the bag sizes and noise scales of \Cref{tab:label-flip,tab:backdoor}.
In \Cref{tab:label-flip}, bagging uses the same $k$ as our method.
In \Cref{tab:backdoor}, RAB uses the $\sigma_a$ and $\sigma_b$ of our method on the full training set.

\begin{table}[ht]
\centering
\footnotesize
\caption{Bag sizes and noise scales of \Cref{tab:label-flip} and \Cref{tab:backdoor}.}
\begin{tabular}{l ccc ccc}
    \toprule
    & \multicolumn{3}{c}{\textbf{Joint feature \& label}} & \multicolumn{3}{c}{\textbf{Backdoor}} \\
    \cmidrule(lr){2-4} \cmidrule(lr){5-7}
    Dataset & $k$ & $\sigma_a$ & $\sigma_b$ & $k$ & $\sigma_a$ & $\sigma_b$ \\
    \midrule
    MNIST & 30 & 0.25 & 0.5 & 300 & 0.75 & 1.0 \\
    MNIST-17 & 30 & 0.5 & 1.0 & 100 & 0.75 & 1.0 \\
    CIFAR10 & 200 & 0.2 & 0.35 & 1000 & 0.5 & 0.75 \\
    CIFAR10-02 & 25 & 0.35 & 0.5 & 100 & 0.35 & 0.5 \\
    \bottomrule
\end{tabular}
\label{tab:configs}
\end{table}

\subsection{Confidence intervals}
\label{apx:confidence}

Certification in our framework requires high-confidence bounds on teh vote probabilities of the smoothed classifier.
We estimate them from $N$ votes of an ensemble using the following confidence interval.

\begin{lem}[Clopper--Pearson Interval~\citep{clopper1934use}]
    \label{lem:clopper-pearson}
    Let $Z_1, \ldots, Z_N$ be independent, identically distributed Bernoulli random variables with mean $p \in [0,1]$, and let $S_N \defeq \sum_{i=1}^N Z_i$.
    Fix a failure probability $\zeta \in (0,1)$, and let $Q(q; \alpha, \beta)$ denote the $q$-quantile of the Beta distribution with shape parameters $\alpha, \beta$.
    Define
    \begin{align}
        \underline{p} &\defeq Q\!\left(\tfrac{\zeta}{2};\, S,\, N - S + 1\right), \\
        \overline{p} &\defeq Q\!\left(1 - \tfrac{\zeta}{2};\, S + 1,\, N - S\right),
    \end{align}
    with the conventions $\underline{p} \defeq 0$ when $S_N = 0$ and $\overline{p} \defeq 1$ when $S_N = N$.
    Then, with probability at least $1 - \zeta$, $p \in [\underline{p}, \overline{p}]$.
    In particular, the one-sided upper confidence bound $\overline{p}^{\,+} \defeq Q(1 - \zeta;\, S + 1,\, N - S)$ satisfies $\Pr[\, p \le \overline{p}^{\,+} \,] \ge 1 - \zeta$.
\end{lem}

More precisely, we use SimuEM as in \citet{jia2021intrinsic}, which applies a Bonferroni correction over all $C$ classes and the $e$ test points.
Each bound thus has failure probability $\alpha / (eC)$ and the full test-set failure probability is bounded by $\alpha$ via the union bound.

\subsection{Algorithms}
\label{apx:algorithms}

\Cref{alg:training,alg:certify} follow \citet[Algorithm~1]{jia2021intrinsic} and \citet[Algorithms~1 and~2]{weber2023rab}.
We perform certification following \citet[Algorithm~2]{dvijotham2020framework} with hockey-stick constraints of all orders.

\begin{algorithm}[t]
    \caption{\textsc{TrainUnderMechanism}}
    \label{alg:training}
    \begin{algorithmic}[1]
\Require Training set $S = ((x_j, y_j))_{j=1}^n$, learning algorithm $\cA$, mechanism $(\Pi, \phi_{\cX}, \phi_{\cY})$, number of models $N$
\Ensure Base models $g_1, \ldots, g_N$
    \For{$i = 1, \ldots, N$}
        \State Sample $U_i = (u_1, \ldots, u_k) \sim \Pi$
        \State Sample $\widetilde{S}_i \gets \bigl((\widetilde{x}_j, \widetilde{y}_j)\bigr)_{j=1}^{k}$ with $\widetilde{x}_j \sim \phi_{\cX}(\cdot \mid x_{u_j})$ and $\widetilde{y}_j \sim \phi_{\cY}(\cdot \mid y_{u_j})$
        \State $g_i \gets \cA(\widetilde{S}_i)$
    \EndFor
    \State \Return $g_1, \ldots, g_N$
    \end{algorithmic}
\end{algorithm}

\begin{algorithm}[t]
    \caption{\textsc{Certify}}
    \label{alg:certify}
    \begin{algorithmic}[1]
\Require Base models $g_1, \ldots, g_N$, test examples $x_1, \ldots, x_e$, feature noise mechanism $\phi_{\cX}$, per-example budgets $(\epsf, \epsl, \epste)$, confidence level $1 - \alpha$
\Ensure Predicted label and certified number of corrupted examples for each test example
    \For{$x$ in $x_1, \ldots, x_e$}
        \State Sample $\widetilde{x}^{(i)} \sim \phi_{\cX}(\cdot \mid x)$ for $i = 1, \ldots, N$ \Comment{$\widetilde{x}^{(i)} = x$ if the test input is not noised}
        \State $\mathrm{counts}[y] \gets \sum_{i=1}^N \ind[g_i(\widetilde{x}^{(i)}) = y]$ for $y \in [C]$
        \State $c, c' \gets$ top two indices in $\mathrm{counts}$
        \State $\theta_a, \theta_b \gets \textsc{SimuEM}(\mathrm{counts}, \alpha / e)$
        \State $m^\star \gets \textsc{BinarySearch}(\theta_a, \theta_b)$ \Comment{find largest $\emp$ with $\underline{p} > \overline{p}$}
        \State Output $c$ and $m^\star$, or \textsc{Abstain} if no $\emp \ge 0$ is certified
    \EndFor
    \end{algorithmic}
\end{algorithm}

%% file: sections/A_divergence_bounds.tex
\section{Bounding divergences}
\label{apx:calculus}
The core technical challenge of our method is to derive the divergence bounds
\begin{equation}
    \label{eq:worst-case-tau}
    \tau_f(w, \rho) \ge \tau_f^\star(w, \rho) = \sup_{w' \in \cB_{\rho}(w)}  \Df(\cM(w') \| \cM(w))
\end{equation}
for various inputs $w$, smoothing mechanisms $\cM$, threat models $\rho$, and generators $f$.
Via \Cref{thm:joint-convexity-bound}, the problem reduces to bounding divergences between the paired components of the mixture induced by $\cM$.
Fixing mixture component $U$, this decomposes into two steps.

The first step is to argue that the maximum attainable divergence is exhibited by some worst-case perturbation of the subsampled input $x = S_U$, so that we may write
\begin{equation}
    \sup_{x' \in \cB(x)} \Df(\phi_{\#} x' \| \phi_{\#} x) = \Df(\mu^\star \| \nu^\star)
\end{equation}
for some pair of probability measures $\mu^\star, \nu^\star$, where $\phi$ is the joint feature-and-label noising mechanism.
\footnote{This is similar to the concept of a \textit{dominating pair}~\citep{zhu2022optimal} in the privacy literature.}
Note that $x$ and $\cB(x)$ do not necessarily correspond exactly with the joint input $w$ and constraint set $\cB_\rho(w)$; for example, $\cB(x)$ likely depends on $U$.
For reasons we will describe later, the two main desiderata are that the same pair $\mu^\star, \nu^\star$ works for \textit{every} $f$, and that such pairs compose across the factors of $\phi$.
We handle this step in \Cref{apx:worst-case-adjacent}.

The second step is to compute the full set of divergences for $\mu^\star, \nu^\star$.
This is typically straightforward once $\mu^\star, \nu^\star$ are known.
We handle this step in \Cref{apx:computing-adjacent-divergences}, where we work out closed-form divergences for Uniform, Flipping, Gaussian, and Laplace measures by hand.
We also show how to construct product measures using these measures as building blocks, and reduce divergence computation to a one-dimensional numerical computation.

Together, the above steps are sufficient to derive bounds for Uniform, Gaussian, and Flipping measures, as well as product combinations of them, under all ($\ell_p$, $\ell_0$) feature-label constraints.
In some cases the procedure does not work so cleanly.
Informally speaking, if the smoothing geometry does not mesh with the perturbation geometry, no single pair $\mu^\star, \nu^\star$ need exist.
This arises for Laplace feature noise under general $\ell_p$ norms with $p \neq 1$, which we treat separately in \Cref{apx:worst-case-laplace}.

\subsection{Finding worst-case adjacent distributions}
\label{apx:worst-case-adjacent}

\paragraph{Adjacency parameters}
The left side of \Cref{eq:worst-case-tau} optimizes over a high-dimensional set.
In many cases the divergence depends only on some lower-dimensional ``sufficient statistic'' computable from $x, x'$, motivating the following definition.

\begin{dfn}[Adjacency parameter]
    \label{dfn:adjacency-parameter}
    Let $\phi$ be a noising mechanism over space $\cX$.
    An \emph{adjacency parameter} is a map $\theta \colon \cX \times \cX \to \Theta$ such that, for every generator $f$, there exists a function $\delta_f \colon \Theta \to \bbR^+$ such that $\Df(\phi_{\#} x' \| \phi_{\#} x) = \delta_f(\theta(x, x'))$ for all $x, x' \in \cX$.
\end{dfn}

Writing $\Theta_{\cB}(x) \defeq \{\theta(x, x') \;:\; x' \in \cB(x)\}$, \Cref{eq:worst-case-tau} becomes
\begin{equation}
    \label{eq:worst-case-reduced}
    \sup_{x' \in \cB(x)} \Df(\phi_{\#} x' \| \phi_{\#} x) = \sup_{\theta \in \Theta_{\cB}(x)} \delta_f(\theta),
\end{equation}
which is often simpler when $\Theta$ is low-dimensional.
How low-dimensional $\Theta$ can be is governed by symmetry of the constraint set under the noise $\phi$.
By \Cref{lem:data-processing}, $\Df$ is invariant to invertible transformations.
So, if $T$ is a transform with $T_{\#} \phi(x) = \phi(T x)$ for every $x \in \cX$, the divergence can only depend on those qualities of the pair $x, x'$ invariant under $T$.
Taking Gaussian measures as an example, $T$ can be taken to be any Euclidean isometry, leaving $\theta(x, x') = \norm{x - x'}_2$.

\paragraph{Worst-case pairs}
\Cref{eq:worst-case-reduced} may have a different maximizer for each $f$.
Ideally, however, this is not the case.
Define the binary relation $(\mu', \nu') \dle (\mu, \nu)$ if and only if, for every generator $f$, $\Df(\mu' \| \nu') \le \Df(\mu \| \nu)$.
Similarly, call $\theta^\star$ a \textit{worst-case} for $\Theta' \subseteq \Theta$ if $\sup_{\theta \in \Theta'} \delta_f(\theta) = \delta_f(\theta^\star)$ for every $f$.

It can be convenient to characterize worst-case measures through their likelihood ratios.
We have the following lemmata, which will become useful later.

\begin{lem}[Hockey-stick characterization of $\dle$]
    \label{lem:dle-via-hs}
    Let $\mu \ll \nu$ and $\mu' \ll \nu'$ be measures on the same space.
    Then
    \begin{equation}
        \label{eq:dle-via-hs}
        (\mu', \nu') \dle (\mu, \nu)
        \quad\iff\quad
        \Hock[\beta](\mu' \| \nu') \le \Hock[\beta](\mu \| \nu) \;\text{ for every } \beta > 0.
    \end{equation}
\end{lem}
\begin{proof}
    Each $f_\beta$ is itself a generator, which gives the forward implication.
    Conversely, $\rho$ is nonnegative in \Cref{lem:hs-integral-representation}, so a uniform bound on $\Hock[\beta](\mu' \| \nu')$ is preserved into $\Df(\mu' \| \nu')$.
\end{proof}

\begin{lem}[Product composition of worst-case pairs]
    \label{lem:worst-case-product}
    For each $j = 1, \ldots, k$, let $\mu_j, \mu'_j, \nu_j$ be measures on $(\Omega_j, \cA_j)$ with $\mu_j, \mu'_j \ll \nu_j$ and $(\mu'_j, \nu_j) \dle (\mu_j, \nu_j)$.
    Then $\big(\bigotimes_{j} \mu'_j, \bigotimes_{j} \nu_j\big) \dle \big(\bigotimes_{j} \mu_j, \bigotimes_{j} \nu_j\big)$.
    In particular, if $\theta_j^\star$ is a worst case parameter for each $j$, then $(\theta_1^\star, \ldots, \theta_k^\star)$ is worst case for the product measure.
\end{lem}
\begin{proof}
    Replacing $\mu'_j$ by $\mu_j$ one index at a time forms a sequence from $\mu'$ to $\mu$.
    Thus, by transitivity of $\dle$, it suffices to prove the case where the two products differ only in a single factor (WLOG, the first): $\mu'_j = \mu_j$ for all $j \ne 1$.

    Group the factors $\mu_{>1} \defeq \bigotimes_{j \ge 2} \mu_j$ and $\nu_{>1} \defeq \bigotimes_{j \ge 2} \nu_{j}$, and let $\Lambda_{>1}$ denote their likelihood ratio.
    Applying \Cref{lem:hockey-stick-product} with $\mu_{>1}$ as the first factor and using \Cref{eq:dle-via-hs}, we have
    \begin{align*}
        \HockU[\beta](\mu'_1 \otimes \mu_{>1} \| \nu_1 \otimes \nu_{> 1}) &= \bbE_{\mu_{>1}}\left[\HockU[\beta / \Lambda_{>1}](\mu'_1 \| \nu_1)\right] \\
        &\le \bbE_{\mu_{>1}}\left[\HockU[\beta / \Lambda_{>1}](\mu_1 \| \nu_1)\right]
        = \HockU[\beta](\mu_1 \otimes \mu_{>1} \| \nu_1 \otimes \nu_{>1}).
    \end{align*}
    Since $\beta > 0$ was arbitrary, the claim follows from \Cref{lem:dle-via-hs}.

\end{proof}

\subsection{Bounds between adjacent smoothing measures}
\label{apx:computing-adjacent-divergences}
In this appendix we derive Hockey-stick and R\'enyi divergences for the distributions arising in our analysis.

\paragraph{General Strategy}
We show how to compute $\Df(\mu \| \nu)$ when $\mu$ and $\nu$ are adjacent measures.
In these computations, $\nu$ may be interpreted as the clean distribution, and $\mu$ the corrupted one.
Each computation proceeds by writing down the distribution of $\Lambda = p/q$ under $\nu$ and evaluating \Cref{eq:f-divergence-lambda}.

\begin{lem}[Hockey-stick for product measures]
    \label{lem:hockey-stick-product}
    Let $\mu = \mu_1 \otimes \mu_2$ and $\nu = \nu_1 \otimes \nu_2$ and write $\Lambda_1$ and $\Lambda_2$ for the likelihood ratios under $(\mu_1, \nu_1)$ and $(\mu_2, \nu_2)$, respectively.
    Then, for every $\beta > 0$,
    \begin{equation}
        \HockU[\beta](\mu \| \nu) = \bbE_{\nu_1}\left[\Lambda_1 \HockU[\beta/\Lambda_1](\mu_2 \| \nu_2)\right] = \bbE_{\mu_1}\left[\HockU[\beta/\Lambda_1](\mu_2 \| \nu_2)\right].
    \end{equation}
\end{lem}
\begin{proof}
    Using $(\lambda u - \beta)_+ = \lambda(u - \beta / \lambda)_+$ for $\lambda > 0$, we have
    \begin{equation*}
        \HockU[\beta](\mu \| \nu) = \bbE_{\nu}\left[(\Lambda - \beta)_+\right] = \bbE_{\nu_1} \bbE_{\nu_2} \left[(\Lambda_1 \Lambda_2 - \beta)_+ \right] = \bbE_{\nu_1} \left[\Lambda_1 \HockU[\beta / \Lambda_1](\mu_2 \| \nu_2)\right].
    \end{equation*}
    The final equality holds since $\bbE_{\nu_1}[\Lambda_1 h(\Lambda_1)] = \bbE_{\mu_1}[h(\Lambda_1)]$.%
\end{proof}

\begin{cor}
    \label{cor:discrete-continuous-hockey}
    Suppose $\mu = \bigotimes_j \mu_j$, $\nu = \bigotimes_j \nu_j$, and that $\Lambda_j$ takes finitely many values for $j < k$.
    Let $\cV$ be set of values of $\Lambda_{<k} \defeq \prod_{j<k} \Lambda_j$, and let $P_\lambda \defeq \bbP_{\mu_1 \otimes \cdots \otimes \mu_{k-1}}[\Lambda_{<k} = \lambda]$.
    Then for every $\beta > 0$,
    \begin{equation}
        \label{eq:discrete-hockey}
        \HockU[\beta](\mu \| \nu) = \sum_{\lambda \in \cV} P_\lambda \HockU[\beta / \lambda](\mu_k \| \nu_k).
    \end{equation}
    Moreover, if $\mu_k = \nu_k$, then $\HockU[\beta](\mu \| \nu) = \sum_{\lambda > \beta}(1 - \beta/\lambda)P_\lambda$.
\end{cor}
\begin{proof}
    Apply \Cref{lem:hockey-stick-product} with the first factor $\mu_1 \otimes \cdots \otimes \mu_{k-1}$ and take the sum over the values of $\Lambda_{<k}$.
    The final claim holds by taking $\mu_k = \nu_k$ to be a point mass, so $\HockU[\beta / \lambda] = (1 - \beta / \lambda)_+$.
\end{proof}

\Cref{cor:discrete-continuous-hockey} allows us to fully understand the hockey stick profile of any measure constructed from discrete factors combined with a factor whose entire profile is known.
The relevant application below is to construct joint feature-and-label noises which combine discrete label flipping with continuous feature noise (e.g., \Cref{dfn:gaussian-flipping-measure})).

The convolution in \Cref{eq:discrete-hockey} in general can be computationally expensive if many factors are tensored together.
The number of discrete likelihood ratio values is upper-bounded by the product $\prod_{j} \abs{\cV_{j}}$.
Thankfully, for the noises we consider, the likelihood ratio values collapse onto a lattice, greatly simplifying the computation.

\paragraph{Smoothing measures}
We now catalog the smoothing measures used in our work and their associated divergences.

\begin{dfn}[Uniform measure]
    \label{dfn:uniform-measure}
    Let $m \in \bbN$ and $\gamma \in \bbR^+$.
    Given $x \in \bbR^m$, the associated \emph{Uniform} measure is the product measure
    \begin{equation}
        \label{eq:uniform-measure}
        \mu \defeq \bigotimes_{i=1}^m \Unif(x_i - \gamma, x_i + \gamma).
    \end{equation}
    We say that two Uniform measures on $\bbR^m$ with the same $\gamma$ and data $x, x'$ are $\theta$-adjacent if
    \begin{equation}
        \theta = 1 - \prod_{i=1}^m \left(1 - \frac{\abs{x_i - x'_i}}{2\gamma}\right)_+.
    \end{equation}
\end{dfn}

\begin{thm}[Divergences for adjacent uniform measures]
    \label{thm:uniform-measure-divergence}
    Let $\mu, \nu$ be $\theta$-adjacent Uniform measures on $\bbR^m$.
    Then
    \begin{align}
        \Ren[\alpha](\mu \| \nu) &= \tfrac{1}{\alpha - 1} \log (1 - \theta), \;\text{for}\; \alpha \in (0, 1), \;\text{else}\; \infty, \\
        \KL(\mu \| \nu) &= \infty, \\
        \Hock[\beta](\mu \| \nu) &= \min(1, \beta) \theta, \\
        \TV(\mu \| \nu) &= \theta.
    \end{align}
\end{thm}
\begin{proof}
    $\mu$ and $\nu$ have densities $p, q$ which evaluate to $(2\gamma)^{-m}$ on their supports and 0 everywhere else.
    Thus the likelihood ratio $\Lambda$ satisfies the hypotheses of \Cref{lem:f-common-component}, so we may write $\Df(\mu \| \nu) = (f(0) + f'(\infty))\TV(\mu \| \nu)$ for any generator $f$.
    The overlap region defining $\Lambda = 1$ can be seen to be of length $(2\gamma - \abs{x_i - x'_i})_+$ on each coordinate, which has mass under $\nu$ of $(1 - \tfrac{\abs{x_i - x'_i}}{2\gamma})_+$. Thus $\Lambda = 1$ with probability $1 - \TV(\mu \| \nu) = 1 - \theta$, completing the proof.
\end{proof}

\begin{dfn}[Gaussian measure]
    \label{dfn:gaussian-measure}
    Let $m \in \bbN$ and $\sigma \in \bbR^+$.
    Given $x \in \bbR^m$, the associated \emph{Gaussian} measure is the standard isotropic Gaussian measure:
    \begin{equation}
        \mu \defeq \cN(x, \sigma^2 I_m).
    \end{equation}
    We say that two Gaussian measures on the same space $\bbR^{m}$ and same parameter $\sigma$ with data $x$ and $x'$, respectively, are $r$-adjacent if $\frac{\norm{x - x'}_2}{\sigma} = r$.
\end{dfn}

\begin{thm}[Divergences for adjacent Gaussian measures]
    \label{thm:gaussian-measure-divergence}
    Let $\mu = \cN(x, \sigma^2 I_m)$ and $\nu = \cN(x', \sigma^2 I_m)$ be $r$-adjacent Gaussian measures on $\bbR^m$.
    Then
    \begin{align}
        \label{eq:adj-gauss-ry}
        \Ren[\alpha](\mu \| \nu) &= \frac{\alpha r^2}{2} \\
        \label{eq:adj-gauss-kl}
        \KL(\mu \| \nu) &= \frac{r^2}{2} \\
        \label{eq:adj-gauss-hs}
        \HS{\beta}(\mu \| \nu) &= \Phi \left(\tfrac{r}{2} - \tfrac{\log \beta}{r}\right) - \beta \Phi\left(-\tfrac{r}{2} - \tfrac{\log \beta}{r}\right) - (1 - \beta)_+ \\
        \label{eq:adj-gauss-tv}
        \TV(\mu \| \nu) &= 2\Phi\left(\tfrac{r}{2}\right) - 1
    \end{align}
\end{thm}
\begin{proof}
    \Cref{eq:adj-gauss-ry,eq:adj-gauss-kl,eq:adj-gauss-hs,eq:adj-gauss-tv} are well-known; see \citet{mironov2017renyi} for \Cref{eq:adj-gauss-ry} and \citet{balle2018improving} for \Cref{eq:adj-gauss-hs}.
    \Cref{eq:adj-gauss-kl} follows by taking the limit $\alpha \to 1$ in \Cref{eq:adj-gauss-ry}, and \Cref{eq:adj-gauss-tv} follows by taking $\beta = 1$ in \Cref{eq:adj-gauss-hs} and applying the identity $\Phi(-x) = 1 - \Phi(x)$.
\end{proof}

\begin{dfn}[Laplace measure]
    \label{dfn:laplace-measure}
    Let $m \in \bbN$ and $b \in \bbR^+$.
    Given $x \in \bbR^m$, the associated \emph{Laplace} measure is the product measure
    \begin{equation}
        \mu \defeq \Lap(x, b)^{\otimes m},  \qquad \frac{d\mu}{dz}(z) = (2b)^{-m} \exp(- \norm{z - x}_1 / b).
    \end{equation}
    We say that two Laplace measures on the same space $\bbR^m$ and same parameter $b$ with data $x$ and $x'$, respectively, are $\mathbf{t}$-adjacent if
    \begin{equation}
        \label{eq:laplace-adjacency}
        \mathbf{t} = (\abs{x_1 - x'_1}, \ldots, \abs{x_m - x'_m}) / b \in \bbR_{\ge 0}^m,
    \end{equation}
    taken in nonincreasing order.
\end{dfn}

\begin{thm}[Divergences for adjacent Laplace measures]
    \label{thm:laplace-measure-divergence}
    Let $\mu, \nu$ be $\mathbf{t}$-adjacent Laplace measures on $\bbR^m$.
    For notational convenience, for $\alpha \in (0, 1) \cup (1, \infty)$ and $t \ge 0$ define the quantities
    \begin{equation}
        \zeta_\alpha(t) \defeq \tfrac{\alpha}{2\alpha - 1}e^{(\alpha - 1)t} + \tfrac{\alpha - 1}{2\alpha - 1}e^{-\alpha t}, \qquad \kappa_{\alpha}(t) \defeq \tfrac{1}{\alpha - 1} \log \zeta_a(t),
    \end{equation}
    where we define by the limit $\zeta_{1/2}(t) = e^{-t/2}(1 + t/2)$.
    Under $\nu$, the likelihood ratio factorizes as $\Lambda = \prod_{j=1}^m \Lambda_j = e^{-\norm{\mathbf{t}}_1} \prod_{j=1}^m V_{\mathbf{t}_j}$ where each $V_{\mathbf{t}_j}$ is independent and
    \begin{equation}
        \label{eq:laplace-coordinate-law}
        \bbP[V_t = 1] = \frac{1}{2}, \qquad \tfrac{d\bbP}{dv}\bigg|_{(1, e^{2t})} = \frac{1}{4}v^{-3/2}, \qquad \bbP[V_t = e^{2t}] = \frac{1}{2}e^{-t}.
    \end{equation}
    Then, we have that
    \begin{align}
        \label{eq:adj-laplace-ry}
        \Ren[\alpha](\mu \| \nu) &= \textstyle\sum_{j} \kappa_\alpha(\mathbf{t}_j) \\
        \label{eq:adj-laplace-kl}
        \KL(\mu \| \nu) &= \textstyle \sum_{j}(\mathbf{t}_j + e^{-\mathbf{t}_j} - 1).
    \end{align}
    If $\mathbf{t} = t e_1$ has a single nonzero coordinate, then for all $\beta > 0$,
    \begin{align}
        \label{eq:adj-laplace-hs}
        \Hock[\beta](\mu \| \nu) &= \left[(1 - \sqrt{\beta} e^{-t / 2})_+ - (1 - \beta)_+\right]_+, \\
        \label{eq:adj-laplace-tv}
        \TV(\mu \| \nu) &= 1 - e^{-t / 2}.
    \end{align}
    For general $\mathbf{t}$, the Hockey-stick divergence can be obtained numerically by convolving $m$ copies of the distribution in \Cref{eq:laplace-coordinate-law}.
\end{thm}
\begin{proof}
    We first exhibit the likelihood ratio structure from \Cref{eq:laplace-coordinate-law}.
    That the likelihood ratio factorizes into $\Lambda_j$ is clear from the product measure structure of $\mu, \nu$.
    WLOG we can assume $b = 1$, $x = 0$, and $x'_j = \mathbf{t}_j \ge 0$.
    Write $p_j, q_j$ for the densities of $\mu_j, \nu_j$, respectively.
    Each coordinate $j$ has
    \begin{equation*}
        \Lambda_j(x)
        \;=\;
        \frac{p_j(x)}{q_j(x)} = \exp(-\abs{x - \mathbf{t}_j} + \abs{x}) \\
        \;=\;
        \begin{cases}
            \exp(-\mathbf{t}_j), &x \le 0, \\
            \exp(2x - \mathbf{t}_j), &0 \le x \le \mathbf{t}_j, \\
            \exp(\mathbf{t}_j), &x \ge \mathbf{t}_j.
        \end{cases}
    \end{equation*}
    We would like to understand the distribution of the variable $V_{\mathbf{t}_j} \defeq \exp(\mathbf{t}_j) \Lambda_j$ under $\nu$.
    On the endpoints we see that
    \begin{align*}
        \bbP_\nu[V_{\mathbf{t}_j} = 1] = \int_{\infty}^0 1/2 \exp(-\abs{x}) &= 1/2 \\
        \bbP_\nu[V_{\mathbf{t}_j} = e^{2t_j}] = \int_{t}^\infty 1/2 \exp(-\abs{x}) &= e^{-t}/2,
    \end{align*}
    recovering the first and last identities in \Cref{eq:laplace-coordinate-law}.
    For the points $x \in (1, e^{2t_j})$, we note that the map $f(x) = e^{2x}$ is differentiable with differentiable inverse $f^{-1}(v) = \tfrac{1}{2}\ln(v)$, so we may apply the change of coordinates formula to arrive at the density
    \begin{equation*}
        \tfrac{d\bbP}{dv}\bigg|_{(1, e^{2\mathbf{t}_j})}
        \;=\;
        \frac{q(f^{-1}(v))}{\abs{f'(f^{-1}(v))}}
        \;=\;
        \frac{\tfrac{1}{2} v^{-1/2}}{2 v}
        \;=\;
        \tfrac{1}{4} v^{-3/2}.
    \end{equation*}
    Noting that $\prod_j e^{\mathbf{t}_j} = e^{\norm{\mathrm{t}}_1}$, we recover $\Lambda = e^{-\norm{\mathbf{t}}_1} \prod_{j=1}^m V_{\mathbf{t}_j}$.
    \paragraph{R\'enyi and KL divergences}
    Evaluating the Hellinger integral $\Hint[\alpha]$ for a single coordinate via \Cref{eq:laplace-coordinate-law} gives
    \begin{equation*}
        \Hint[\alpha](\mu_j \| \nu_j)
        \;=\;
        \bbE_\nu[\lambda_j^\alpha]
        \;=\;
        e^{-\alpha \mathbf{t}_j} \bbE_\nu[V_{\mathbf{t}_j}^\alpha].
    \end{equation*}
    Computing the latter expectation gives
    \begin{align*}
        \bbE_{\nu}[V_{\mathbf{t}_j}^\alpha]
        &=
        \tfrac{1}{2} + \int_{1}^{e^{2t}}\tfrac{1}{4}v^{\alpha-3/2}\,dv + \tfrac{1}{2}e^{-\mathbf{t}_j + 2\alpha \mathbf{t}_j} \\
        &= \tfrac{1}{2} + \tfrac{1}{4\alpha-2}(e^{(2\alpha-1)t}-1) + \tfrac{1}{2}e^{-\mathbf{t}_j + 2\alpha \mathbf{t}_j} \\
        &= (\tfrac{1}{2} - \tfrac{1}{2(2\alpha - 1)}) + (\tfrac{1}{2} + \tfrac{1}{2(2\alpha - 1)})e^{(2\alpha - 1)\mathbf{t}_j} \\
        &= (\tfrac{\alpha - 1}{2\alpha - 1}) + (\tfrac{\alpha}{2\alpha - 1})e^{(2\alpha - 1)\mathbf{t}_j}.
    \end{align*}
    Multiplying through by $e^{- \alpha \mathbf{t}_j}$ recovers the expression $\Hint[\alpha](\mu_j \| \nu_j) = \zeta_\alpha(t)$.
    The case for $\alpha = 1/2$ follows by recomputing the integral directly and noting $\int v^{-1} = \ln(v)$.
    Finally, taking $\Ren[\alpha](\mu_j \| \nu_j) = \tfrac{1}{\alpha - 1} \log \Hint[\alpha](\mu_j \| \nu_j)$ and summing over coordinates (applying \Cref{lem:tensorization}) yields \Cref{eq:adj-laplace-kl}.
    It can be verified that taking the limit $\alpha \to 1$ yields \Cref{eq:adj-laplace-kl}.

    \paragraph{Hockey-stick divergence}
    Suppose $\mathbf{t} = t e_1$.
    Cancelling identical factors (\Cref{lem:f-cancel}), we need only compute $\Hock[\beta](\mu_1 \| \nu_1)$.
    By \Cref{eq:laplace-coordinate-law} for $\Lambda_1$, we see that $x \mapsto \Lambda_1(x)$ is nondecreasing in $x$ with range $[e^{-t}, e^{t}]$.
    Hence, for $\beta \in (e^{-t}, e^{t})$, we have the set
    \begin{equation*}
        E_\beta \defeq \{\Lambda_1 \ge \beta\} = (x_\beta, \infty), \qquad x_\beta \defeq \tfrac{1}{2}(t + \log \beta) \in (0, t).
    \end{equation*}
    Integrating over these regions gives
    \begin{align*}
        \nu(E_\beta) &= \textstyle\int_{x_\beta}^\infty \tfrac{1}{2} \exp(-\abs{x})\,dx = \tfrac{1}{2} e^{-x_\beta}, \\
        \mu(E_\beta) &= \textstyle \int_{x_\beta}^\infty \tfrac{1}{2} \exp(- \abs{x - t})\,dx = 1 - \tfrac{1}{2} e^{-(t - x_\beta)}.
    \end{align*}
    Using $e^{x_\beta} = \sqrt{\beta}e^{t/2}$, we compute the (unnormalized) hockey-stick directly as
    \begin{equation*}
        \HockU[\beta](\mu \| \nu)
        \;=\;
        \mu(E_{\beta}) - \beta\nu(E_\beta) = 1 - \tfrac{1}{2}\sqrt{\beta}e^{-t/2} - \tfrac{1}{2}\sqrt{\beta}e^{-t/2}
        \;=\;
        1 - \sqrt{\beta} e^{-t/2}.
    \end{equation*}
    When $\beta \ge e^{t}$, $\Lambda_1 \le \beta$ everywhere and $\HockU[\beta](\mu\|\nu) = 0$.
    When $\beta \le e^{-t}$, we have $\Lambda_1 \ge \beta$ everywhere and $\HockU[\beta](\mu \| \nu) = 1 - \beta$.
    Combining all three cases (and re-normalizing) yields \Cref{eq:adj-laplace-hs}.
    Taking $\beta = 1$ yields \Cref{eq:adj-laplace-tv}.
\end{proof}

\begin{dfn}[Flipping Measure]
    \label{dfn:flipping-measure}
    Let $C \ge 2$ be an integer alphabet size, $y = (y_1, \ldots, y_m) \in [C]^m$, and $\gamma \in [0, 1]$.
    We define the \emph{flipping measure}
    \begin{equation}
        \phi_{\gamma}(y) \defeq \bigotimes_{i=1}^m \left( (1 - \gamma) e_{y_i} + \gamma \unif([C]) \right).
    \end{equation}
    We say that two flipping measures on the same space $[C]^m$ are $s$-adjacent if $\norm{y - y'}_0 = s$
\end{dfn}

\begin{thm}[Divergences for Adjacent Flipping Measures]
    \label{thm:flip-measure-divergence}
    Let
    \begin{equation*}
        \mu = \phi_{\gamma}(y) \qquad\text{and}\qquad \nu = \phi_{\gamma}(y')
    \end{equation*}
    be $s$-adjacent flipping measures on $[C]^m$.
    For notational convenience define the quantities
    \begin{equation}
        p_{+} \defeq 1 - \gamma + \tfrac{\gamma}{C}, \qquad
        p_{-} \defeq \tfrac{\gamma}{C}, \qquad
        p_{0} \defeq \tfrac{(C-2)}{C} \gamma
    \end{equation}
    and
    \begin{equation}
        d_\beta \defeq \floor*{\tfrac{\log \beta}{\log \tfrac{p_{+}}{p_{-}}}} + 1 \qquad \beta \ge 1
    \end{equation}
    Then
    \begin{align}
        \label{eq:adj-flip-ry}
        \RY{\alpha}(\mu \| \nu) &= \left(\tfrac{s}{\alpha - 1}\right) \log(p_+^\alpha p_-^{1 - \alpha} + p_+^{1 - \alpha} p_-^\alpha + p_0) \\
        \label{eq:adj-flip-kl}
        \KL(\mu \| \nu) &= s(1 - \gamma)\log\tfrac{p_{+}}{p_{-}} \\
        \label{eq:adj-flip-hs}
        \HS{\beta}(\mu \| \nu) &= \sum_{d = d_\beta}^{s} \left(1 - \beta (\tfrac{p_{+}}{p_{-}})^{-d} \right) P_d \\
        \label{eq:adj-flip-tv}
        \TV(\mu \| \nu) &= \sum_{d=1}^s \left(1 - (\tfrac{p_{+}}{p_{-}})^{-d}\right) P_d
    \end{align}
    where $P_d$ is the distribution defined by
    \begin{align}
        P_d &\defeq \bbP_{\widetilde{y} \sim \mu}[\norm{\widetilde{y} - y'}_0 - \norm{\widetilde{y} - y}_0 = d] \\
        \label{eq:label-flip-law}
        &= \sum_{\substack{n_{+} + n_{-} + n_0 = s \\ n_{+} - n_{-} = d}} \binom{s}{n_{+}, n_{-}, n_0} p_{+}^{n_{+}} p_{-}^{n_{-}} p_0^{n_0} \\
        &= \sum_{l = \max(0, -d)}^{\floor{(s - d) / 2}} \binom{s}{l + d, l, s - 2l - d} p_{+}^{l + d} p_{-}^{l} p_0^{s - 2l - d}
    \end{align}
\end{thm}
\begin{proof}
    Without loss of generality, we may relabel coordinates so that $y_i \ne y'_i$ for $i \in [s]$, and $y_i = y'_i$ otherwise.
    Write $\mu = \bigotimes_{i=1}^m \mu_i = \bigotimes_{i=1}^m \phi_{\gamma}(y_i)$ and $\nu = \bigotimes_{i=1}^m \nu_i = \bigotimes_{i=1}^m \phi_{\gamma}(y'_i)$
    By definition, $\mu_i$ assigns mass $p_{+} = 1 - \gamma + \gamma / C$ to $y_i$ and mass $p_{-} = \gamma / C$ to every other symbol.
    $\nu_i$ is defined similarly, swapping $y_i$ for $y'_i$.

    \paragraph{Tensorizing Divergences}
    First, we point out that the R\'enyi divergence \Cref{eq:adj-flip-ry} follows from Tensorization (\Cref{lem:tensorization}) since $\mu$ and $\nu$ decompose as product measures.
    In particular, we need only understand the component-wise divergences $\RY{\alpha}(\mu_i \| \nu_i)$ for $1 \le i \le s$, and we may recover $\RY{\alpha}(\mu \| \nu) = \sum_{i=1}^s \RY{\alpha}(\mu_i \| \nu_i)$.
    For $1 \le i \le s$, $y_i \ne y'_i$, and so writing down the likelihood ratio for a single component takes one of three values:
    \begin{equation}
        \label{eq:flip-component-ratio}
        \frac{\mu_i(\widetilde{y})}{\nu_i(\widetilde{y})} =
        \begin{cases}
            p_{+}\,/\,p_{-} &\widetilde{y}_i = y_i \\
            p_{-} \,/\, p_{+} &\widetilde{y}_i = y'_i \\
            1 &\text{otherwise}. \\
        \end{cases}
    \end{equation}
    We can then compute the expectation over $\widetilde{y}_i \sim \mu_i$:
    \begin{align*}
        \bbE_{\widetilde{y}_i \sim \mu_i}\left[\left(\frac{\mu_i(\widetilde{y})}{\nu_i(\widetilde{y})}\right)^{\alpha - 1}\right] &= p_{+} \left(\tfrac{p_{+}}{p_{-}}\right)^{\alpha - 1} + p_{-} \left(\tfrac{p_{-}}{p_{+}}\right)^{\alpha - 1} + p_0 \\
        &= p_{+}^\alpha p_{-}^{1 - \alpha} + p_{-}^\alpha p_{+}^{1 - \alpha} + p_0
    \end{align*}
    from which it follows
    \begin{align*}
        \RY{\alpha}(\mu_i \| \nu_i) &= \left(\tfrac{1}{\alpha - 1}\right) \log\left(\bbE_{\widetilde{y}_i \sim \nu_i}\left[\left(\tfrac{\mu_i(\widetilde{y})}{\nu_i(\widetilde{y})}\right)^{\alpha - 1}\right]\right) \\
        &= \left(\tfrac{1}{\alpha - 1}\right) \log(p_+^\alpha p_-^{1 - \alpha} + p_+^{1 - \alpha} p_-^\alpha + p_0)
    \end{align*}
    allowing us to conclude
    \begin{align*}
        \RY{\alpha}(\mu \| \nu) &= \sum_{i=1}^s \RY{\alpha}(\mu_i \| \nu_i) \\
        &= \left(\tfrac{s}{\alpha - 1}\right) \log(p_+^\alpha p_-^{1 - \alpha} + p_+^{1 - \alpha} p_-^\alpha + p_0).
    \end{align*}
    
    Computing $\KL(\mu \| \nu)$ proceeds similarly by using \Cref{eq:flip-component-ratio} and evaluating
    \begin{align*}
        \KL(\mu_i \| \nu_i) &= \bbE_{\widetilde{y} \sim \mu} \left[\log \tfrac{\mu_i(\widetilde{y})}{\nu_i(\widetilde{y})}\right] \\
        &= p_{+} \log \tfrac{p_+}{p_-} + p_{-} \log \tfrac{p_-}{p_+} \\
        &= (p_{+} - p_{-}) \log \tfrac{p_{+}}{p{-}} \\
        &= (1 - \gamma) \log \tfrac{p_{+}}{p_{-}}.
    \end{align*}
    Using tensorization of KL gives the identify from \Cref{eq:adj-flip-kl}.

    \paragraph{Hockey Stick Divergence}
    To compute the Hockey Stick Divergence we must understand the entire likelihood ratio distribution $\mu / \nu$ holistically.
    For $\widetilde{y} \sim \nu$, define the random distances
    \begin{align}
        \delta &\defeq \norm{\widetilde{y} - y}_0 \qquad&& \delta' \defeq \norm{\widetilde{y} - y'}_0, \\
        \delta_s &\defeq \norm{\widetilde{y}^{(s)} - y^{(s)}}_0 \qquad&& \delta' \defeq \norm{\widetilde{y}^{(s)} - y'^{(s)}}_0
    \end{align}
    where $\widetilde{y}^{(s)}$ and $y^{(s)}$ are the respective projections onto the first $s$ coordinates.
    Define the signed random difference $d \defeq \delta' - \delta$.
    Notably, since $y$ and $y'$ only disagree in the first $s$ coordinates, the difference $d$ is determined entirely by the projection of $y$, $y'$, and  $\widetilde{y}$ onto the first $s$ coordinates.
    That is, $d = \delta' - \delta = \delta'^{(s)} - \delta^{(s)}$.
    
    For any given $\widetilde{y} \in [C]^m$, we can write
    \begin{align*}
        \mu(\widetilde{y}) &= \prod_{i=1}^n p_{+}^{\ind[\widetilde{y}_i = y_i]} p_{-}^{\ind[\widetilde{y} \ne y_i]} = p_{+}^{n - \delta}p_{-}^{\delta} \\
        \nu(\widetilde{y}) &= \prod_{i=1}^n p_{+}^{\ind[\widetilde{y}_i = y'_i]} p_{-}^{\ind[\widetilde{y} \ne y'_i]} = p_{+}^{n - \delta'}p_{-}^{\delta'}
    \end{align*}
    Thus the likelihood ratio $\mu(\widetilde{y}) / \nu(\widetilde{y})$ is given by
    \begin{equation}
        \label{eq:flip-full-ratio}
        \frac{\mu(\widetilde{y})}{\nu(\widetilde{y})} = p_{+}^{n - \delta} p_{-}^{\delta} p_{+}^{\delta' - n} p_{-}^{-\delta'} \,=\, p_{+}^{\delta' - \delta} p_{-}^{\delta - \delta'} \,=\, \left(\tfrac{p_{+}}{p_{-}}\right)^d.
    \end{equation}
    $d$ takes values in $\{-s, \ldots, 0, \ldots, s\}$, and so there are exactly $2s + 1 = \cO(s)$ distinct likelihood ratio value regions.
    Now, from \Cref{eq:flip-full-ratio}, we would like to understand the law for $d$ under $\nu$.
    We start by expressing $d$ as the result of a biased random walk:
    \begin{align}
        d &\overset{d}{=} \sum_{i} d_i, \\
        d_i &\overset{\mathrm{i.i.d.}}{\sim}
        \begin{cases}
            1 & \text{with prob. } p_{+} \\
            0 & \text{with prob. } p_0 \\
            -1 & \text{with prob. } p_{-}.
        \end{cases}
    \end{align}
    We can understand the law for $d$ using the theory of probability generating functions (PGFs).
    If $X$ is a discrete integer-values random variable taking values in $\{-B, \ldots, B\}$ with mass function $p(x)$, the PGF for $X$ is defined as
    \begin{equation*}
        G_X(z) \defeq \bbE_{x \sim X}[z^x] = \sum_{x=-b}^b p(x) z^x.
    \end{equation*}
    The PGF for $d_i$ is thus $G_i(z) = p_{-} z^{-1} + p_0 + p_{+} z$.
    The PGF for a sum of independent random variables is simply a product the PGFs for the variables, and thus the PGF for $d$ is given by
    \begin{equation}
        G(z) = \prod_{i=1}^s G_i(z) = (p_{-} z^{-1} + p_0 + p_{+} z)^s.
    \end{equation}
    We recover the probability $P_d = \bbP_{\widetilde{y} \sim \nu}[d = \widetilde{d}]$ by extracting the coefficient of the $z^{\widetilde{d}}$ term, in this case computed as
    \begin{equation}
        [z^{\widetilde{d}}]G(z) = \sum_{\substack{n_{+} + n_{-} + n_0 = s \\ n_{+} - n_{-} = d}} \binom{s}{n_{+}, n_{-}, n_0} p_{+}^{n_{+}} p_{-}^{n_{-}} p_0^{n_0}
    \end{equation}
    where the inner probability corresponds to choosing $n_+$ choices of $p_{+}$, $n_{-}$ choices of $p_{-}$, and $n_0$ choices of $p_0$ in the random walk.

    Knowing the law for $d$ under $\nu$, we recover $\cO(s)$ regions of constant likelihood ratio, one for each $d$, and each of which has likelihood ratio $(p_{+}/p_{-})^d$.
    To compute the Hockey Stick divergence, we must also understand the law for $d$ under $\mu$, defined by
    \begin{equation}
        Q_d \defeq \bbP_{\widetilde{y} \sim \mu}[d = \widetilde{d}].
    \end{equation}
    Luckily, moving from $\nu$ to $\mu$ under a given $\widetilde{y}$ is straightforward to understand.
    Using the same random walk formulation gives inverted step probabilities, and pushing through the PGF $H$ gives $H(z) = G(1/z)$.
    Extracting coefficients, we see that
    \begin{equation}
        \label{eq:law-for-q}
        Q_d = P_{-d}.
    \end{equation}
    Moreover, $P_d$ itself admits a certain multiplicative symmetry simplifying the expression further.
    For a fixed $n_{+}, n_{-}$ satisfying $n_{+} - n_{-} = -d$, we have
    \begin{align}
        p_{+}^{n_{+}} p_{-}^{n_{-}} p_0^{n_0} &= \left(\tfrac{p_+}{p_-}p_{-}\right)^{n_{+}} \left(\tfrac{p_-}{p_+}p_{+}\right)^{n_{-}} p_0^{n_0} \\
        &= \left(\tfrac{p_{+}}{p_{-}}\right)^{n_{+} - n_{-}} p_{-}^{n_{+}} p_{+}^{n_{-}} p_0^{n_0} \\
        \label{eq:mult-symmetry}
        &= \left(\tfrac{p_{+}}{p_{-}}\right)^{-d} p_{-}^{n_{+}} p_{+}^{n_{-}} p_0^{n_0}.
    \end{align}
    From \Cref{eq:mult-symmetry} we can simply swap the labels of $n_{-}$ and $n_{+}$ to get the constraint $n_{+} - n_{-} = d$, and plugging into \Cref{eq:label-flip-law} and combining \Cref{eq:law-for-q} gives us the identity
    \begin{align}
        Q_d &= P_{-d} \\
        &= \left(\tfrac{p_{+}}{p_{-}}\right)^{-d} \sum_{\substack{n_{+} + n_{-} + n_0 = s \\ n_{+} - n_{-} = d}} \binom{s}{n_{+}, n_{-}, n_0} p_{+}^{n_{+}} p_{-}^{n_{-}} p_0^{n_0} \\
        &= \left(\tfrac{p_{+}}{p_{-}}\right)^{-d} P_d.
    \end{align}
    All that remains is to integrate the divergence loss $\mu - \beta \nu$ over the region $\Omega_\beta \defeq \{\tilde{y} \in [C]^n \,:\, \mu(\widetilde{y}) > \beta \nu(\widetilde{y})\}$.
    Defining the sets $\Omega_d \defeq \{\widetilde{y} \in [C]^n \,:\, \delta' - \delta = d\}$ for each $-s \le d \le s$, we see that $\{\Omega_d\}_{d=-s}^s$ partitions the outcome space, and in particular integrating $\mu - \beta \nu$ over $\Omega_\beta$ can be done by integrating over each $\Omega_\beta \cap \Omega_d$ and taking the sum.
    Since the likelihood ratio is constant on each $\Omega_d$, this amounts to isolating the $d$ satisfying $(p_{+} / p_{-})^{d} \ge \beta$ and summing the corresponding masses.
    We see that
    \begin{align*}
        (p_{+} / p_{-})^{d} > \beta &\iff d > \frac{\log(\beta)}{\log(p_{+} / p_{-})} \\
        &\iff d \ge \floor*{\frac{\log(\beta)}{\log(p_{+} / p_{-})}} + 1.
    \end{align*}
    Defining $d_\beta \defeq \floor{\log(\beta) / \log(p_{+} / p_{-})} + 1$, we conclude
    \begin{align}
        \HS{\beta}(\mu \| \nu) &= \sum_{d = d_\beta}^{s} \left(P_d - \beta Q_d\right) \\
        &= \sum_{d = d_\beta}^{s} \left(1 - \beta (\tfrac{p_{+}}{p_{-}})^{-d} \right) P_d
    \end{align}

    \paragraph{TV Distance}
    The TV distance follows directly from setting $\beta = 1$ in \Cref{eq:adj-flip-hs} and observing that $d_\beta = 1$.
\end{proof}

The Flipping Measure of \Cref{dfn:flipping-measure} has an interesting history.
It is essentially a $C$-ary randomized response~\citep{warner1965randomized}, well-studied in both the DP and randomized smoothing literatures~\citep{mironov2017renyi,lee2019tight,bojchevski2020efficient,dvijotham2020framework,zhang2022bagflip}.
\citet{mironov2017renyi} computes the $s=1, C=2$ case; see Table II in that work.
The general R\'enyi divergence \Cref{eq:adj-flip-ry} appears in Appendix A.10 of \citet{dvijotham2020framework}, but they do not derive the corresponding $\HS{\beta}$ divergence which can yield tighter certificates via \Cref{thm:tight-certificate-hockey}.
~\citet{lee2019tight} and \citet{zhang2022bagflip} both operate implicitly with the Hockey Stick divergences by way of the Neyman-Pearson Lemma characterization of optimal certificates, although these works elude an efficient closed-form expression and instead compute with a runtime-suboptimal partitioning algorithm with runtime $\cO(d^3 + k^2d^2)$ (the practical impacts of which are discussed in ~\citep{dvijotham2020framework}).
\citet{bojchevski2020efficient} consider flipping measures and identify the $\cO(s)$ likelihood ratio partition structure that allows us to obtain the closed-form \Cref{eq:adj-flip-hs,eq:adj-flip-tv}, but do not explicitly write down the closed-form divergence.

\begin{dfn}[Gaussian--Flipping Measure]
    \label{dfn:gaussian-flipping-measure}
    Let $m, n \in \bbN$, an integer alphabet size $C \ge 2$, $\sigma \in \bbR^+$, and $\gamma \in [0, 1]$. For $j \in [C]$, define the flipping measure
    \begin{equation}
        \phi_{\gamma}(j) \defeq (1 - \gamma) e_{j} + \gamma \unif([C])
    \end{equation}
    on $[C]$.
    Given $x \in \bbR^m$ and $y \in [C]^{n}$, the associated \emph{Gaussian--Flip measure} is the product measure on $\bbR^{m} \times [C]^{n}$
    \begin{equation}
        \mu \defeq \cN(x, \sigma^2 I_{m}) \otimes \bigotimes_{i=1}^n \phi_{\gamma}(y_i).
    \end{equation}

    We say that two Gaussian--Flip measures on the same space $\bbR^{m} \times [C]^{n}$ and same parameters $\sigma, \gamma$, with data $(x, y)$ and $(x', y')$, respectively, are \emph{$(r, s)$-adjacent} if
    \begin{equation}
        \tfrac{\norm{x - x'}_2}{\sigma} = r \quad \text{and} \quad \norm{y - y'}_0 = s.
    \end{equation}
\end{dfn}

The Gaussian--Flipping measure of \Cref{dfn:gaussian-flipping-measure} has two main applications within our framework.
First, it allows joint smoothing of features and labels for dirty-label data poisoning adversaries.
Second, it allows smoothing over feature spaces that combine continuous feature representations with discrete ones, for example of the form $\cX = \bbR^{m} \times K^{n}$ for some discrete space $K$.
In such a setting a data sample can be noised by applying continuous noise (e.g., Gaussian) to the continuous features and discrete noise (e.g., flipping) to the categorical features.

\begin{thm}[Divergences for Adjacent Gaussian--Flipping Measures]\label{thm:gauss-flip-div}
    Let
    \begin{equation*}
        \mu = \cN(\delta, \sigma^2 I_{m}) \otimes \bigotimes_{i=1}^n \phi_{\gamma}(y_i), \qquad
        \nu = \cN(\delta', \sigma^2 I_{m}) \otimes \bigotimes_{i=1}^n \phi_{\gamma}(y'_i)
    \end{equation*}
    be $(r, s)$-adjacent Gaussian--Flip measures on $\bbR^{m} \times [C]^{n}$.
    Define $p_+ \defeq 1 - \gamma + \tfrac{\gamma}{C}$, $p_- \defeq \tfrac{\gamma}{C}$, $p_0 \defeq \tfrac{(C - 2)}{C} \gamma$.
    For $-s \le d \le s$, let $P_d$ be the law of \Cref{eq:label-flip-law} and define the \emph{tilted threshold}
    \begin{equation}
        \label{eq:tilted-threshold}
        \beta_d \defeq \beta \left(\tfrac{p_+}{p_-}\right)^{-d}, \qquad \beta > 0.
    \end{equation}
    Then
    \begin{align}
        \label{eq:adj-gaussflip-ry}
        \Ren[\alpha](\mu \| \nu) &= \frac{\alpha r^2}{2} + \frac{s}{\alpha - 1} \log(p_+^\alpha p_-^{1 - \alpha} + p_+^{1 - \alpha} p_-^\alpha + p_0) \\
        \label{eq:adj-gaussflip-kl}
        \KL(\mu \| \nu) &= \frac{r^2}{2} + s (1 - \gamma) \log \tfrac{p_+}{p_-} \\
        \label{eq:adj-gaussflip-hs}
        \HS{\beta}(\mu \| \nu) &= \sum_{d=-s}^{s} P_d \Big( \Phi\left(\tfrac{r}{2} - \tfrac{\log \beta_d}{r}\right) - \beta_d \Phi\left(-\tfrac{r}{2} - \tfrac{\log \beta_d}{r}\right)\Big) - (1 - \beta)_+ \\
        \label{eq:adj-gaussflip-tv}
        \TV(\mu \| \nu) &= \sum_{d=-s}^s P_d \Big(\Phi\left(\tfrac{r}{2} + \tfrac{d}{r}\log \tfrac{p_+}{p_-}\right) - \left(\tfrac{p_{+}}{p_{-}}\right)^{-d} \Phi\left(-\tfrac{r}{2} + \tfrac{d}{r}\log \tfrac{p_+}{p_-}\right)\Big)
    \end{align}
\end{thm}
\begin{proof}

    \textbf{R\'enyi and KL divergences.} Applying tensorization (\Cref{lem:tensorization}), the result follows by adding the Gaussian values of \Cref{thm:gaussian-measure-divergence} and the flipping values of \Cref{thm:flip-measure-divergence}.

    \textbf{Hockey-stick divergence.} Apply \Cref{lem:hockey-stick-product} with the flipping block as the first factor. Under the flipping factor of $\mu$, the likelihood ratio equals $(p_+/p_-)^d$ with probability $P_d$, so $\HockU[\beta](\mu \| \nu) = \sum_d P_d\, \HockU[\beta_d]\bigl(\cN(x, \sigma^2 I_m) \| \cN(x', \sigma^2 I_m)\bigr)$, where the Gaussian term is $\Phi(\tfrac{r}{2} - \tfrac{\log \beta_d}{r}) - \beta_d \Phi(-\tfrac{r}{2} - \tfrac{\log \beta_d}{r})$. Subtracting $(1 - \beta)_+$ gives \Cref{eq:adj-gaussflip-hs}.

    \textbf{TV Divergence.}
    The TV Divergence bound follows from taking $\beta = 1$ in \Cref{eq:adj-gaussflip-hs}.
\end{proof}

\subsection{Worst-case Uniform parameters}
\label{apx:worst-cast-uniform}

In this appendix we derive the worst-case Uniform noise adjacency parameter under different $\ell_p$ constraints.
Luckily, Uniform noise admits a low-dimensional adjacency parameter $\theta = 1 - \prod_{i}(1 - \tfrac{\abs{x_i - x'_i}}{2\gamma})_+$ (\Cref{dfn:uniform-measure}).
By the proof of \Cref{thm:uniform-measure-divergence}, $\Df(\mu \| \nu) = (f(0) + f'(\infty)) \theta$ for every generator $f$, which is nondecreasing in $\theta$.
Hence, the optimization problem \Cref{eq:worst-case-reduced} is always maximized at the maximum attainable $\theta$.
This reduces the problem to maximizing $\theta$ under an $\ell_p$ constraint.

Write $\Delta \defeq x - x'$.
For a budget $\norm{\Delta}_p \le \varepsilon$, write $t_i \defeq \abs{\delta_i} / (2\gamma)$ and $\eta = \varepsilon / (2\gamma)$.
Define the quantity
\begin{equation}
    \label{eq:unif-log}
    g(t) \defeq -\log(1 - t) \;\text{ for }t < 1,\qquad g(t) \defeq \infty \;\text{ for }t \ge 1.
\end{equation}
so that $- \log (1 - \theta) = \sum_{i=1}^m g(t_i)$.
Evidently $\theta$ is nondecreasing in $\sum_{i=1}^m g(t_i)$, so we seek to maximize $\sum_{i=1}^m g(t_i)$ subject to the constraint $\norm{\mathbf{t}}_p \le \eta$.
If $\eta \ge 1$, the adversary can shift one coordinate by $2 \gamma$, in which case the supports are disjoint, in which case $\theta =1 $ for every $p$.
We therefore assume $\eta < 1$.

\begin{thm}[Worst-case Uniform parameter]
    \label{thm:uniform-worst}
    Let $\eta < 1$, and write $\theta^{\star}_p \defeq \sup \{\theta(x, x') \colon \norm{x - x'}_p \le \varepsilon\}$ for Uniform measures on $\bbR^m$ with radius $\gamma$.
    Define $\Phi_p \defeq -\log(1 - \theta^\star_p)$.
    Then $\theta^{\star}_p = 1 - e^{-\Phi_p}$, where
    \begin{enumerate}
        \item $\Phi_1 = g(\eta)$, attained at $\mathbf{t} = \eta e_1$;
        \item $\Phi_{\infty} = m g(\eta)$, attained at $\mathbf{t} = \eta \mathbf{1}$;
        \item for $1 < p < \infty$ and $m \ge 2$, we have
        \begin{equation}
            \label{eq:unif-worst-phi}
            m^{1-1/p} \eta \le \Phi_p \le m^{1-1/p}\eta + m^{\max(0, 1 - 2/p)}\frac{\eta^2}{2(1 - \eta)}.
        \end{equation}
    \end{enumerate}
\end{thm}
\begin{proof}
    We start with the case $p = 1$.
    Since $g$ is convex with $g(0) = 0$, for any $a, c$ with $a + c < 1$, we have $g(a) \le \tfrac{a}{a + c}g(a + c)$ and $g(c) \le \tfrac{c}{a + c}g(a + c)$, so that $g(a) + g(c) \le g(a + c)$.
    In particular, if $\mathbf{t}$ is nonzero in more than one coordinate, we can obtain a larger objective $sum_{i} g(t_i)$ by merging coordinates.
    Via inductive argument, the objective is maximized at $\mathbf{t} = \eta e_1$.
    The corresponding maximum is $g(\eta)$.

    Now, the case when $p = \infty$.
    This case is easy, since $g$ is increasing in each coordinate, we need only maximize each coordinate independently to obtain $g(\mathbf{t}) = m g(\eta)$.

    Finally, the case when $1 < p < \infty$.
    For $0 \le t < 1$, we have the series representation $g(t) = \sum_{j \ge 1} t^j / j$, which gives the bound $g(t) \le t + \frac{1}{2}\sum_{j\ge2}t^j = t + t^2/(2(1 - t))$.
    Every coordinate satisfies $t_i \le \norm{\mathbf{t}} \le \eta$, so
    \begin{equation*}
        \sum_{i=1}^m g(t_i) \le \norm{\mathbf{t}}_1 + \frac{\norm{\mathbf{t}}_2^2}{2(1 - \eta)}.
    \end{equation*}
    By H\"older's inequality, $\norm{\mathbf{t}}_1 \le m^{1 - 1/p} \norm{\mathbf{t}}_p$ and $\norm{\mathbf{t}}_2 \le m^{\max(0, 1/2 - 1/p)} \norm{\mathbf{t}}_p$, giving the upper bound.
    For the lower bound, we can take $t_i = \eta m^{-1/p}$ so that $\norm{\mathbf{t}}_p = \eta$.
    We conclude that
    \begin{equation*}
        \sum_{i=1}^m g(t_i) = m g(\eta m^{-1/p}) \ge m^{1 - 1/p}\eta,
    \end{equation*}
    where the final inequality uses that $g(s) \ge s$.
\end{proof}

\subsection{Worst-case Laplace parameters}
\label{apx:worst-case-laplace}

To bound worst-case divergences for Laplace noise, we must solve \Cref{eq:worst-case-reduced} applied to \Cref{thm:laplace-measure-divergence}.
Unlike other noises like Gaussian or Uniform, the adjacency parameter for Laplace noise is the list-valued $\mathbf{t}$ in \Cref{eq:laplace-adjacency}.
Concretely, in this appendix we attempt to solve (or bound)
\begin{equation}
    \label{eq:laplace-worst-parameter}
    \tau^\star_f(\varepsilon, p) = \sup \left\{\Df(\mu_{\mathbf{t}} \| \nu) \;:\; \mathbf{t} \in \bbR^{m}_{\ge 0},\; \norm{\mathbf{t}}_p \le \eta \right\}, \qquad \eta \defeq \varepsilon / b.
\end{equation}
Throughout this appendix we write $\nu \defeq \Lap(0, b)^{\otimes m}$ and $\mu_{\mathbf{t}}$ for the measure centered at $b \mathbf{t}$, and take $b = 1$.

We summarize our findings here.
The $\ell_1$ and $\ell_\infty$ cases are solvable exactly for all $f$-divergences.
For $\ell_\infty$, the argument is straightforward: every divergence is monotone in each coordinate $\mathbf{t}$, so is maximized at $\mathbf{t} = \eta \mathbf{1}$.
The result for $\ell_\infty$ is stated in \Cref{thm:laplace-worst-linf}.
For $\ell_1$, the divergence is maximized when the adversary concentrates their perturbation in a single coordinate.
The result for $\ell_1$ is stated in \Cref{thm:laplace-worst-l1}.

\paragraph{Technical lemmas}
We first present some technical lemmas which will help us understand the structure of Laplace noise.

\begin{lem}[Coordinate-wise monotonicity]
    \label{lem:laplace-monotonicity}
    For $0 \le t \le t'$, $(\Lap(t, 1), \Lap(0, 1)) \dle (\Lap(t', 1), \Lap(0, 1))$.
    In particular, $(\mu_{\mathbf{t}}, \nu) \dle (\mu_{\mathbf{t}'}, \nu)$ whenever $\mathbf{t} \le \mathbf{t'}$ coordinate-wise.
\end{lem}
\begin{proof}
    By \Cref{eq:dle-via-hs} it suffices to compare the Hockey-stick profiles at every $\beta > 0$, which are determined by \Cref{eq:adj-laplace-hs} as $\Hock(\Lap(t, 1) \| \Lap(0, 1)) = [(1 - \sqrt{\beta} e^{-t/2})_+ - (1 - \beta)_+]_+$.
    By inspection, this quantity is nondecreasing in $t$ for every $\beta$, which gives the first claim.
    
    The second claim follows from the first by applying \Cref{lem:worst-case-product} for each coordinate.
\end{proof}

The following technical lemma characterizes the Hockey-stick behavior under an $\ell_1$ constraint when the adversary splits the perturbation between two coordinates.

\begin{lem}
    \label{lem:laplace-two-coordinate}
    Let $s, t > 0$ and write $\eta \defeq s + t$.
    Let $\mu_{(s,t)},\nu$ be $(s,t)$-adjacent Laplace measures on $\bbR^2$.
    Then, for every $\beta \in (e^{-\eta}, e^{\eta})$,
    \begin{equation}
        \label{eq:laplace-two-coordinate}
        \HockU[\beta](\mu_{(s,t)} \| \nu) = 1 - \sqrt{\beta}e^{-\eta/2}(1 + \Delta),\qquad \Delta \defeq \tfrac{1}{2}\min\{\ell, s, t, \eta - \ell\}, \quad \ell\defeq \tfrac{\eta + \log \beta}{2},
    \end{equation}
    while $\HockU[\beta](\mu_{(s,t)} \| \nu) = 1 - \beta$ for $\beta < e^{-\eta}$ and 0 for $\beta \ge e^{\eta}$.
    In particular, $\Delta \ge 0$, so $\HockU[\beta](\mu_{(s,t)} \| \nu) \le \HockU[\beta](\mu_{\eta e_1} \| \nu)$ for every $\beta > 0$, with equality if and only if $\beta \notin (e^{-\eta}, e^{\eta})$.
\end{lem}
\begin{proof}
    First, when $\beta \notin (e^{-\eta}, e^{\eta})$, the corresponding results hold since $\abs{\log \Lambda} \le \eta$ almost surely by \Cref{thm:laplace-measure-divergence}.
    Fix $\beta \in (e^{-\eta}, e^{\eta})$, and write $\mu_s \defeq \Lap(s, 1)$ and $\mu_t \defeq \Lap(t, 1)$ for the marginals of $\mu_{(s,t)}$, $\nu_0 \defeq \Lap(0, 1)$ for the common marginal of $\nu$, and $\Lambda_t$ for the likelihood ratio of $(\mu_t, \nu_0)$.

    Applying \Cref{lem:hockey-stick-product} to $\mu_{(s,t)}, \nu$ (with $\mu_t$ as the first factor), we have
    \begin{equation*}
        \HockU[\beta](\mu_{(s,t)} \| \nu) = \bbE_{\mu_t}[\HockU[\beta / \Lambda_t](\mu_s \| \nu_0)],
    \end{equation*}
    where by \Cref{eq:adj-laplace-hs},
    \begin{equation}
        \label{eq:two-coord-mixture}
        \HockU[\beta / \Lambda_t](\mu_s \| \nu_0)
        = 
        \begin{cases}
            0, &\Lambda_t \le \beta e^{-s}, \\
            1 - \sqrt{\beta / \Lambda_t}e^{-s/2}, &\beta e^{-s} \le \Lambda_t \le \beta e^{s}, \\
            1 - \beta / \Lambda_t, &\Lambda_t \ge \beta e^{s}.
        \end{cases}
    \end{equation}
    Since $\Lambda_t \in [e^{-t}, e^{t}]$, the branches of \Cref{eq:two-coord-mixture} entering into the expectation depend on $s$ and $t$.
    The first line applies exactly when $\beta e^{-s} \ge e^{-t}$; that is, $\ell \ge s$, and the third exactly when $\beta e^s \le e^t$; that is, $\ell \le t$.

    By \Cref{thm:laplace-measure-divergence}, we have $\Lambda_t = e^{-t} V_t$, so under $\nu_0$, it follows that $\Lambda_t$ is distributed with mass $\tfrac{1}{2}$ at $e^{-t}$, mass $\tfrac{1}{2}e^{-t}$ at $e^t$, and density $\tfrac{1}{4}e^{-t/2}\lambda^{-3/2}$ on $(e^{-t}, e^t)$.
    Integrating the measure gives, for $x \in (e^{-t}, e^t]$,
    \begin{align*}
        \bbP_{\nu_0}[\Lambda_t \ge x] &= \tfrac{1}{2}e^{-t/2}x^{-1/2}, \\
        \bbE_{\nu_0}[\Lambda_t; \Lambda_t \ge x] &= 1 - \tfrac{1}{2}e^{-t/2}x^{1/2}, \\
        \bbE_{\nu_0}[\Lambda_t^{1/2}; \Lambda_t < x] &= e^{-t/2}(\tfrac{1}{2} + \tfrac{1}{4}\log(xe^t)), \\
        \bbE_{\nu_0}[\Lambda_t^{1/2}] &= \Hint[1/2](\mu_t \| \nu_0) = \zeta_{1/2}(t).
    \end{align*}

    We have four meaningful cases depending on the ordering of $s, t, \ell$.
    \paragraph{Case $t \le \ell \le s$}
    In this case only the second line of \Cref{eq:two-coord-mixture} applies, so
    \begin{equation*}
        \HockU[\beta](\mu_{(s,t)} \| \nu) = 1 - \sqrt{\beta}e^{-s/2} \bbE_{\nu_0}[\Lambda_t^{1/2}] = 1 - \sqrt{\beta}e^{-s/2}e^{-t/2}(1 + \tfrac{t}{2}) = 1 - \sqrt{\beta}e^{-\eta/2}(1 + \tfrac{t}{2}).
    \end{equation*}
    In this case $\min\{\ell, s, t, \eta - \ell\} = t$.
    
    \paragraph{Case $\ell \ge \max(s, t)$}
    In this case the first line of \Cref{eq:two-coord-mixture} applies on $\{\Lambda_t \le \beta e^{-s}\}$ and the second on the rest, so
    \begin{align*}
        \HockU[\beta](\mu_{(s,t)} \| \nu) &= \bbE_{\nu_0}[\Lambda_t; \Lambda_t \ge \beta e^{-s}] - \sqrt{\beta} e^{-s/2} \bbE_{\nu_0}[\Lambda_t^{1/2}; \Lambda_t \ge \beta e^{-s}] \\
        &= 1 - \sqrt{\beta}e^{-\eta/2}(1 + \tfrac{1}{2}(\eta - \ell)).
    \end{align*}
    In this case $\min\{\ell, s, t, \eta - \ell\} = \eta - \ell$.
    
    \paragraph{Case $\ell \le \min(s, t)$}
    In this case the third line of \Cref{eq:two-coord-mixture} applies on $\{\Lambda_t \ge \beta e^s\}$ and the second on the rest, so
    \begin{align*}
        \HockU[\beta](\mu_{(s,t)} \| \nu) &= 1 -\sqrt{\beta}e^{-s/2}\bbE_{\nu_0}[\Lambda_t^{1/2}; \Lambda_t < \beta e^s] - \beta \bbP[\Lambda_t \ge \beta e^s] \\
        &= 1 - \sqrt{\beta}e^{-\eta/2}(1 + \tfrac{\ell}{2}).
    \end{align*}
    In this case $\min\{\ell, s, t, \eta - \ell\} = \ell$.

    \paragraph{Case $s \le \ell \le t$}
    In this case $\mu_{(s,t)}$ and $\mu_{(t,s)}$ are related by swapping coordinates.
    By \Cref{lem:data-processing} both sides of \Cref{eq:laplace-two-coordinate} are symmetric in $s$ and $t$, so we can reduce to the case $t \le \ell \le s$ handled above.

    In all four cases $\Delta > 0$ since $s, t > 0$ and $0 < \ell < \eta$.
    For the final comparison, \Cref{eq:adj-laplace-hs} with $t = \eta$ gives $\HockU[\beta](\mu_{\eta e_1} \| \nu) = 1 - \sqrt{\beta} e^{-\eta/2}$ for $\beta \in (e^{-\eta}, e^{\eta})$.
    This is \Cref{eq:laplace-two-coordinate} when $\Delta = 0$.
\end{proof}

The property that ultimately allows us to characterize Laplace noise for $\ell_1$ adversaries is \textit{Schur convexity}.
Recall that we say $\mathbf{t}$ is \textit{majorized} by $\mathbf{t}'$, written $\mathbf{t} \prec \mathbf{t}'$, if the two have equal sums and every prefix sum of the coordinates of $\mathbf{t}$ when listed in nonincreasing order is at most the corresponding sum for $\mathbf{t}'$.
If the prefix-sum condition holds but the sums of $\mathbf{t}$ and $\mathbf{t}'$ may differ, we say $\mathbf{t}$ is \textit{weakly submajorized} by $\mathbf{t}'$, written $\mathbf{t} \prec_w \mathbf{t}'$.

\begin{lem}[Schur convexity in Laplace parameter]
    \label{lem:laplace-schur}
    For every generator $f$, the map $\mathbf{t} \mapsto \Df(\mu_{\mathbf{t}} \| \nu)$ is Schur convex.
    In particular, if $\mathbf{t} \dle \mathbf{t}'$, then $(\mu_{\mathbf{t}}, \nu) \prec (\mu_{\mathbf{t}'}, \nu)$.
\end{lem}
\begin{proof}
    Suppose $\mathbf{t} \prec \mathbf{t}'$.
    As a property of the majorization relation, it holds that there exists a sequence $\mathbf{t} = \mathbf{t}^{(0)} \prec \mathbf{t}^{(1)} \prec \cdots \prec \mathbf{t}^{(n)} = \mathbf{t}'$ such that each adjacent pair differs only a single pair of coordinates $i,j$ and with $\mathbf{t}^{(k)}_{i} + \mathbf{t}^{(k)}_j = \mathbf{t}^{(k+1)}_{i} + \mathbf{t}^{(k+1)}_j$.
    Thus, it is sufficient to prove the result only for the case when $\mathbf{t},\mathbf{t}'$ differ in a single coordinate pair with $\mathbf{t}_i + \mathbf{t}_j = \mathbf{t}'_i + \mathbf{t}'_j$ and $\min(\mathbf{t}_i, \mathbf{t}_j) \ge \min(\mathbf{t}'_i,\mathbf{t}'_j)$.

    Applying \Cref{eq:laplace-two-coordinate} on the two swapped coordinates for $\textbf{t}$ and $\mathbf{t}'$, with $\eta = \mathbf{t}_i + \mathbf{t}_j$ and fixed $\beta$, we see that $\Delta \ge \Delta'$ because $\Delta$ is nondecreasing in $\min(\cdot, \cdot)$.
    In particular, the pair $(\mathbf{t}_i, \mathbf{t}_j)$ generates a smaller Hockey-stick divergence under \Cref{lem:laplace-two-coordinate}.

    Finally, applying \Cref{lem:worst-case-product}, we conclude that $(\mu_{\mathbf{t}}, \nu) \dle (\mu_{\mathbf{t}'}, \nu)$.
\end{proof}

\paragraph{$\ell_\infty$ constraints}
\begin{thm}[Worst-case $\ell_\infty$ Laplace parameter]
    \label{thm:laplace-worst-linf}
    Let $\mathbf{t} \in \bbR^m_{\ge 0}$ with $\norm{\mathbf{t}}_\infty \le \eta$.
    Then $(\mu_{\mathbf{t}}, \nu) \dle (\mu_{\eta \mathbf{1}}, \nu)$, and the $\mathbf{t}$ maximizing \Cref{eq:laplace-worst-parameter} is $\mathbf{t}^\star = \eta \mathbf{1}$ and $\tau^\star_f(\varepsilon, \infty) = \Df(\mu_{\eta \mathbf{1}} \| \nu)$ for every generator $f$.
    In particular
    \begin{equation}
        \label{eq:laplace-worst-linf}
        \tau^\star_{\mathrm{R}, \alpha}(\varepsilon, \infty) = m \kappa_{\alpha}(\eta),
    \end{equation}
    and $\tau^\star_{\mathrm{HS},\beta}(\varepsilon, \infty)$ is obtained by convolving $m$ copies of the distribution in \Cref{eq:laplace-coordinate-law} at $t = \eta$ and evaluating \Cref{eq:f-divergence-lambda}.
\end{thm}
\begin{proof}
    The constraint $\norm{\mathbf{t}}_\infty \le \eta$ is exactly $\mathbf{t} \le \eta \mathbf{1}$ coordinate-wise.
    By \Cref{lem:laplace-monotonicity}, we have $(\mu_{\mathbf{t}}, \nu)\ \dle (\mu_{\eta \mathbf{1}}, \nu)$.
    \Cref{eq:laplace-worst-linf} follows by evaluating \Cref{eq:adj-laplace-ry} at $\mathbf{t} = \eta \mathbf{1}$.
\end{proof}

\paragraph{$\ell_1$ constraints}
\begin{thm}[Worst-case $\ell_1$ Laplace parameter]
    \label{thm:laplace-worst-l1}
    Let $\mathbf{t} \in \bbR^m_{\ge 0}$ with $\norm{\mathbf{t}}_1 \le \eta$.
    Then, $(\mu_{\mathbf{t}}, \nu) \dle (\Lap(b \eta, b), \Lap(0, b))$.
    In particular, the $\mathbf{t}$ maximizing \Cref{eq:laplace-worst-parameter} is $\mathbf{t}^\star = \eta e_1$ and, for every $\beta \ge 1$,
    \begin{equation}
        \label{eq:laplace-worst-l1}
        \tau_{\mathrm{HS}, \beta}^\star(\varepsilon, 1) = \left(1 - \sqrt{\beta} e^{-\varepsilon / 2 b}\right)_+, \qquad \tau_{\mathrm{R}, \alpha}^\star(\varepsilon, 1) = \kappa_{\alpha}(\eta),
    \end{equation}
    with $\beta < 1$ determined by \Cref{lem:hockey-stick-reversal}.
\end{thm}
\begin{proof}
    Suppose first that $\norm{\mathbf{t}}_1 = \eta$ exactly.
    Then $\mathbf{t} \prec \eta e_1$, since the two have equal sums and the largest partial sum of $\eta e_1$ is $\eta$, which dominates every partial sum of $\mathbf{t}$.
    \Cref{lem:laplace-schur} therefore gives $(\mu_{\mathbf{t}}, \nu) \dle (\mu_{\eta e_1}, \nu)$.
    If instead $\norm{\mathbf{t}}_1 = \eta' < \eta$, apply the above to obtain $(\mu_{\mathbf{t}}, \nu) \dle (\mu_{\eta' e_1}, \nu)$ and then \Cref{lem:laplace-monotonicity} in the single coordinate to obtain $(\mu_{\eta' e_1}, \nu) \dle (\mu_{\eta e_1}, \nu)$.
    Since $\mu_{\eta e_1}$ and $\nu$ agree in all coordinates but the first, canceling common factors (\Cref{lem:f-cancel}) identifies the pair $(\mu_{\eta e_1}, \nu)$ with the scalar pair $(\Lap(b\eta, b), \Lap(0,b))$, proving the first claim.
    \Cref{eq:laplace-worst-l1} then follows from \Cref{eq:adj-laplace-hs,eq:adj-laplace-ry} evaluated at $\mathbf{t} = \eta e_1$.
\end{proof}

\paragraph{$\ell_p$ constraints}
For $1 < p < \infty$, no single $\mathbf{t}$ maximizes the Hockey-stick divergences of all orders simultaneously.
We instead bound every admissible $\mathbf{t}$ by a dominating parameter, which concentrates the perturbation into as few coordinates as the budget allows.

\begin{thm}[Dominating $\ell_p$ Laplace parameter]
    \label{thm:laplace-dominating}
    Let $1 < p < \infty$ and $\mathbf{t} \in \bbR^m_{\ge 0}$ with $\norm{\mathbf{t}}_p \le \eta$.
    Let $K \defeq \lceil m^{1 - 1/p} \rceil$ and $\mathbf{t}^\star \defeq \eta\,(\mathbf{1}_K, \mathbf{0}_{m - K})$.
    Then $(\mu_{\mathbf{t}}, \nu) \dle (\mu_{\mathbf{t}^\star}, \nu)$.
    In particular, $\tau^\star_f(\varepsilon, p) \le \Df(\mu_{\mathbf{t}^\star} \| \nu)$ for every generator $f$, and the right-hand side is obtained by convolving $K$ copies of the distribution in \Cref{eq:laplace-coordinate-law} at $t = \eta$.
    For $p = 2$, $K = \lceil \sqrt{m}\, \rceil$.
\end{thm}
\begin{proof}
    Every coordinate satisfies $\mathbf{t}_j \le \norm{\mathbf{t}}_p \le \eta$, and by H\"older's inequality $\norm{\mathbf{t}}_1 \le m^{1 - 1/p}\norm{\mathbf{t}}_p \le K\eta$.
    Listing $\mathbf{t}$ in nonincreasing order, the sum of its $j$ largest coordinates is therefore at most $j\eta$ for $j \le K$ and at most $K\eta$ for $j > K$.
    These are exactly the prefix sums of $\mathbf{t}^\star$, so $\mathbf{t} \prec_w \mathbf{t}^\star$.
    As a property of weak majorization, there exists $\mathbf{u}$ with $\mathbf{t} \le \mathbf{u}$ coordinate-wise and $\mathbf{u} \prec \mathbf{t}^\star$.
    \Cref{lem:laplace-monotonicity} gives $(\mu_{\mathbf{t}}, \nu) \dle (\mu_{\mathbf{u}}, \nu)$, and \Cref{lem:laplace-schur} gives $(\mu_{\mathbf{u}}, \nu) \dle (\mu_{\mathbf{t}^\star}, \nu)$.
\end{proof}

%% file: main.bbl
\begin{thebibliography}{45}
\providecommand{\natexlab}[1]{#1}
\providecommand{\url}[1]{\texttt{#1}}
\expandafter\ifx\csname urlstyle\endcsname\relax
  \providecommand{\doi}[1]{doi: #1}\else
  \providecommand{\doi}{doi: \begingroup \urlstyle{rm}\Url}\fi

\bibitem[Aghakhani et~al.(2021)Aghakhani, Meng, Wang, Kruegel, and Vigna]{aghakhani2021bullseye}
Hojjat Aghakhani, Dongyu Meng, Yu-Xiang Wang, Christopher Kruegel, and Giovanni Vigna.
\newblock Bullseye polytope: A scalable clean-label poisoning attack with improved transferability.
\newblock In \emph{2021 IEEE European Symposium on Security and Privacy (EuroS\&P)}, pp.\  159--178, 2021.
\newblock \doi{10.1109/EuroSP51992.2021.00021}.

\bibitem[Athalye et~al.(2018)Athalye, Carlini, and Wagner]{athalye2018obfuscated}
Anish Athalye, Nicholas Carlini, and David Wagner.
\newblock Obfuscated gradients give a false sense of security: Circumventing defenses to adversarial examples.
\newblock In Jennifer Dy and Andreas Krause (eds.), \emph{Proceedings of the 35th International Conference on Machine Learning}, volume~80 of \emph{Proceedings of Machine Learning Research}, pp.\  274--283. PMLR, July 2018.
\newblock URL \url{https://proceedings.mlr.press/v80/athalye18a.html}.

\bibitem[Balle \& Wang(2018)Balle and Wang]{balle2018improving}
Borja Balle and Yu-Xiang Wang.
\newblock Improving the {G}aussian mechanism for differential privacy: Analytical calibration and optimal denoising.
\newblock In Jennifer Dy and Andreas Krause (eds.), \emph{Proceedings of the 35th International Conference on Machine Learning}, volume~80 of \emph{Proceedings of Machine Learning Research}, pp.\  394--403. PMLR, July 2018.
\newblock URL \url{https://proceedings.mlr.press/v80/balle18a.html}.

\bibitem[Bojchevski et~al.(2020)Bojchevski, Gasteiger, and G{\"u}nnemann]{bojchevski2020efficient}
Aleksandar Bojchevski, Johannes Gasteiger, and Stephan G{\"u}nnemann.
\newblock Efficient robustness certificates for discrete data: Sparsity-aware randomized smoothing for graphs, images and more.
\newblock In Hal~Daumé III and Aarti Singh (eds.), \emph{Proceedings of the 37th International Conference on Machine Learning}, volume 119 of \emph{Proceedings of Machine Learning Research}, pp.\  1003--1013. PMLR, July 2020.
\newblock URL \url{https://proceedings.mlr.press/v119/bojchevski20a.html}.

\bibitem[Chen et~al.(2022)Chen, Li, Li, Yan, and Wu]{chen2022collective}
Ruoxin Chen, Zenan Li, Jie Li, Junchi Yan, and Chentao Wu.
\newblock On collective robustness of bagging against data poisoning.
\newblock In Kamalika Chaudhuri, Stefanie Jegelka, Le~Song, Csaba Szepesvari, Gang Niu, and Sivan Sabato (eds.), \emph{Proceedings of the 39th International Conference on Machine Learning}, volume 162 of \emph{Proceedings of Machine Learning Research}, pp.\  3299--3319. PMLR, July 2022.
\newblock URL \url{https://proceedings.mlr.press/v162/chen22k.html}.

\bibitem[Clopper \& Pearson(1934)Clopper and Pearson]{clopper1934use}
C.~J. Clopper and Egon~S. Pearson.
\newblock The use of confidence or fiducial limits illustrated in the case of the binomial.
\newblock \emph{Biometrika}, 26\penalty0 (4):\penalty0 404--413, December 1934.
\newblock ISSN 0006-3444.
\newblock \doi{10.1093/biomet/26.4.404}.
\newblock URL \url{https://doi.org/10.1093/biomet/26.4.404}.

\bibitem[Cohen et~al.(2019)Cohen, Rosenfeld, and Kolter]{cohen2019certified}
Jeremy Cohen, Elan Rosenfeld, and Zico Kolter.
\newblock Certified adversarial robustness via randomized smoothing.
\newblock In Kamalika Chaudhuri and Ruslan Salakhutdinov (eds.), \emph{Proceedings of the 36th International Conference on Machine Learning}, volume~97 of \emph{Proceedings of Machine Learning Research}, pp.\  1310--1320. PMLR, June 2019.
\newblock URL \url{https://proceedings.mlr.press/v97/cohen19c.html}.

\bibitem[Csiszár \& Shields(2004)Csiszár and Shields]{csiszar2004information}
Imre Csiszár and Paul~C. Shields.
\newblock Information theory and statistics: A tutorial.
\newblock \emph{Foundations and Trends in Communications and Information Theory}, 1\penalty0 (4):\penalty0 417--528, November 2004.
\newblock ISSN 1567-2190.
\newblock \doi{10.1561/0100000004}.
\newblock URL \url{https://doi.org/10.1561/0100000004}.

\bibitem[Drews et~al.(2020)Drews, Albarghouthi, and D'Antoni]{drews2020proving}
Samuel Drews, Aws Albarghouthi, and Loris D'Antoni.
\newblock Proving data-poisoning robustness in decision trees.
\newblock In \emph{Proceedings of the 41st ACM SIGPLAN Conference on Programming Language Design and Implementation}, PLDI 2020, pp.\  1083--1097, New York, NY, USA, 2020. Association for Computing Machinery.
\newblock ISBN 9781450376136.
\newblock \doi{10.1145/3385412.3385975}.
\newblock URL \url{https://doi.org/10.1145/3385412.3385975}.

\bibitem[Dvijotham et~al.(2020)Dvijotham, Hayes, Balle, Kolter, Qin, Gyorgy, Xiao, Gowal, and Kohli]{dvijotham2020framework}
Krishnamurthy~(Dj) Dvijotham, Jamie Hayes, Borja Balle, Zico Kolter, Chongli Qin, Andras Gyorgy, Kai Xiao, Sven Gowal, and Pushmeet Kohli.
\newblock A framework for robustness certification of smoothed classifiers using f-divergences.
\newblock In \emph{International Conference on Learning Representations}, 2020.
\newblock URL \url{https://openreview.net/forum?id=SJlKrkSFPH}.

\bibitem[Geiping et~al.(2021)Geiping, Fowl, Huang, Czaja, Taylor, Moeller, and Goldstein]{geiping2021witches}
Jonas Geiping, Liam~H Fowl, W.~Ronny Huang, Wojciech Czaja, Gavin Taylor, Michael Moeller, and Tom Goldstein.
\newblock Witches' brew: Industrial scale data poisoning via gradient matching.
\newblock In \emph{International Conference on Learning Representations}, 2021.
\newblock URL \url{https://openreview.net/forum?id=01olnfLIbD}.

\bibitem[Gosch et~al.(2026)Gosch, Chen, Scholten, and G{\"u}nnemann]{gosch2026certifying}
Lukas Gosch, Xichuan Chen, Yan Scholten, and Stephan G{\"u}nnemann.
\newblock Certifying graph neural networks against label and structure poisoning.
\newblock In \emph{Forty-third International Conference on Machine Learning}, 2026.
\newblock URL \url{https://openreview.net/forum?id=kQuvO5gaNr}.

\bibitem[Gu et~al.(2019)Gu, Liu, Dolan-Gavitt, and Garg]{gu2019badnets}
Tianyu Gu, Kang Liu, Brendan Dolan-Gavitt, and Siddharth Garg.
\newblock Badnets: Evaluating backdooring attacks on deep neural networks.
\newblock \emph{IEEE Access}, 7:\penalty0 47230--47244, 2019.
\newblock \doi{10.1109/ACCESS.2019.2909068}.

\bibitem[Hammoudeh \& Lowd(2024)Hammoudeh and Lowd]{hammoudeh2024provable}
Zayd Hammoudeh and Daniel Lowd.
\newblock Provable robustness against a union of $\ell_0$ adversarial attacks.
\newblock \emph{Proceedings of the AAAI Conference on Artificial Intelligence}, 38\penalty0 (19):\penalty0 21134--21142, March 2024.
\newblock \doi{10.1609/aaai.v38i19.30106}.
\newblock URL \url{https://ojs.aaai.org/index.php/AAAI/article/view/30106}.

\bibitem[Jia et~al.(2021)Jia, Cao, and Gong]{jia2021intrinsic}
Jinyuan Jia, Xiaoyu Cao, and Neil~Zhenqiang Gong.
\newblock Intrinsic certified robustness of bagging against data poisoning attacks.
\newblock \emph{Proceedings of the AAAI Conference on Artificial Intelligence}, 35\penalty0 (9):\penalty0 7961--7969, May 2021.
\newblock \doi{10.1609/aaai.v35i9.16971}.
\newblock URL \url{https://ojs.aaai.org/index.php/AAAI/article/view/16971}.

\bibitem[Jia et~al.(2022)Jia, Liu, Cao, and Gong]{jia2022certified}
Jinyuan Jia, Yupei Liu, Xiaoyu Cao, and Neil~Zhenqiang Gong.
\newblock Certified robustness of nearest neighbors against data poisoning and backdoor attacks.
\newblock \emph{Proceedings of the AAAI Conference on Artificial Intelligence}, 36\penalty0 (9):\penalty0 9575--9583, June 2022.
\newblock ISSN 2159-5399.
\newblock \doi{10.1609/aaai.v36i9.21191}.
\newblock URL \url{http://dx.doi.org/10.1609/aaai.v36i9.21191}.

\bibitem[Krizhevsky(2009)]{krizhevsky2009learning}
Alex Krizhevsky.
\newblock Learning multiple layers of features from tiny images.
\newblock Technical report, University of Toronto, 2009.

\bibitem[Lecun et~al.(1998)Lecun, Bottou, Bengio, and Haffner]{lecun1998gradient}
Y.~Lecun, L.~Bottou, Y.~Bengio, and P.~Haffner.
\newblock Gradient-based learning applied to document recognition.
\newblock \emph{Proceedings of the IEEE}, 86\penalty0 (11):\penalty0 2278--2324, 1998.
\newblock \doi{10.1109/5.726791}.

\bibitem[Lecuyer et~al.(2019)Lecuyer, Atlidakis, Geambasu, Hsu, and Jana]{lecuyer2019certified}
Mathias Lecuyer, Vaggelis Atlidakis, Roxana Geambasu, Daniel Hsu, and Suman Jana.
\newblock Certified robustness to adversarial examples with differential privacy.
\newblock In \emph{2019 IEEE Symposium on Security and Privacy (SP)}, pp.\  656--672, 2019.
\newblock \doi{10.1109/SP.2019.00044}.

\bibitem[Lee et~al.(2019)Lee, Yuan, Chang, and Jaakkola]{lee2019tight}
Guang-He Lee, Yang Yuan, Shiyu Chang, and Tommi Jaakkola.
\newblock Tight certificates of adversarial robustness for randomly smoothed classifiers.
\newblock In H.~Wallach, H.~Larochelle, A.~Beygelzimer, F.~d\textquotesingle Alch\'{e}-Buc, E.~Fox, and R.~Garnett (eds.), \emph{Advances in Neural Information Processing Systems}, volume~32. Curran Associates, Inc., 2019.
\newblock URL \url{https://proceedings.neurips.cc/paper_files/paper/2019/file/fa2e8c4385712f9a1d24c363a2cbe5b8-Paper.pdf}.

\bibitem[Levine \& Feizi(2021)Levine and Feizi]{levine2021deep}
Alexander Levine and Soheil Feizi.
\newblock Deep partition aggregation: Provable defenses against general poisoning attacks.
\newblock In \emph{International Conference on Learning Representations}, 2021.
\newblock URL \url{https://openreview.net/forum?id=YUGG2tFuPM}.

\bibitem[Li et~al.(2019)Li, Chen, Wang, and Carin]{li2019certified}
Bai Li, Changyou Chen, Wenlin Wang, and Lawrence Carin.
\newblock Certified adversarial robustness with additive noise.
\newblock In H.~Wallach, H.~Larochelle, A.~Beygelzimer, F.~d\textquotesingle Alch\'{e}-Buc, E.~Fox, and R.~Garnett (eds.), \emph{Advances in Neural Information Processing Systems}, volume~32. Curran Associates, Inc., 2019.
\newblock URL \url{https://proceedings.neurips.cc/paper_files/paper/2019/file/335cd1b90bfa4ee70b39d08a4ae0cf2d-Paper.pdf}.

\bibitem[Liese \& Vajda(2006)Liese and Vajda]{liese2006divergences}
Friedrich Liese and Igor Vajda.
\newblock On divergences and informations in statistics and information theory.
\newblock \emph{IEEE Transactions on Information Theory}, 52\penalty0 (10):\penalty0 4394--4412, 2006.
\newblock \doi{10.1109/TIT.2006.881731}.

\bibitem[Liu et~al.(2023)Liu, Cullen, Montague, Erfani, and Rubinstein]{liu2023enhancing}
Shijie Liu, Andrew~C. Cullen, Paul Montague, Sarah~M. Erfani, and Benjamin I.~P. Rubinstein.
\newblock Enhancing the antidote: Improved pointwise certifications against poisoning attacks.
\newblock \emph{Proceedings of the AAAI Conference on Artificial Intelligence}, 37\penalty0 (7):\penalty0 8861--8869, June 2023.
\newblock \doi{10.1609/aaai.v37i7.26065}.
\newblock URL \url{https://ojs.aaai.org/index.php/AAAI/article/view/26065}.

\bibitem[Ma et~al.(2019)Ma, Zhu, and Hsu]{ma2019data}
Yuzhe Ma, Xiaojin Zhu, and Justin Hsu.
\newblock Data poisoning against differentially-private learners: Attacks and defenses.
\newblock In \emph{Proceedings of the Twenty-Eighth International Joint Conference on Artificial Intelligence, {IJCAI-19}}, pp.\  4732--4738. International Joint Conferences on Artificial Intelligence Organization, July 2019.
\newblock \doi{10.24963/ijcai.2019/657}.
\newblock URL \url{https://doi.org/10.24963/ijcai.2019/657}.

\bibitem[Mironov(2017)]{mironov2017renyi}
Ilya Mironov.
\newblock Rényi differential privacy.
\newblock In \emph{2017 IEEE 30th Computer Security Foundations Symposium (CSF)}, pp.\  263--275, 2017.
\newblock \doi{10.1109/CSF.2017.11}.

\bibitem[Neyman \& Pearson(1933)Neyman and Pearson]{neyman1933problem}
Jerzy Neyman and Egon~Sharpe Pearson.
\newblock On the problem of the most efficient tests of statistical hypotheses.
\newblock \emph{Philosophical Transactions of the Royal Society of London, Series A: Containing Papers of a Mathematical or Physical Character}, 231\penalty0 (694-706):\penalty0 289--337, February 1933.
\newblock ISSN 0264-3952.
\newblock \doi{10.1098/rsta.1933.0009}.
\newblock URL \url{https://doi.org/10.1098/rsta.1933.0009}.

\bibitem[Polyanskiy \& Wu(2025)Polyanskiy and Wu]{polyanskiy2025information}
Yury Polyanskiy and Yihong Wu.
\newblock \emph{Information Theory: From Coding to Learning}.
\newblock Cambridge University Press, 2025.

\bibitem[Rezaei et~al.(2023)Rezaei, Banihashem, Chegini, and Feizi]{rezaei2023runoff}
Keivan Rezaei, Kiarash Banihashem, Atoosa Chegini, and Soheil Feizi.
\newblock Run-off election: Improved provable defense against data poisoning attacks.
\newblock In Andreas Krause, Emma Brunskill, Kyunghyun Cho, Barbara Engelhardt, Sivan Sabato, and Jonathan Scarlett (eds.), \emph{Proceedings of the 40th International Conference on Machine Learning}, volume 202 of \emph{Proceedings of Machine Learning Research}, pp.\  29030--29050. PMLR, July 2023.
\newblock URL \url{https://proceedings.mlr.press/v202/rezaei23a.html}.

\bibitem[Rosenfeld et~al.(2020)Rosenfeld, Winston, Ravikumar, and Kolter]{rosenfeld2020certified}
Elan Rosenfeld, Ezra Winston, Pradeep Ravikumar, and Zico Kolter.
\newblock Certified robustness to label-flipping attacks via randomized smoothing.
\newblock In Hal~Daumé III and Aarti Singh (eds.), \emph{Proceedings of the 37th International Conference on Machine Learning}, volume 119 of \emph{Proceedings of Machine Learning Research}, pp.\  8230--8241. PMLR, July 2020.
\newblock URL \url{https://proceedings.mlr.press/v119/rosenfeld20b.html}.

\bibitem[Sabanayagam et~al.(2025)Sabanayagam, Gosch, G\"{u}nnemann, and Ghoshdastidar]{sabanayagam2025exact}
Mahalakshmi Sabanayagam, Lukas Gosch, Stephan G\"{u}nnemann, and Debarghya Ghoshdastidar.
\newblock Exact certification of (graph) neural networks against label poisoning.
\newblock In Y.~Yue, A.~Garg, N.~Peng, F.~Sha, and R.~Yu (eds.), \emph{International Conference on Learning Representations}, volume 2025, pp.\  25663--25691, 2025.
\newblock URL \url{https://proceedings.iclr.cc/paper_files/paper/2025/file/401aa72e0e3be680348a5b0ffdb1a5aa-Paper-Conference.pdf}.

\bibitem[Saha et~al.(2020)Saha, Subramanya, and Pirsiavash]{saha2020hidden}
Aniruddha Saha, Akshayvarun Subramanya, and Hamed Pirsiavash.
\newblock Hidden trigger backdoor attacks.
\newblock \emph{Proceedings of the AAAI Conference on Artificial Intelligence}, 34\penalty0 (07):\penalty0 11957--11965, April 2020.
\newblock ISSN 2159-5399.
\newblock \doi{10.1609/aaai.v34i07.6871}.
\newblock URL \url{http://dx.doi.org/10.1609/aaai.v34i07.6871}.

\bibitem[Salman et~al.(2019)Salman, Li, Razenshteyn, Zhang, Zhang, Bubeck, and Yang]{salman2019provably}
Hadi Salman, Jerry Li, Ilya Razenshteyn, Pengchuan Zhang, Huan Zhang, Sebastien Bubeck, and Greg Yang.
\newblock Provably robust deep learning via adversarially trained smoothed classifiers.
\newblock In H.~Wallach, H.~Larochelle, A.~Beygelzimer, F.~d\textquotesingle Alch\'{e}-Buc, E.~Fox, and R.~Garnett (eds.), \emph{Advances in Neural Information Processing Systems}, volume~32. Curran Associates, Inc., 2019.
\newblock URL \url{https://proceedings.neurips.cc/paper_files/paper/2019/file/3a24b25a7b092a252166a1641ae953e7-Paper.pdf}.

\bibitem[Saxena et~al.(2026)Saxena, Schuchardt, Scholten, and Günnemann]{saxena2026provable}
Aman Saxena, Jan Schuchardt, Yan Scholten, and Stephan Günnemann.
\newblock Provable robustness against backdoor attacks via the primal-dual perspective on differential privacy, May 2026.
\newblock URL \url{https://arxiv.org/abs/2605.21780}.

\bibitem[Shafahi et~al.(2018)Shafahi, Huang, Najibi, Suciu, Studer, Dumitras, and Goldstein]{shafahi2018poison}
Ali Shafahi, W.~Ronny Huang, Mahyar Najibi, Octavian Suciu, Christoph Studer, Tudor Dumitras, and Tom Goldstein.
\newblock Poison frogs! targeted clean-label poisoning attacks on neural networks.
\newblock In S.~Bengio, H.~Wallach, H.~Larochelle, K.~Grauman, N.~Cesa-Bianchi, and R.~Garnett (eds.), \emph{Advances in Neural Information Processing Systems}, volume~31. Curran Associates, Inc., 2018.
\newblock URL \url{https://proceedings.neurips.cc/paper_files/paper/2018/file/22722a343513ed45f14905eb07621686-Paper.pdf}.

\bibitem[Sosnin et~al.(2025)Sosnin, Mueller, Baader, Tsay, and Wicker]{sosnin2025certified}
Philip Sosnin, Mark~Niklas Mueller, Maximilian Baader, Calvin Tsay, and Matthew~Robert Wicker.
\newblock Certified robustness to data poisoning in gradient-based training.
\newblock \emph{Transactions on Machine Learning Research}, 2025.
\newblock ISSN 2835-8856.
\newblock URL \url{https://openreview.net/forum?id=9WHifn9ZVX}.

\bibitem[Souri et~al.(2022)Souri, Fowl, Chellappa, Goldblum, and Goldstein]{souri2022sleeper}
Hossein Souri, Liam Fowl, Rama Chellappa, Micah Goldblum, and Tom Goldstein.
\newblock Sleeper agent: Scalable hidden trigger backdoors for neural networks trained from scratch.
\newblock In S.~Koyejo, S.~Mohamed, A.~Agarwal, D.~Belgrave, K.~Cho, and A.~Oh (eds.), \emph{Advances in Neural Information Processing Systems}, volume~35, pp.\  19165--19178. Curran Associates, Inc., 2022.
\newblock \doi{10.52202/068431-1393}.
\newblock URL \url{https://proceedings.neurips.cc/paper_files/paper/2022/file/79eec295a3cd5785e18c61383e7c996b-Paper-Conference.pdf}.

\bibitem[Steinhardt et~al.(2017)Steinhardt, Koh, and Liang]{steinhardt2017certified}
Jacob Steinhardt, Pang~Wei Koh, and Percy~S Liang.
\newblock Certified defenses for data poisoning attacks.
\newblock In I.~Guyon, U.~Von Luxburg, S.~Bengio, H.~Wallach, R.~Fergus, S.~Vishwanathan, and R.~Garnett (eds.), \emph{Advances in Neural Information Processing Systems}, volume~30. Curran Associates, Inc., 2017.
\newblock URL \url{https://proceedings.neurips.cc/paper_files/paper/2017/file/9d7311ba459f9e45ed746755a32dcd11-Paper.pdf}.

\bibitem[Wang et~al.(2022)Wang, Levine, and Feizi]{wang2022improved}
Wenxiao Wang, Alexander~J Levine, and Soheil Feizi.
\newblock Improved certified defenses against data poisoning with (deterministic) finite aggregation.
\newblock In Kamalika Chaudhuri, Stefanie Jegelka, Le~Song, Csaba Szepesvari, Gang Niu, and Sivan Sabato (eds.), \emph{Proceedings of the 39th International Conference on Machine Learning}, volume 162 of \emph{Proceedings of Machine Learning Research}, pp.\  22769--22783. PMLR, July 2022.
\newblock URL \url{https://proceedings.mlr.press/v162/wang22m.html}.

\bibitem[Warner(1965)]{warner1965randomized}
Stanley~L. Warner.
\newblock Randomized response: A survey technique for eliminating evasive answer bias.
\newblock \emph{Journal of the American Statistical Association}, 60\penalty0 (309):\penalty0 63--69, 1965.
\newblock \doi{10.1080/01621459.1965.10480775}.
\newblock URL \url{https://www.tandfonline.com/doi/abs/10.1080/01621459.1965.10480775}.
\newblock PMID: 12261830.

\bibitem[Weber et~al.(2023)Weber, Xu, Karlaš, Zhang, and Li]{weber2023rab}
Maurice Weber, Xiaojun Xu, Bojan Karlaš, Ce~Zhang, and Bo~Li.
\newblock {RAB}: Provable robustness against backdoor attacks.
\newblock In \emph{2023 IEEE Symposium on Security and Privacy (SP)}, pp.\  1311--1328, 2023.
\newblock \doi{10.1109/SP46215.2023.10179451}.

\bibitem[Xie et~al.(2023)Xie, Long, Chen, Li, Koyejo, and Li]{xie2023unraveling}
Chulin Xie, Yunhui Long, Pin-Yu Chen, Qinbin Li, Sanmi Koyejo, and Bo~Li.
\newblock Unraveling the connections between privacy and certified robustness in federated learning against poisoning attacks.
\newblock In \emph{Proceedings of the 2023 ACM SIGSAC Conference on Computer and Communications Security}, CCS '23, pp.\  1511--1525, New York, NY, USA, 2023. Association for Computing Machinery.
\newblock ISBN 9798400700507.
\newblock \doi{10.1145/3576915.3623193}.
\newblock URL \url{https://doi.org/10.1145/3576915.3623193}.

\bibitem[Yang et~al.(2020)Yang, Duan, Hu, Salman, Razenshteyn, and Li]{yang2020randomized}
Greg Yang, Tony Duan, J.~Edward Hu, Hadi Salman, Ilya Razenshteyn, and Jerry Li.
\newblock Randomized smoothing of all shapes and sizes.
\newblock In Hal~Daumé III and Aarti Singh (eds.), \emph{Proceedings of the 37th International Conference on Machine Learning}, volume 119 of \emph{Proceedings of Machine Learning Research}, pp.\  10693--10705. PMLR, July 2020.
\newblock URL \url{https://proceedings.mlr.press/v119/yang20c.html}.

\bibitem[Zhang et~al.(2022)Zhang, Albarghouthi, and D\textquotesingle~Antoni]{zhang2022bagflip}
Yuhao Zhang, Aws Albarghouthi, and Loris D\textquotesingle~Antoni.
\newblock {BagFlip}: A certified defense against data poisoning.
\newblock In S.~Koyejo, S.~Mohamed, A.~Agarwal, D.~Belgrave, K.~Cho, and A.~Oh (eds.), \emph{Advances in Neural Information Processing Systems}, volume~35, pp.\  31474--31483. Curran Associates, Inc., 2022.
\newblock URL \url{https://proceedings.neurips.cc/paper_files/paper/2022/file/cc19e4ffde5540ac3fcda240e6d975cb-Paper-Conference.pdf}.

\bibitem[Zhu et~al.(2022)Zhu, Dong, and Wang]{zhu2022optimal}
Yuqing Zhu, Jinshuo Dong, and Yu-Xiang Wang.
\newblock Optimal accounting of differential privacy via characteristic function.
\newblock In Gustau Camps-Valls, Francisco J.~R. Ruiz, and Isabel Valera (eds.), \emph{Proceedings of The 25th International Conference on Artificial Intelligence and Statistics}, volume 151 of \emph{Proceedings of Machine Learning Research}, pp.\  4782--4817. PMLR, March 2022.
\newblock URL \url{https://proceedings.mlr.press/v151/zhu22c.html}.

\end{thebibliography}
